\documentclass[manuscript]{acmart}
\AtBeginDocument{%
  }

\setcopyright{acmlicensed}
\copyrightyear{2018}
\acmYear{2018}
\acmDOI{XXXXXXX.XXXXXXX}

\acmJournal{JACM}
\acmVolume{37}
\acmNumber{4}
\acmArticle{111}
\acmMonth{8}

\usepackage{tikz}
\usetikzlibrary{fit,backgrounds}
 
\usepackage{subfig} 

\begin{document}

\title{Proper decision trees: An axiomatic framework for solving optimal
	decision tree problems with arbitrary splitting rules}

\author{Xi He}
\affiliation{%
  \institution{Peking University}
  \city{Beijing}
  \country{China}}
\email{xihe@pku.edu.cn}

\author{Max A. Little}
\affiliation{%
  \institution{University of Birmingham}
  \city{Birmingham}
  \country{UK}}
\email{maxl@mit.edu}

\renewcommand{\shortauthors}{He and Little}

\begin{abstract}
We present an axiomatic framework for analyzing the algorithmic properties
of decision trees. This framework supports the classification of decision
tree problems through \emph{structural} and \emph{ancestral} constraints
within a rigorous mathematical foundation.

The central focus of this paper is a special class of decision tree
problems---which we term \emph{proper decision trees}---due to their
\emph{versatility} and \emph{effectiveness}. In terms of versatility,
this class subsumes several well-known data structures, including
binary space partitioning trees, $K$-D trees, and machine learning
decision tree models. Regarding effectiveness, we prove that only
proper decision trees can be uniquely characterized as $K$-permutations,
whereas typical non-proper decision trees correspond to binary-labeled
decision trees with substantially greater complexity.

Using this formal characterization, we develop a generic algorithmic
approach for solving optimal decision tree problems over arbitrary
splitting rules and objective functions for proper decision trees.
We constructively derive a generic dynamic programming recursion for
solving these problems exactly. However, we show that memoization
is generally impractical in terms of space complexity, as both datasets
and subtrees must be stored. This result contradicts claims in the
literature that suggest a trade-off between memoizing datasets and
subtrees. Our framework further accommodates constraints such as tree
depth and leaf size, and can be accelerated using techniques such
as thinning.

Finally, we extend our analysis to several non-proper decision trees,
including the commonly studied decision tree over binary feature data,
the binary search tree, and the tree structure arising in the matrix
chain multiplication problem. We demonstrate how these problems can
be solved by appropriately modifying or discarding certain axioms.
\end{abstract}



\received{20 February 2007}
\received[revised]{12 March 2009}
\received[accepted]{5 June 2009}

\maketitle

\section{Introduction}

A decision tree $t:\mathbb{R}^{D}\to L$ is a labeled binary tree
data structure that maps a data point $x\in\mathbb{R}^{D}$ or some
abstract objects that can be \emph{compared}, to a label $l:L$ by
traversing a sequence of logical questions. These questions, referred
to as \emph{splitting rules} $r:\mathcal{R}$, are determined by the
keys stored in the branch nodes. In the classical decision tree learning
problem, a splitting rule $r$ is defined by questions such as whether
the $i$-th feature of $x$ smaller than value $v$? This divides
the feature space into two regions, $x_{i}\leq v$ and $x_{i}>v$,
through hyperplanes parallel to the axis, $x_{i}=0$. After following
a sequence of such questions along a \emph{path} in the tree, the
data point is assigned to a leaf node, which stores information of
type $L$.

Rather than focusing solely on restricted axis-parallel splits, this
paper introduces a more general notion of decision trees, classifying
them via axioms based on \emph{structural} and\emph{ ancestral constraints
}of decision trees (see Subsection \ref{subsec: Decision-tree,-complete}
for definitions). Those that satisfy a special class of constraints
are referred to as \emph{proper decision trees}, while those that
do not are classified as \emph{non-proper}. Examples of proper and
non-proper decision trees are summarized in Figure \ref{fig: example problems}.

We are particularly interested in the optimization problem for decision
trees: Given a list of data $\mathit{xs}:\left[\mathbb{R}^{D}\right]$
(denote as $\mathcal{D}$), and a set of possible splitting rules
$\mathit{rs}:\left[\mathcal{R}\right]$ for constructing decision
trees, our goal is to find a size-$K$ decision tree $s:\mathit{DTree}\left(\mathcal{R},\mathcal{D}\right)$
within the search space $\mathcal{S}\left(K,\mathit{rs}\right)$ that
is optimal with respect to the objective $E:\mathit{DTree}\left(\mathcal{R},\mathcal{D}\right)\to\mathbb{R}$.
Formally, 
\begin{equation}
	s^{*}=\text{argmin}_{s\in\mathcal{S}\left(K,\mathit{rs}\right)}E\left(s\right).\label{eq: ODT-specification-MIP}
\end{equation}
Finding an optimal decision tree (ODT), however, is notoriously difficult.
It is well known that the problem of constructing the \emph{smallest
	optimal decision tree} (as opposed to fixed-size trees defined in
\ref{eq: ODT-specification-MIP}) is NP-hard \citet{laurent1976constructing}.
Furthermore, \citet{sieling2008minimization} proved that it is even
NP-hard to construct a polynomial time approximation algorithm with
a constant performance ratio for the ODT problem over binary feature
data, i.e., mappings of the form $T:\left\{ 0,1\right\} ^{n}\to\left\{ 0,1\right\} $\footnote{The number of features $n$ here is treated as a variable. In this
	work, we will also investigate the optimal decision tree problem over
	binary feature data studied in machine learning, where $n$ is fixed
	as a constant determined by the number of features.}.

However, from a purely combinatorial perspective, decision trees can
be viewed as \emph{labeled binary trees}, whose maximal complexity
is given by the \emph{number of possible tree shapes} multiplied by
the \emph{number of ways to label the internal nodes}. For a tree
with $K$ splitting rules (internal nodes), there are $K!$ labelings
and $\mathit{Catalan}\left(K\right)$\footnote{The Catalan number counts the number of distinct binary tree shapes
	of size $K$, given by $\mathit{Catalan}\left(K\right)=\frac{1}{K+1}\left(\begin{array}{c}
		2K\\
		K
	\end{array}\right)$.} shapes, giving a total $K!\times\mathit{Catalan}\left(K\right)$
complexity. This remains \emph{polynomial} when $K$ is fixed, suggesting
the ODT problem is polynomial-time solvable under bounded size. A
similar argument applies to depth-constrained trees, since a tree
of depth $d$ has at most $2^{d-1}-1$ splitting rules.

Interestingly, under the fixed-parameter assumption, \citet{ordyniak2021parameterized}
showed that the optimal decision tree problem is\emph{ fixed-parameter
	tractable (fixed tree size or depth)}. In other words, when the size
or depth is fixed, there exists a polynomial-time algorithm for solving
the problem. Indeed, the well-known CART algorithm \citep{breiman1984classification}
follows this idea: it uses a top-down greedy strategy to optimize
a binary tree under size or depth constraints, rather than solving
the general NP-hard problem studied by \citet{laurent1976constructing}.

Under reasonable assumptions on decision trees---specifically, the
\emph{proper decision trees} data structure proposed in this paper---we
further show that the $K!\times\mathit{Catalan}\left(K\right)$ complexity
of decision trees can be reduced to $K!$ by characterize decision
trees using $K$-permutations. Since permutations are among the most
well-studied combinatorial structures, this characterization provides
a fundamental basis for analyzing the combinatorial and algorithmic
properties of decision trees. The following theorem summarizes our
main algorithmic contribution.

\begin{figure}[H]
	\begin{centering}
		\subfloat[Decision tree models in machine learning]{\includegraphics[viewport=250 50 1000 750,clip,scale=0.15]{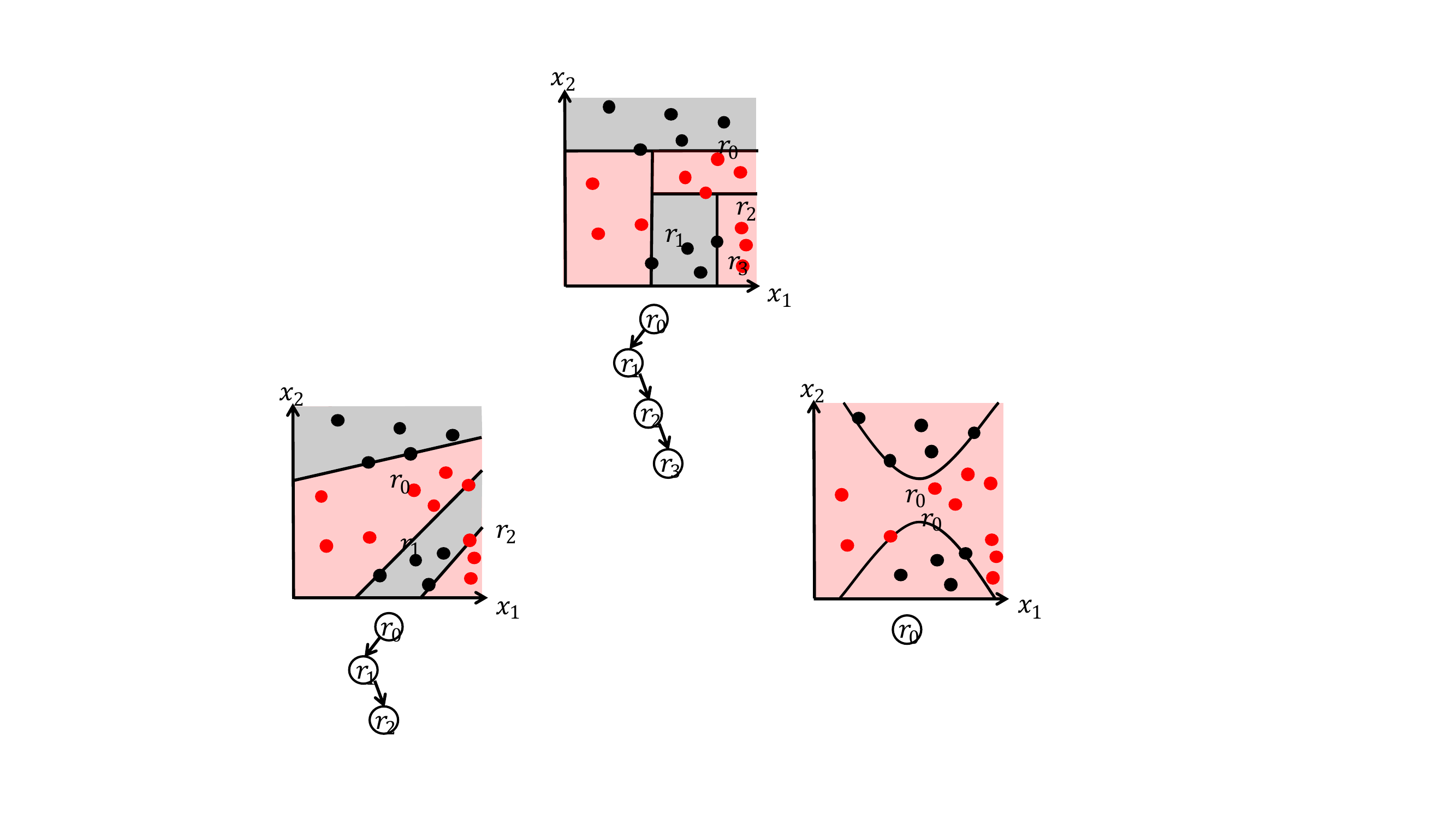}}
		\subfloat[Binary space partition tree]{\includegraphics[viewport=20 100 1280 680,clip,scale=0.14]{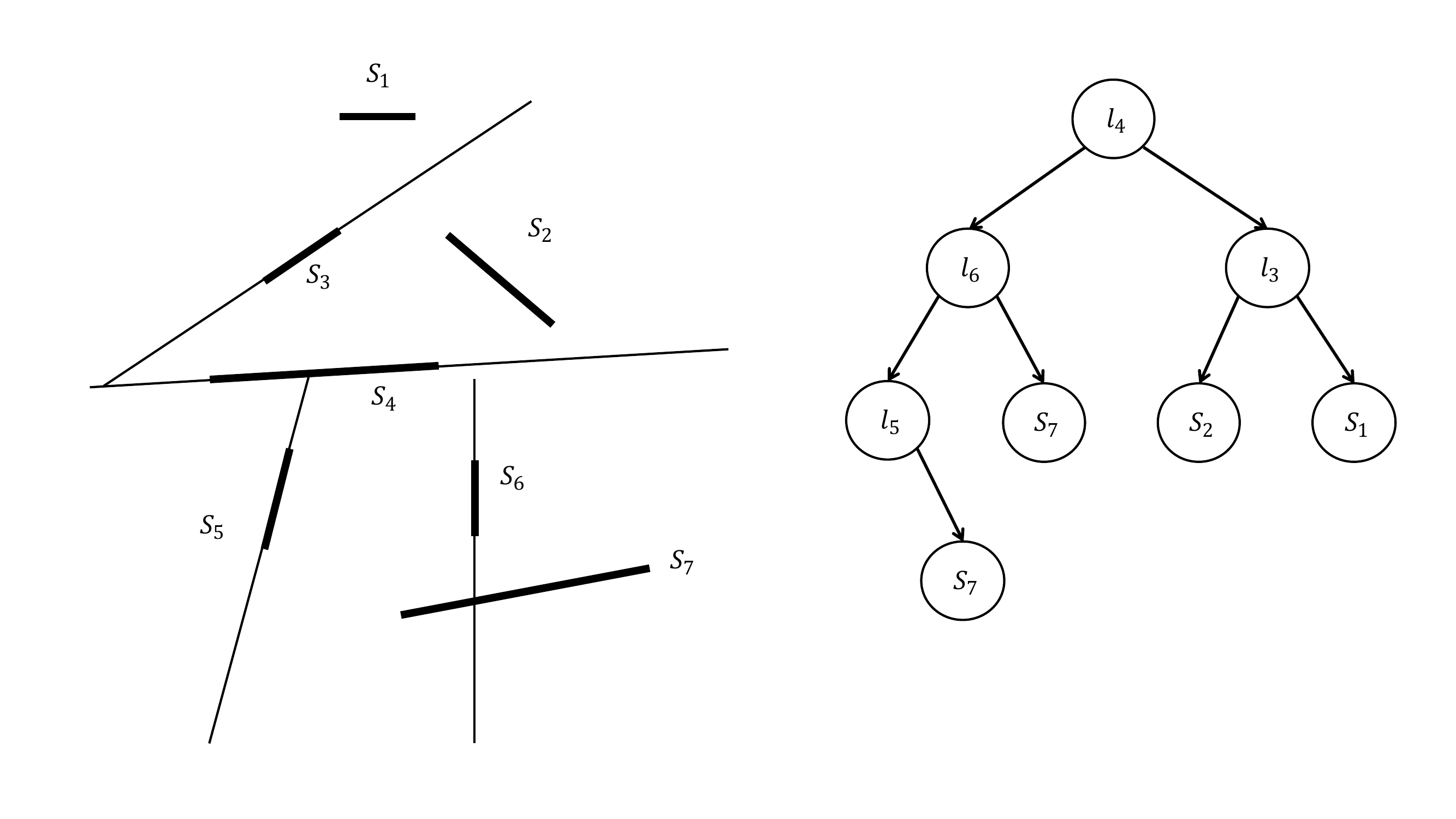}}
		\subfloat[$K$-$D$ tree]{\includegraphics[viewport=20 150 1280 750,clip,scale=0.13]{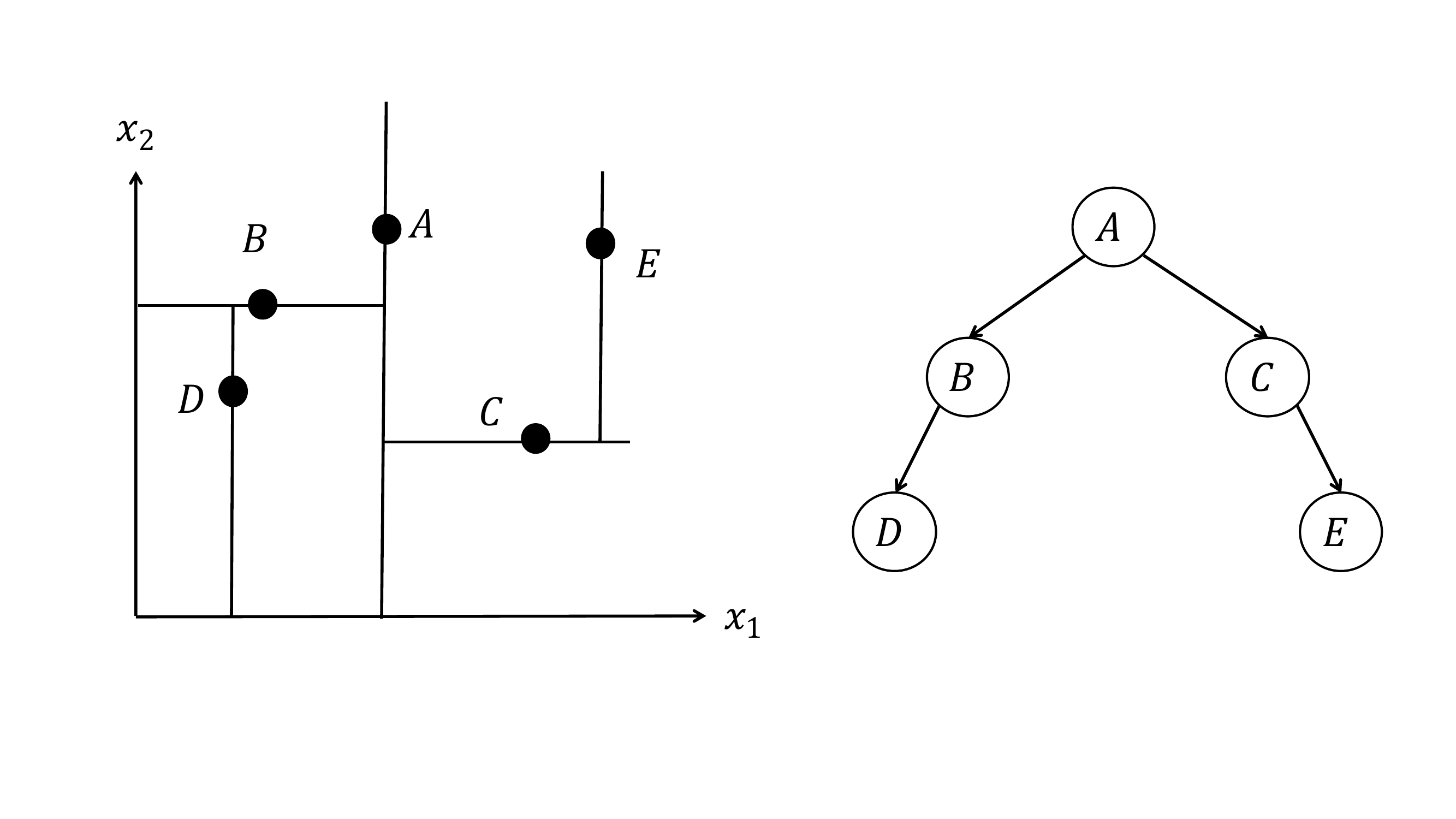}}\\
		\subfloat[Trees in matrix chain multiplication]{\includegraphics[viewport=100 20 1100 800,clip,scale=0.12]{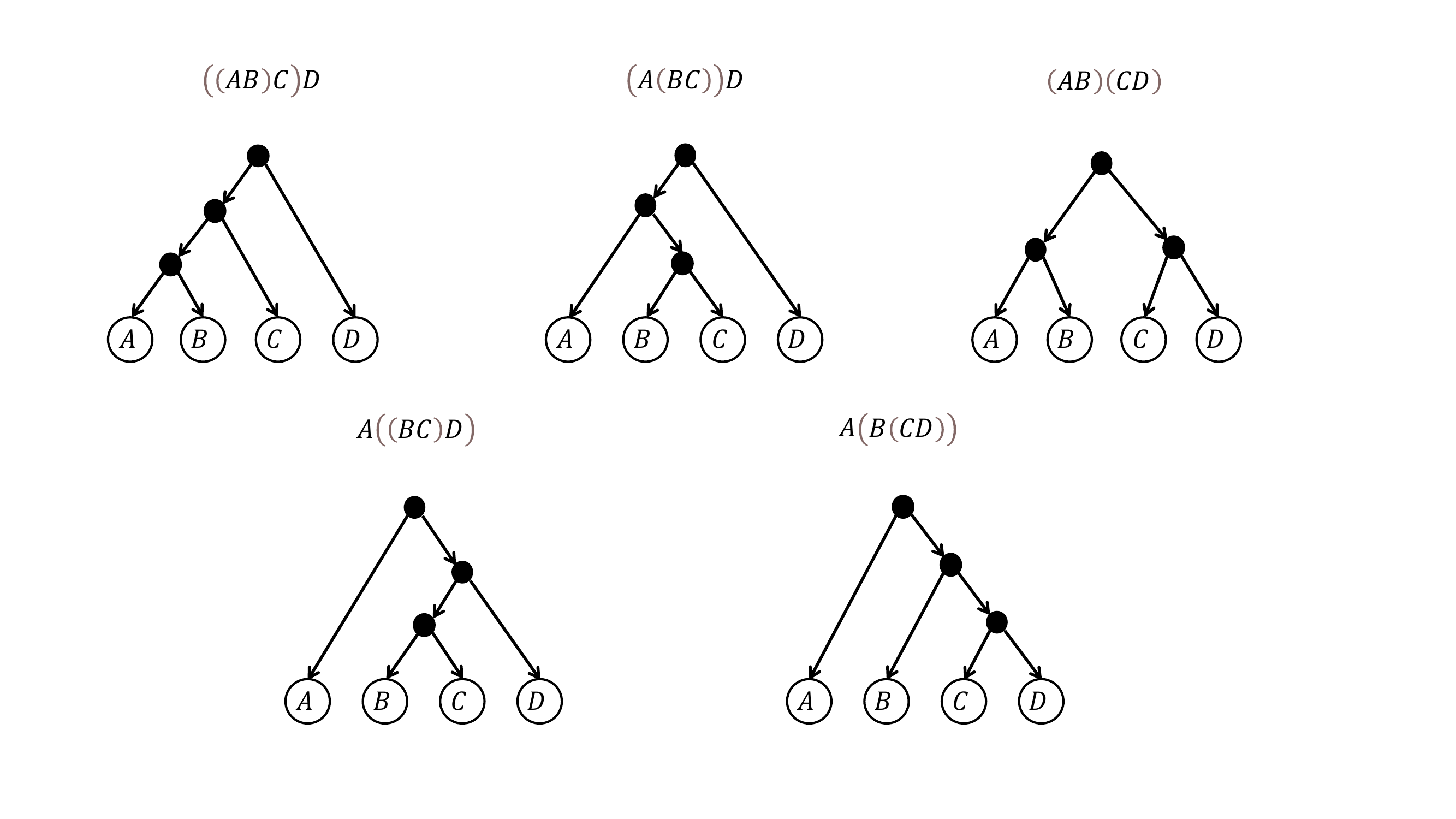}}
		\subfloat[Decision tree over binary feature data]{\includegraphics[viewport=20 20 1280 600,clip,scale=0.15]{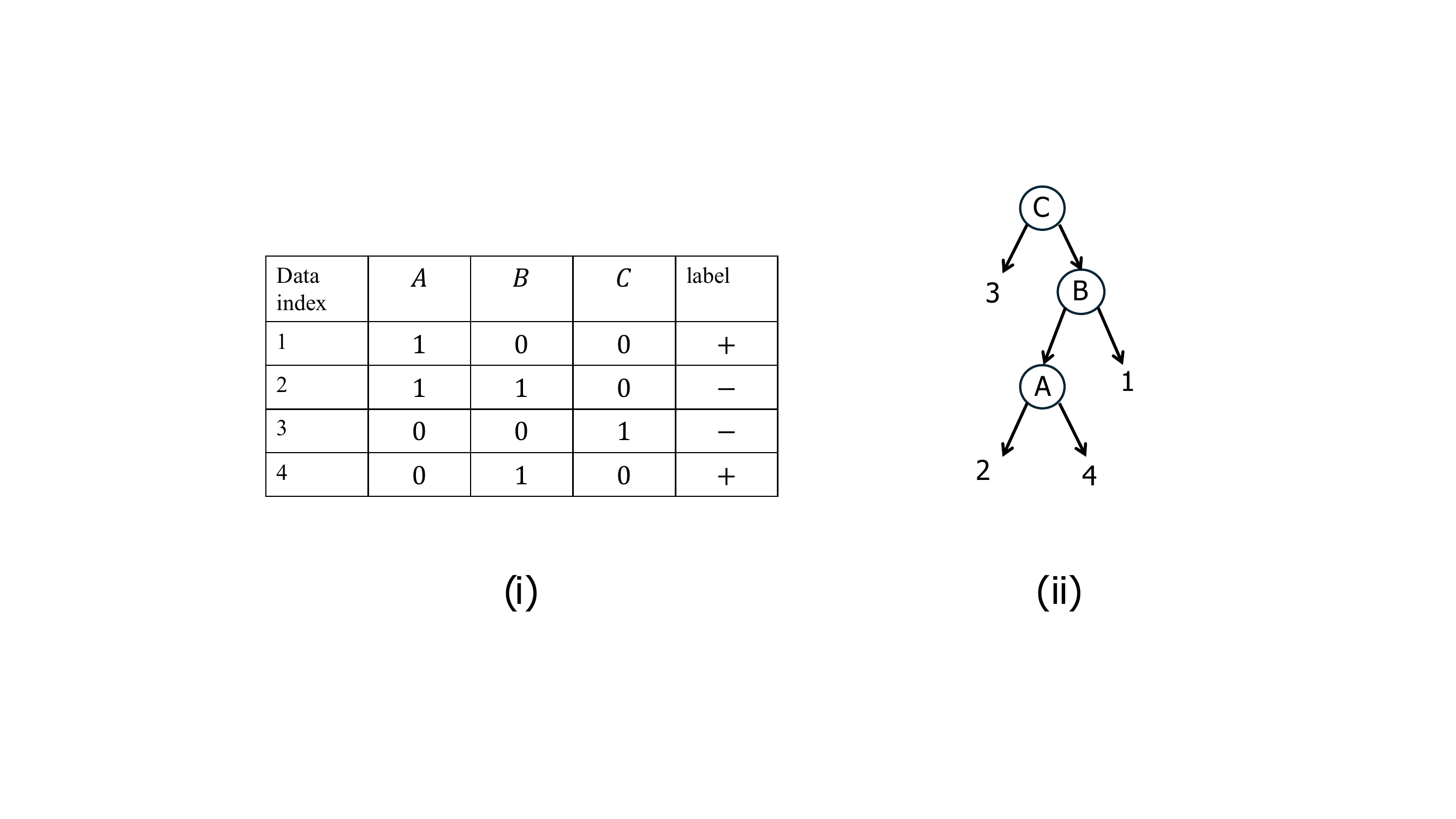}}
		\subfloat[Binary search trees]{\includegraphics[viewport=300 20 1000 750,clip,scale=0.15]{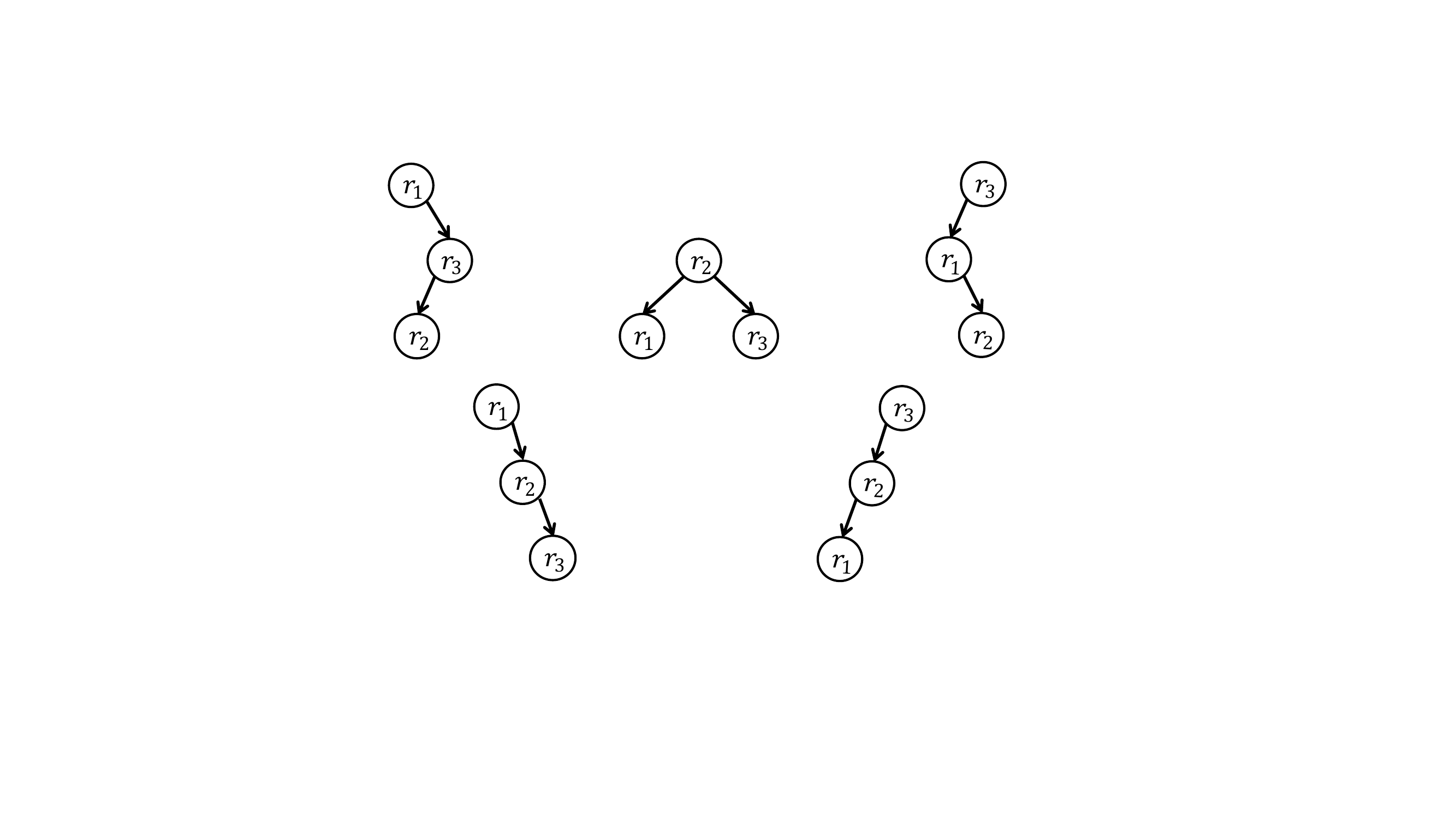}}
		
		\par\end{centering}
	\caption{Panels (a–c) describe three types of proper decision tree problems: (a)\textbf{ Axis-parallel decision tree model in machine
			learning}: This model uses axis-parallel splits to divide data into regions. For example, panel (a) shows four splits creating five regions (leaves), with predictions based on the majority class or average value in each region. This can be extended to more complex splits like hyperplanes or hypersurfaces. (b) \textbf{Binary
			space partition tree} : A segment-based decision tree that divides space into unique cells, each corresponding to a leaf in the tree, aiming for a minimal structure. (c) \textbf{$K$-D
			tree}: Similar to the axis-parallel tree, but nodes at the same level split along the same dimension. For instance, nodes
		$B$ and $C$ split along the $x_{2}$-axis, while nodes $D$
		and $E$ split along the $x_{1}$-axis. Panels (d–f) illustrate non-proper decision tree problems: (d) \textbf{Matrix
			chain multiplication problem}: This seeks the optimal order for multiplying matrices to minimize computational cost. Panel (d) shows five ways to multiply four matrices, with the tree's structure defining the order. (e) \textbf{Optimal decision tree problem over binary
			feature data }: Unlike axis-parallel trees, splits are based on binary questions (e.g., ``Is feature $A$ present?''). Paths go left for ``yes'' and right for ``no.'' In panel (e), figure (ii) shows a tree that classifies the four data points with three features shown in (i). (f) Binary
		search tree (BST): Given three nodes $r_1 \leq r_2 \leq r_3$, this panel illustrates five possible BSTs for these nodes. Each BST classifies nodes smaller than the root to the left and nodes greater than the root to the right.  \label{fig: example problems}}
\end{figure}

\begin{theorem}
	Simplified optimal decision tree problem.\emph{ Assume a list of rules
		$\mathit{rs}:\left[\mathcal{R}\right]$ and a size constraint $K:\mathbb{N}$.
		Let the search space }$\mathcal{S}\left(K,\mathit{rs}\right)$ \emph{of
		size-$K$ decision trees be defined by the program $\mathit{genDTKs}\left(K,\mathit{rs}\right)$.
		Then there exists a dynamic programming (DP) algorithm $\mathit{sodt}$
		such that the following inclusion holds:
		\begin{equation}
			\mathit{min}_{E}\left(\mathit{concatMapL}\left(\mathit{sodt},\mathit{kcombs}\left(K,\mathit{rs}\right)\right)\right)\subseteq\mathit{min}_{E}\left(\mathit{genDTKs}\left(K,\mathit{rs}\right)\right)\label{eq: SODT-introduction}
		\end{equation}
		where the symbol ``$\subseteq$'' indicates that the solution on the
		left-hand side is also a solution on the right-hand side. Here, $\mathit{kcombs}\left(K,\mathit{rs}\right)$
		generates all possible $K$-combinations of rules from $\mathit{rs}$.
		The function $\mathit{concatMapL}$ applies $\mathit{sodt}$ to each
		combination returned by $\mathit{kcombs}\left(K,\mathit{rs}\right)$
		and flattens the resulting list of lists into a single list. Finally,
		the operator $\mathit{min}_{E}$ is the programmatic definition of
		$\text{argmin}$, which selects the }first\emph{ optimal solution
		with respect to $E$ from a list of candidates. \label{thm: sodt-introduction}}
\end{theorem}
The key difference of our work with any previous research into combinatorial
ODT algorithms is that we formalize the search space $\mathcal{S}\left(K,\mathit{rs}\right)$
unambiguously as a recursive programs, then we can essentially ``compute''
it for a given input. Specifically, if $\mathit{genDTKs}\left(K,\mathit{rs}\right)$
will exhaustively explore the search space $\mathcal{S}\left(K,\mathit{rs}\right)$,
then the right-hand side of (\ref{eq: SODT-introduction}), which
is essentially a brute-force algorithm, constitutes a provably correct
procedure for solving (\ref{eq: ODT-specification-MIP}). In this
sense, Theorem \ref{thm: sodt-introduction} shows that the left-hand
side of (\ref{eq: SODT-introduction}) also solves (\ref{eq: ODT-specification-MIP}),
but more efficiently. In particular, in Subsection \ref{subsec:Complexity-of-the},
we will show that the left-hand side of (\ref{eq: SODT-introduction})
has a complexity of $O\left(K!\times M^{K}\right)$, which is polynomial
in the size of the input rules if $K$ is fixed, but factorial in
$K$.

Some readers from the machine learning (ML) community may wonder why
the search space $\mathcal{S}\left(K,\mathit{rs}\right)$ is defined
over rules rather than data points. This design choice aims to enhance
both the generality and brevity of the solution. Let $\mathit{odt}_{K}:\left[\mathcal{R}\right]\to\mathit{DTree}\left(\mathcal{R},\mathcal{D}\right)$
be a program that takes a list of rules $\mathit{rs}:\left[\mathcal{R}\right]$
and returns the optimal decision tree $s:\mathit{DTree}\left(\mathcal{R},\mathcal{D}\right)$
for solving (\ref{eq: ODT-specification-MIP}). For ML problems, we
compose $\mathit{odt}_{K}$ with another function $\mathit{genSplits}:\left[\mathbb{R}^{D}\right]\to\left[\mathcal{R}\right]$\footnote{Examples of $\mathit{genSplits}$ for specific ML problems are discussed
	in Subsection \ref{subsec: ODT in ML}.}, which converts a list of data points into a list of rules. This
design provides substantial benefits: it \textbf{modularizes} the
program, allowing us to easily switch the definition of $\mathit{genRules}$
for different problems while keeping the main program $\mathit{odt}_{K}$
unchanged. In applications without data points, such as the binary
space partition problem mentioned in Subsection \ref{subsec:Binary-space-partition},
$\mathit{genRules}$ can be omitted, and $\mathit{odt}_{K}$ can be
applied directly.

The paper is organized as follows: In Section 3, we present the background,
formalizing the concept of a decision tree using recursive datatypes
and novel axioms, which we refer to as defining \emph{proper decision
	trees}. Building on this foundation, we further prove that a size-$K$
proper decision tree can be uniquely characterized using a $K$-permutation.

In Section 4, we present the main algorithmic results for proving
Theorem (\ref{thm: sodt-introduction}). The proof is divided into
two parts:
\begin{enumerate}
	\item In the first part of the proof (validity of Equation (\ref{eq: SODT-introduction})),
	we formalize the search space of the ODT problem $\mathcal{S}\left(K,\mathit{rs}\right)$
	using $K$-permutations in Subsection \ref{subsec:tree datatype decision tree generatr}
	and in Subsection \ref{subsec: Deriving simplified odt problem using equational reasoning},
	using basic laws from\emph{ theory of lists} \citep{bird1987introduction}.
	We show that the ``\emph{original ODT problem}'' (\ref{eq: ODT-specification-MIP})
	can be factorized into a simplified version, denoted as $\mathit{sodt}$---the
	\emph{``simplified optimal decision tree problem}''. This establishes
	the first part of Theorem \ref{thm: sodt-introduction}.
	\item In the second part of the proof (construction of a dynamic programming
	algorithm for $\mathit{sodt}$), Subsection \ref{subsec: partial decision tree generator},
	designs a ``\emph{partial decision tree generator}'' based on the
	\emph{binary tree datatype}, which exhaustively enumerates all possible
	decision trees within the search space $\mathcal{S}\left(K,\mathit{rs}\right)$,
	except that the leaves of these trees contain no information. By generalizing
	\citet{gibbons1996computing}'s \emph{downwards accumulation }technique,
	as discussed in Subsection \ref{subsec:Downwards-accumulation-for},
	we show in Subsection \ref{subsec: complete decision tree generator}
	that a ``\emph{complete decision tree generator}'' satisfying all
	proper decision tree axioms can be derived from the partial generator,
	resulting in a complete program for defining $\mathcal{S}\left(K,\mathit{rs}\right)$.
	Finally, in Subsection \ref{subsec: DP algorithm}, we formally specify
	the objective function and derive an efficient dynamic programming
	recursion for solving the simplified ODT problem, completing the second
	part of Theorem \ref{thm: sodt-introduction}.
\end{enumerate}
In Section 5, we discuss applications of our results to constructing
optimal binary space partitioning trees, ML decision tree models,
and $K$-D trees. In Section 6, we explore extensions to non-proper
decision trees. Specifically, we analyze four cases: the optimal decision
tree problem over binary feature data (ODT-BF), the classical matrix
chain multiplication problem, and \emph{two} cases (one of which corresponds
to the binary search tree) derived from modeling the algorithm of
\citet{demirovic2022murtree}. These examples demonstrate the remarkable
flexibility of our framework in addressing diverse decision tree problems
through simple, minimal modifications of the axioms.

Lastly, in Section 7, we present a summary and brief discussion of
contributions, and suggest future research directions.

\section{Related studies and consequences of ambiguous problem definitions}

\subsection*{Related studies}

Because of the fixed-parameter tractability of the ODT problem, the
study of depth- and size-constrained decision tree optimization has
increasingly attracted attention in optimization \citet{bertsimas2017optimal},
ML \citet{demirovic2022murtree,lin2020generalized}, and theoretical
computer science \citet{blanc2019top,blanc2022properly}. In this
paper, we focus primarily on the globally optimal (exact) solution
to the ODT problem. Consequently, heuristic or approximate methods---such
as studies based on assumptions about the data distribution \citet{blanc2022properly},
or greedy methods \citet{blanc2019top,breiman1984classification}---are
outside the scope of our discussion.

A prominent approach for exact ODT algorithms is the use of mixed-integer
programming (MIP) solvers. Since the pioneering work of \citet{bertsimas2017optimal},
there has been a surge of interest in the ODT problem within the optimization
community \citep{boutilier2023optimal,gunluk2021optimal,zhu2020scalable,verwer2019learning,shatter2017Justin,verwer2017learning}.
One major advantage of MIP approaches is their flexibility: the objective
function can be easily modified, and constraints on the tree---such
as size, depth, or leaf size---can be adopted, including considerations
for categorical datasets \citep{gunluk2021optimal,gunluk2016optimal}.

Another advantage is the generality in defining decision trees. Classical
decision trees use \emph{axis-parallel splits} \citep{breiman1984classification},
but constraints in the MIP formulation can be easily modified to generalize
axis-parallel splits to \emph{hyperplane splits} \citep{bertsimas2017optimal,shatter2017Justin},
which ask: Do data points lie in the positive or negative region of
a hyperplane $h$? This question divides the feature space using a
more flexible rule $\boldsymbol{w}^{T}x\leq c$ or $\boldsymbol{w}^{T}x>c$,
such that $\boldsymbol{w}\in\mathbb{R}^{D}$.

However, a major drawback of MIP solvers is that their overall complexity
is \emph{unpredictable}, and in many cases, these algorithms exhibit
exponential (or worse) worst-case complexity. Moreover, progress in
this area largely relies on improvements in MIP solvers rather than
a deeper understanding of the underlying problem. Another similar
line of approach reformulates the ODT problem as a Boolean satisfiability
problem (SAT) and solves it using existing SAT solvers \citep{janota2020sat,narodytska2018learning,avellaneda2020efficient}.

Finally, a closely related line of research seeks to solve the ODT
problem using combinatorial methods \citep{brita2025optimal,zhang2023optimal,aglin2021pydl8,nijssen2007mining,nijssen2010optimal,aglin2020learning,
	demirovic2022murtree,hu2019optimal,lin2020generalized,mazumder2022quant}
often ambiguously referred to as branch-and-bound (BnB) or dynamic
programming (DP) methods. However, unlike studies using the MIP formulation,
where \citet{bertsimas2017optimal} have inspired numerous follow-up
studies in their field, \citet{cox1989heuristic} established the
first combinatorial algorithm for the ODT problem in 1989 and no subsequent
research has directly built upon this foundation. Researchers have
instead designed various algorithms using different data structures
to solve decision trees under different constraints. For example,
\citet{demirovic2022murtree,verwer2019learning,nijssen2007mining,aglin2020learning,nijssen2010optimal}
constrain \emph{tree depth}, whereas \citet{zhang2023optimal,lin2020generalized,hu2019optimal}
constrain the \emph{number of leaves}, referring to these as \emph{sparse
	decision trees}.

Moreover, combinatorial methods rarely address the ODT problem in
full generality. Most studies focus on simpler variants, particularly
the ODT-BF problem \citep{zhang2023optimal,aglin2021pydl8,nijssen2007mining,nijssen2010optimal,aglin2020learning,demirovic2022murtree,hu2019optimal,lin2020generalized,rudin2019stop},
where the goal is to find a tree $T:\left\{ 0,1\right\} ^{D}\to L$
over binary inputs $x\in\left\{ 0,1\right\} ^{D}$ rather than over
data with real-valued features $x\in\mathbb{R}^{D}$. Consequently,
the combinatorial complexity of the ODT-BF problem is\emph{ independent
	of the dataset size} (See Section III.3.10 of \citet{he2025ROF}).
Empirical results from \citet{hu2019optimal} indicate that, when
the feature dimension is fixed, the algorithm scales linearly with
data size. To our knowledge, only \citep{brita2025optimal,mazumder2022quant}
address the classical axis-parallel optimal decision tree (AODT) problem.

Furthermore, not only do the problems studied differ significantly,
but the methods applied, even for the same problem, vary significantly.
For instance, in the widely studied ODT-BF problem, \citet{nijssen2007mining,aglin2021pydl8,aglin2020learning,nijssen2010optimal}
characterize decision trees using \emph{itemsets}, \citet{demirovic2022murtree}
use \emph{feature vectors}, and \citet{lin2020generalized,hu2019optimal}
define trees via \emph{two sets of leaves}---those that can be further
split and those that are fixed. Unfortunately, the relative advantages
and limitations of these incompatible representations remain insufficiently
explored.

Compared with other research directions, studies on combinatorial
algorithms for the ODT problem are generally less formal, as discussed
below, revealing ambiguities in their methods and potential implications.

\subsection*{Consequences of ambiguous problem definitions}

Many studies employing combinatorial methods, such as \citet{demirovic2022murtree,hu2019optimal,zhang2023optimal,lin2020generalized},
do not provide a formal problem definition. In particular, \citep{lin2020generalized}
claim to ``present a new representation of the dynamic programming
search space,'' yet this search space is not explicitly defined.

In other cases, the problem is described informally, raising problematic
ambiguities. For example, \citet{brita2025optimal} state, ``Let
$\mathcal{T}\left(\mathcal{D},d\right)$ describe the set of all decision
trees for the dataset $\mathcal{D}$ with a maximum depth of $d$.''
and \citet{mazumder2022quant} define the search space $\mathcal{T}_{2}$
as ``the set of all decision trees with depth 2 whose splitting thresholds
are in ...'' Furthermore, \citet{nijssen2007mining,nijssen2010optimal}
states ``we are interested in expressing decision tree learning problems
... of the form $\text{argmin}_{T}\left(f\left(T\right)\right)\text{ subject to }\varphi\left(T\right)$
which corresponds to finding the best tree(s) ... among all trees
...'' yet it remains unclear how ``all trees'' is formally defined.
Although \citet{nijssen2007mining,nijssen2010optimal} proved an equivalence
between \emph{decision trees} and \emph{sets of itemsets}, it is left
open to interpretation how their algorithms can \emph{exhaustively}
explore the search space of itemsets. Even in their correctness claims---``The
correctness of this approach follows from the following facts''---there
is no discussion regarding the exhaustiveness of the algorithm.

Although these informal definitions may seem intuitively obvious,
they share a common problem: given an input, it is unclear how to
``\emph{compute}'' the search space to verify whether it matches others'
understanding. This is crucial, as different readers may impose different
constraints on the search space or even have different interpretations
of what constitutes a decision tree.

In contrast, in the study of MIP solvers, using a solver requires
specifying the problem unambiguously in a standard form. This ensures
both correctness and unambiguous communication between researchers.
Any special constraints are explicitly encoded in the MIP specification.
Such clarity is lacking in the combinatorial algorithms discussed
above, where researchers often consider it sufficient to prove the
correctness of a bounding strategy or the independence of subtree
evaluations to justify the DP recursion. However, establishing the
exactness of an algorithm fundamentally requires demonstrating that
it exhaustively explores all solutions in the search space---an impossible
task without first defining that space. The correctness of bounding
techniques or DP solutions depends entirely on the algorithm being
exhaustive from the outset.

To clarify our point, in Subsection \ref{subsec:Non-exhua-Murtree},
we explain how the ambiguous description of the state-of-art ODT-BF
algorithm---the Murtree proposed by \citet{demirovic2022murtree}\footnote{Indeed, we choose to analyze \citet{demirovic2022murtree}'s algorithm
	instead of those proposed by \citet{nijssen2007mining,hu2019optimal,lin2020generalized,rudin2019stop,zhang2023optimal,aglin2021pydl8,aglin2020learning,nijssen2010optimal}
	because we believe it provides the clearest description among the
	available approaches.}---leads to multiple interpretations. We present two plausible formal
models that may arise from differing interpretations of their informal
algorithm description, and demonstrate that the search spaces of both
models are non-exhaustive for the ODT-BF problem. This analysis may
help guide future research towards avoiding similar pitfalls.

Our assertion is that the root cause of these issues lies in the absence
of a formal framework for specifying the problem using combinatorial
methods, and a standardized process for deriving efficient algorithms
from the given specification. This paper aims to address this gap
and provide a foundation for future studies in this area.

\section{Background}

The types of real and natural numbers are denoted as $\mathbb{R}$
and $\mathbb{N}$, respectively. We use square brackets $\left[\mathcal{A}\right]$
to denote the set of all finite lists of elements $a:\mathcal{A}$,
where $\mathcal{A}$ (or letters $\mathcal{B}$ and $\mathcal{C}$
at the front of the alphabet) represent type variables. Hence, $\left[\mathcal{R}\right]$,
$\left[\mathcal{H}\right]$, $\left[\mathcal{S}\right]$ and $\left[\mathbb{R}^{D}\right]$,
denote the set of all finite lists of splitting rules, hyperplanes,
hypersurfaces, and lists of data, respectively. We use $\mathcal{D}$
as a short-hand synonym for $\left[\mathbb{R}^{D}\right]$. Variables
of these types are denoted using their corresponding lowercase letters
e.g. $r:\left[\mathcal{R}\right]$, $h:\left[\mathcal{H}\right]$,
$s:\left[\mathcal{S}\right]$.

\subsection{The proper decision tree datatype: A novel axiomatic definition for
	decision trees}

The algorithm design community has shown greater interest in inventing
new algorithms than in organizing existing knowledge. However, a systematic
framework for classifying existing algorithms is equally valuable,
as it not only deepens our understanding of problems through algorithmic
and combinatorial insights but also facilitates unambiguous communication
of ideas. For example, most NP-hardness proofs of the ODT problem,
rely on very general assumptions about the structure of decision trees.
Yet such general definitions are often too broad to justify claims
that ``all ODT problems'' are NP-hard.

By imposing additional constraints on the general problem, it is sometimes
possible to construct efficient algorithms. An analogous example is
the traveling salesperson problem (TSP): although NP-hard in general,
the Euclidean TSP can be approximated to an arbitrary but fixed constant
accuracy in polynomial time when the dimension is fixed \citep{Arora1998PTASofTSP}.
Similarly, the \emph{optimal binary search tree problem} can be regarded
as a special optimal decision tree problem which is solvable in polynomial
time \citep{bird2020algorithm}.

On the other hand, if the constraints imposed are too strong, however,
the resulting theory may lose practical relevance. In this subsection,
we propose a novel framework for classifying decision tree problems
by identifying two key types of constraints that define them: \textbf{structural}
and \textbf{ancestral} constraints.

\subsubsection{Structural constraints on decision trees}

As the name suggests, a decision tree is a tree-based graph model.
In a \emph{directed} graph, if there is a directed edge from one node
to another, the node at the destination of the edge is called the
\emph{child} node of the source node, while the source node is referred
to as the \emph{parent} node. If there is a directed \emph{path} connecting
two nodes, the source node is called an \emph{ancestor }node.

The topmost node of the tree is referred to as the \emph{root}. The
nodes farthest from the root are called \emph{leaf} \emph{nodes},
while all other nodes are referred to as \emph{branch} or \emph{internal}
nodes. The keys stored in branch nodes of decision trees are referred
to as \emph{splitting rules}; thus, we use the terms splitting rules
and internal nodes interchangeably. A sequence of nodes along the
edges from the root to the leaf of a tree is called a \emph{path}.

The concept of ``decision tree'' varies significantly across different
fields. There are several common \emph{structural constraints} in
the definitions of decision trees across studies on ODT problems,
we summarize them (informally) in following definition.
\begin{definition}
	\emph{Structural constraints on optimal decision tree problems}. The
	various definitions of decision trees typically share several common
	constraints:\label{def:Common-features-in}
\end{definition}
\begin{itemize}
	\item Each branch node of a decision tree contains only \emph{one} \emph{splitting}
	\emph{rule}, which partitions the ambient space\footnote{An ambient space is the space in which a mathematical (geometric or
		topological) object is embedded, along with the object itself.} into two \emph{disjoint} and \emph{continuous} subspaces.
	\item Each leaf specifies a region defined by intersection of partitions
	of all splitting rules in the path from root to leaf.
	\item The descandants of a splitting rule are also descandants of its ancestor.
	A new splitting rule can only be generated from the subspace defined
	by its ancestor rules.
\end{itemize}
In particular, the last constraint is justified by the fact that when
\emph{geometric splitting rules} (splitting rules defined by geometrical
objects)---such as axis-parallel or general hyperplanes---are used,
each rule should be generated only from the subspace defined by its
ancestor nodes. In contrast, for \emph{logical splitting rules} (splitting
rules defined by a logical question), such as those in the ODT-BF
problem, where each split corresponds to a logical question about
the presence of a feature in the data, such a logical splitting rule
can be generated from \emph{any} subspace and the last constraint
remains valid.

Other than structural constraints on decision trees, an important
distinction in our research is recognizing the interaction between
splitting rules: the splitting rules in decision trees not only partition
the ambient space into two regions but also constrain the search space
of decision trees by regulating how rules interact when one is considered
the ancestor of another. We refer to this interaction as an \textbf{ancestral
	constraint}, which, as we will see in a later section, implicitly
governs the definition of the recursive step in algorithms for solving
ODT problems.

\subsubsection{Ancestry relations and proper decision tree axioms\label{subsec: Decision-tree,-complete}}

\begin{figure}[h]
	\begin{centering}
		\includegraphics[viewport=200bp 250bp 1080bp 580bp,clip,scale=0.3]{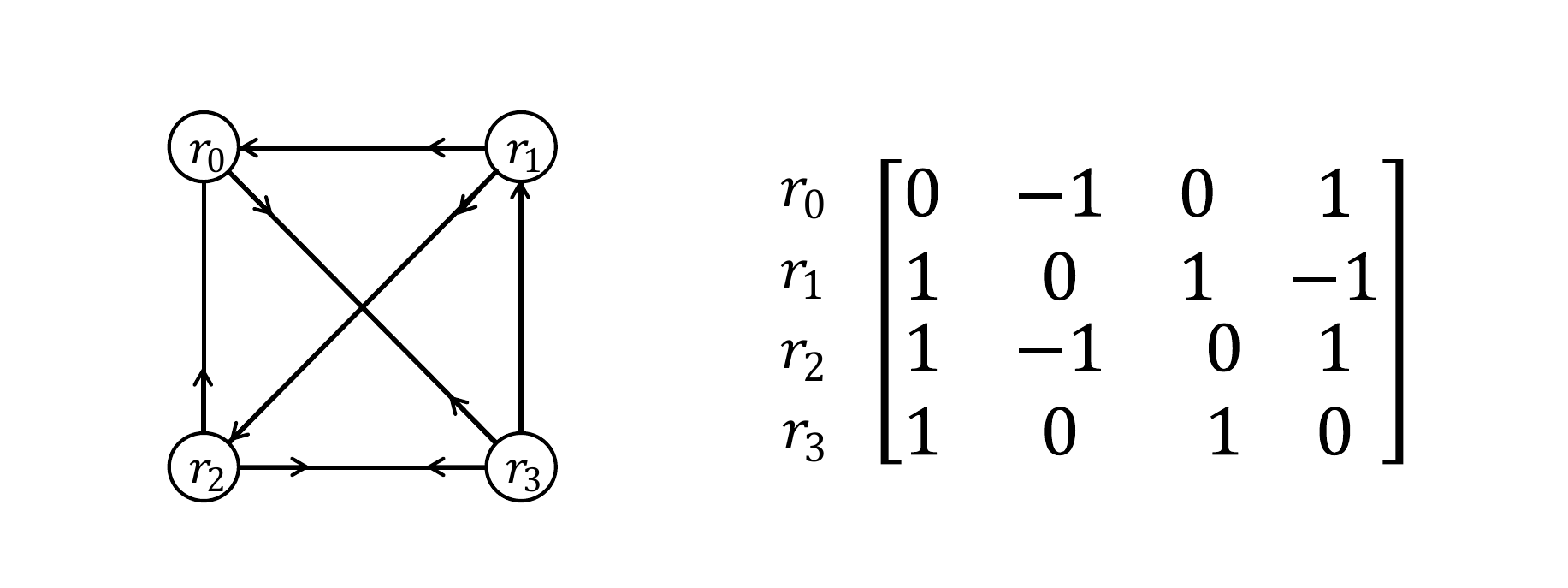}
		\par\end{centering}
	\caption{The \emph{ancestry} \emph{relation} \emph{graph} (left) captures all
		ancestry relations between four splitting rules $\left[r_{0},r_{1},r_{2},r_{3}\right]$.
		In this graph, nodes represent rules, and arrows represent ancestral
		relations. An incoming arrow from $r_{i}$ to a node $r_{j}$ indicates
		that $r_{j}$ is the right-child of $r_{i}$ (read the arrow next
		to $r_{i}$) . The absence of an arrow indicates no ancestral relation.
		An outgoing arrows from $r_{i}$ to a node $r_{j}$ indicates that
		$r_{j}$ is the left-child of $r_{i}$. The ancestral relation matrix
		(right) $\boldsymbol{K}$, where the elements $\boldsymbol{K}_{ij}=1$,
		$\boldsymbol{K}_{ij}=-1$, and $\boldsymbol{K}_{ij}=0$ indicate that
		$r_{j}$ lies on the positive side, negative side of $r_{i}$, or
		that there is no ancestry relation between them, respectively. \label{fig:Equivalent-representation-to}}
\end{figure}

Both structural and ancestral constraints are defined by \emph{ancestry
	relations}, a binary relation that describes the interaction between
two splitting rules. The ancestry relations are formalized below.
\begin{definition}
	\emph{Ancestry relations}. Given a list of $K$ rules $\mathit{rs}=\left[r_{1},r_{2},\ldots r_{K}\right]$,
	the ancestry relation between any pair of rules $r_{i}$ and $r_{j}$
	is defined as follows:
	
	1. $r_{i}\swarrow r_{j}$ if $r_{j}$ is in the left subtree of $r_{i}$
	(i.e., $r_{j}$ is a left descendant of $r_{i}$).
	
	2. $r_{i}\searrow r_{j}$ if $r_{j}$ is in the right subtree of $r_{i}$
	(i.e., $r_{j}$ is a right descendant of $r_{i}$).
	
	Using logical disjunction $\vee$, we write $r_{i}\left(\swarrow\vee\searrow\right)r_{j}$
	if $r_{j}$ is in the left \emph{or} right branch of $r_{i}$; in
	this case, $r_{j}$ is called a \textbf{descendant} of $r_{i}$. The
	complement of the relation $\left(\swarrow\vee\searrow\right)$ can
	be expressed as $\overline{\left(\swarrow\vee\searrow\right)}$, where
	$\overline{R}$ denotes the complement of relation $R$. By De Morgan's
	law, we have $r_{i}\overline{\left(\swarrow\vee\searrow\right)}r_{j}=\left(r_{i}\overline{\swarrow}r_{j}\right)\wedge\left(r_{i}\overline{\searrow}r_{j}\right)$,
	which means that $r_{i}\overline{\left(\swarrow\vee\searrow\right)}r_{j}$
	holds if and only if $r_{j}$ is \textbf{not} a branch node in either
	the left or right subtree of $r_{i}$.
\end{definition}
The notation $\searrow$ and $\swarrow$ must be read from left to
right because $r_{i}\swarrow r_{j}$ and $r_{i}\searrow r_{j}$ do
not imply $r_{j}\searrow r_{i}$ and $r_{j}\swarrow r_{i}$, unless
$r_{i}$ and $r_{j}$ are mutual ancestors of each other. In other
words, $\searrow$ and $\swarrow$ are not \emph{commutative }relations.

These ancestry relations can be characterized as \emph{homogeneous
	binary relations}. Relations and graphs are closely related, and homogeneous
binary relations over a set can be represented as directed graphs
\citep{schmidt2012relations}. Therefore, the ancestry relations can
be encoded as a \emph{complete} \emph{graph}, where the splitting
rules are the nodes in the graph, and the ancestry relations $\searrow$
and $\swarrow$ are represented by outgoing and incoming arrows, respectively,
while a line without an arrowhead ($\smallsetminus$) represents the
complement relation $\overline{\left(\swarrow\vee\searrow\right)}$.
We refer to it as the \emph{ancestry} \emph{relation} \emph{graph}.
The reason this graph is complete is that every splitting rule $r_{i}$
is related to any other splitting rule $r_{j}$ in some way, either
through an ancestry relation or by being unrelated. Figure \ref{fig:Equivalent-representation-to}
illustrates an ancestry relation matrix for a given set of splitting
rules, along with its corresponding complete graph. This example is
derived from a concrete case in which the splitting rules are defined
by hyperplanes, as shown in Figure \ref{fig:Four-hyperplanes, for generating decisiion trees}.
The resulting decision trees constructed from this ancestry relation
matrix are illustrated in Figure \ref{fig:possible_decision_trees}.

Moreover, binary relations can also be characterized as \emph{Boolean}
\emph{matrices}. However, to encode two binary relations,$\searrow$
and $\swarrow$ in one matrix, the values $1$ and $-1$ are used
to distinguish them.
\begin{definition}
	\emph{Ancestry relation} \emph{matrix}. Given a list of $K$ rules
	$\mathit{rs}=\left[r_{1},r_{2},\ldots r_{K}\right]$, the ancestry
	relations between any pair of rules can be characterized as a $K\times K$
	square matrix $\boldsymbol{K}$, with elements defined as follows:
\end{definition}
\begin{itemize}
	\item $\boldsymbol{K}_{ij}=1$ if $r_{i}\swarrow r_{j}$ (i.e.,$r_{j}$
	is in the left subtree of $r_{i}$),
	\item $\boldsymbol{K}_{ij}=-1$ if $r_{i}\searrow r_{j}$ (i.e.,$r_{j}$
	is in the right subtree of $r_{i}$),
	\item $\boldsymbol{K}_{ij}=0$ if $r_{i}\overline{\left(\swarrow\vee\searrow\right)}r_{j}$
	(i.e., if $r_{j}$ is not a branch node in both the left and right
	subtree of $r_{i}$), and $\boldsymbol{K}_{ii}=0$, for all $i$,
	since a splitting rule cannot be an ancestor of itself.
\end{itemize}
The ancestry relation matrix and the ancestry relation graph are two
alternative representations of ancestry relations. As we will see
in a later section, the ancestry relation matrix $\boldsymbol{K}$
enables us to define functions more compactly than using the symbols
$\searrow$ or $\swarrow$. We are now ready to formalize the axioms
of the decision tree.
\begin{definition}
	\emph{Axioms for proper decision trees}. We call a decision tree consists
	of splitting rules that satisfies the following axioms, a \emph{proper}
	decision tree:\label{axioms: proper decision tree}
\end{definition}
\begin{enumerate}
	\item \emph{Structural constraint one} (Ambient space partition): Each branch
	node is defined by a single splitting rule $r:\mathcal{R}$, and each
	splitting rule subdivides the ambient space into two \emph{disjoint}
	and connected subspaces, $r^{+}$ and $r^{-}$.
	\item \emph{Structural constraint two} (Path intersection in leafs): Each
	leaf $L$ is defined by the intersection of subspaces $\bigcap_{p\in P_{L}}r_{p}^{\pm}$
	for all the splitting rules $\left\{ r_{p}\mid p\in P_{L}\right\} $
	in the path $P_{L}$ from the root to leaf $L$. The connected region
	(subspace) defined by $\bigcap_{p\in P_{L}}r_{p}^{\pm}$ is referred
	to as the \emph{decision} \emph{region}.
	\item \emph{Structural constraint three} (Partition transitivity): Part
	I: The ancestry relation between any pair of splitting rules $r_{i}\left(\swarrow\vee\searrow\right)r_{j}$
	is transitive; in other words, if $r_{i}\left(\swarrow\vee\searrow\right)r_{j}$
	and $r_{j}\left(\swarrow\vee\searrow\right)r_{k}$ then $r_{i}\left(\swarrow\vee\searrow\right)r_{k}$.
	Part II: Moreover, $\boldsymbol{K}_{ij}=\pm1$ if $r_{j}$ can be
	generated from $r_{i}^{\pm}$. As a result, any new decision rule
	$r$ added to a leaf must be generated within its corresponding decision
	region.
	\item \emph{Ancestral constraint} (Uniqueness of the ancestry relation):
	For any pair of splitting rules $r_{i}$ and $r_{j}$, only one of
	the following three cases is true: $r_{i}\swarrow r_{j}$, $r_{i}\searrow r_{j}$,
	and $r_{i}\overline{\left(\swarrow\vee\searrow\right)}r_{j}$; additionally,
	$r_{i}\overline{\left(\swarrow\vee\searrow\right)}r_{i}$ is always
	true; in other words, the possible value of $\boldsymbol{K}_{ij}\in\left\{ 1,0,-1\right\} $
	is unique determined for all $i,j$ , and $\boldsymbol{K}_{ii}=0$
	for all $i$.
\end{enumerate}
Note that the second part of Axiom 3 follows directly from combining
Axiom 1 with the first part of Axiom 3. For logical splitting rules,
however, this condition can be ignored, since we have assumed in Definition
\ref{def:Common-features-in} that any logical splitting rule can
be generated from any subspace.. Moreover, although the ancestry relation
satisfies transitivity, it is \emph{not} a \emph{preorder}\footnote{A preorder is a binary relation satisfying reflexivity and transitivity.},
as it fails to satisfy the reflexive property due to Axiom 4---no
rule can be the ancestor of itself in a decision tree.

We define the \textbf{size} of a decision tree as\emph{ the number
	of splitting rules} it possesses, for example, a size three decision
tree with three splitting rules is rendered as

\begin{center}
	\begin{tikzpicture}[level distance=1.5cm, edge from parent/.style={draw,-latex}]
		\node[circle, draw] {$r_1$} 
		[sibling distance=6cm] 
		child {node[circle, draw] {$r_2$}
			[sibling distance=3.5cm] 
			child {node {${r_1^+\cap r_2^+}$} edge from parent node[left] {${r_2^+}$}}
			child {node {${r_1^+\cap r_2^-}$} edge from parent node[right] {${r_2^-}$}}
			edge from parent node[left] {${r_1^+}$}}
		child {node[circle, draw] {$r_3$}
			[sibling distance=3.5cm] 
			child {node {${r_1^-\cap r_3^+}$} edge from parent node[left] {${r_3^+}$}}
			child {node {${r_1^-\cap r_3^-}$} edge from parent node[right] {${r_3^-}$}}
			edge from parent node[right] {${r_1^-}$}};
	\end{tikzpicture}
\end{center}

\paragraph{Why axiomatic definitions are important}

In proper decision tree axioms \ref{axioms: proper decision tree},
we classify Axioms 1-3 as \textbf{structural constraints}, which formalize
the common features of decision trees given in Definition \ref{def:Common-features-in}.
These constraints define the combinatorial structure of decision trees
by specifying the roles of internal and leaf nodes. By contrast, the\textbf{
	ancestral constraint} (Axiom 4) governs how internal nodes are connected
to one another.

Another way of explaining this: the structural constraints ensure
we are referring to a decision tree problem as commonly understood,
while the ancestral constraint specifies which \emph{class} of decision
tree problem is under consideration. Replacing Axiom 4 with an alternative
Axiom 4' yields a different class of decision tree problems. In the
extreme case, Axiom 4 may be omitted entirely. For instance, as shown
in Section \ref{sec:Extension-to-non-proper}, the ODT-BF problem
has no ancestral constraints on splitting rules, reducing the structure
to ordinary binary labeled trees. This, in turn, requires a different
algorithmic procedure with substantially higher complexity. Similarly,
in Subsection \ref{subsec:Non-exhua-Murtree} we will see examples
of replacing Axiom 4 with different Axioms 4' and 4''.

Moreover, \emph{Axiom 4 implicitly determines the algorithmic structure
	of the ODT algorithm}. This key observation allows us to reduce the
analysis of decision tree algorithms to their ancestral constraints,
providing substantial flexibility in addressing various decision tree
problems. As a result, our framework not only offers an unambiguous
way of classifying decision tree problems but also provides a novel
perspective for future research on the algorithmic structure of ODT
problems.

\paragraph{The applicability of proper decision trees}

\begin{figure}[h]
	\begin{centering}
		\includegraphics[viewport=0bp 80bp 1280bp 750bp,clip,scale=0.25]{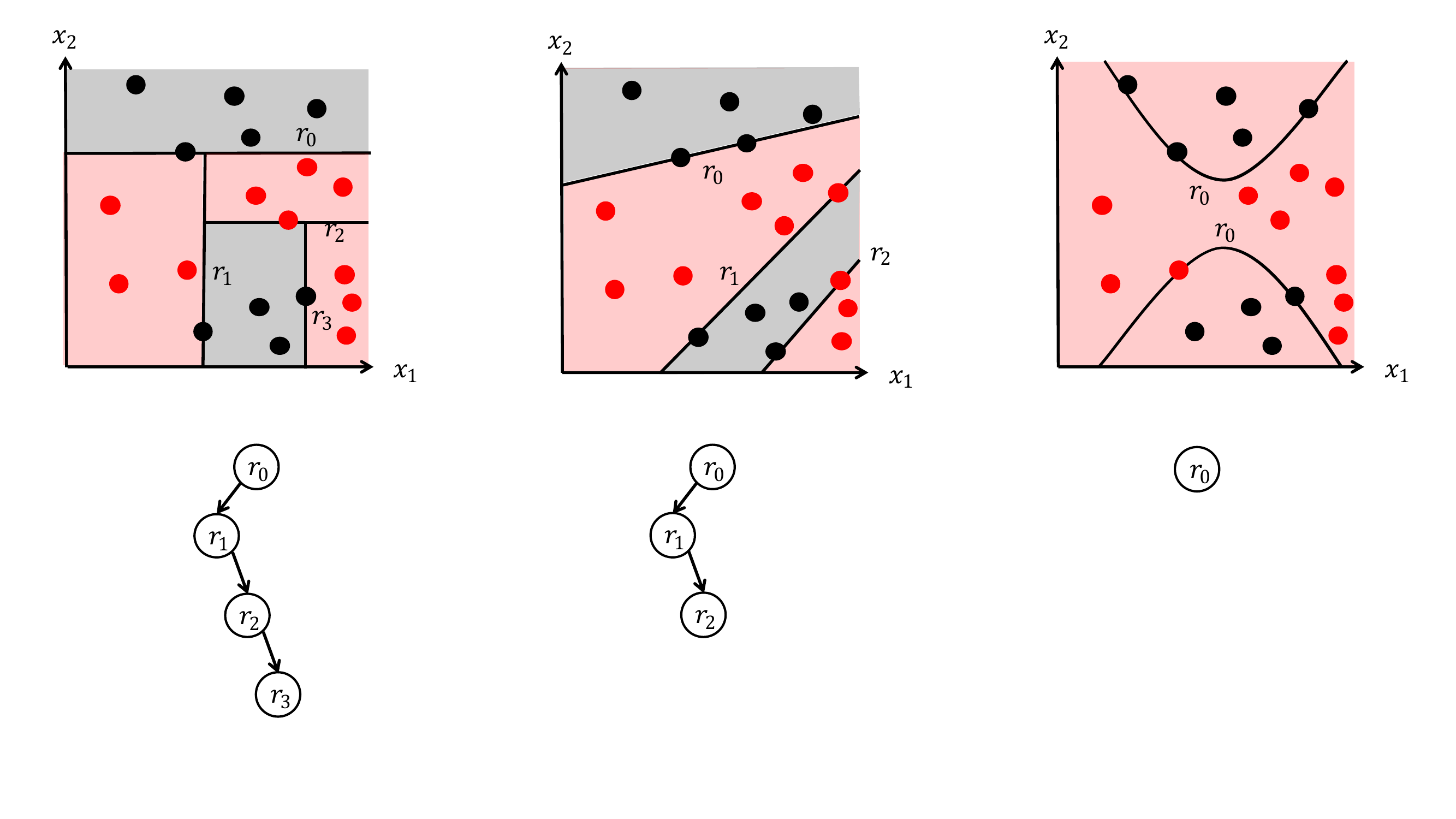}
		\par\end{centering}
	\caption{An axis-parallel decision tree model (left), a hyperplanes (oblique)
		decision tree model (middle), and a hypersurface (defined by degree-$2$
		polynomials) decision tree model (right), characterized by one, two,
		and five points, respectively. As the complexity of the splitting
		functions increase, the tree's complexity decreases (involving fewer
		splitting nodes). \label{fig:DT in ML}}
\end{figure}

Some readers may question the applicability of the proper decision
tree axioms, particularly Axiom 4. We adopt Axiom 4 as the definition
of an ancestral constraint because of its \textbf{versatility} and
\textbf{effectiveness}.

From the perspective of \emph{versatility}, for most optimal decision
tree problems concerned with data partitions, the key concern is not
the explicit representation of splitting rules but how these rules
partition the data. In such cases, it can be proved that partitions
induced by axis-parallel hyperplanes, general hyperplanes, or even
hypersurfaces can be characterized directly using data points \citep{he2025ROF},
see Figure \ref{fig:DT in ML} for example. Under this constraint,
it is reasonable to assume that if a hyperplane $h_{i}$ characterized
by a set of data points, lies on one side of another hyperplane $h_{j}^{\pm}$
then it \emph{cannot} simultaneously lie on the opposite side $h_{j}^{\mp}$.
Another example is the \emph{binary space partition problem} (see
Subsection \ref{subsec:Binary-space-partition}), where hyperplanes
are defined by the \emph{affine flats }spanned by \emph{polygons}
(in $\mathbb{R}^{3}$) or \emph{line segments} (in $\mathbb{R}^{2}$),
and it is straightforward to show that these \emph{convex geometric
	objects} must also lie entirely on one side of the hyperplane. Lastly,
the $K$-D tree data structure is similar to the axis-parallel decision
tree, except that it imposes the additional constraint that all splitting
rules at the same level must be chosen from the same dimension, see
Figure \ref{fig: K-D tree illustrate} illustrating this point.

From the perspective of \emph{effectiveness}, the ancestral constraint
imposed by Axiom 4 drastically reduces the search space of decision
trees. As noted in the introduction, the number of possible size-$K$
labeled binary decision trees is $K!\times\mathit{Catalan}\left(K\right)$.
By introducing Axiom 4, we show in the next Subsection that proper
decision trees have a \emph{maximal} complexity of only $K!$. This
greatly decrease the combinatorial complexity of the problem, for
instance, the four hyperplanes given in \ref{fig:possible_decision_trees}
yields only \emph{three} possible decision trees while the number
of possible binary labeled trees is $4!\times\mathit{Catalan}\left(4\right)=336$.

\subsubsection{Proper decision trees and $K$-permutations}

\begin{figure}[H]
	\begin{center}
		\begin{tikzpicture}
			\node[draw, circle] (r1) at (0, 0) {$r_1$};
			\node[draw, circle] (r2) at (-1.5, -1.5) {$r_2$};
			\node[draw, circle] (r3) at (1.5, -1.5) {$r_3$};
			\draw (r1) -- (r2);
			\draw (r1) -- (r3);
		\end{tikzpicture}
	\end{center}
	
	\caption{A decision tree with three splitting rules, corresponds to 3-permutation
		$\left[r_{1},r_{2},r_{3}\right]$.\label{fig:A decision tree with three splitting rules}}
\end{figure}

The axioms of proper decision trees defined above, enable the analysis
of their algorithmic and combinatorial properties. One of the most
important combinatorial properties discussed in this paper is that
any proper decision tree can be \emph{uniquely} characterized as a
$K$-permutation through a \emph{level-order traversal} of the tree.

Tree traversal refers to the process of visiting or accessing each
node of the tree \emph{exactly} \emph{once} in a specific order. Level-order
traversal visits all nodes at the same level before moving on to the
next level. The main idea of level-order traversal is to visit all
nodes at higher levels before accessing any nodes at lower levels,
thereby establishing a hierarchy of nodes between levels. For example,
the level-order traversal for the binary tree in Figure \ref{fig:A decision tree with three splitting rules}
has two possible corresponding $3$-permutations, $\left[r_{1},r_{2},r_{3}\right]$
or $\left[r_{1},r_{3},r_{2}\right]$. If we fix a traversal order
such that the \emph{left subtree is visited before the right subtree},
only one arrangement of rules can exist. Based on the axioms of proper
decision trees, we can state the following theorem about the level-order
traversal of a proper decision tree.
\begin{theorem}
	\emph{Given a level-order traversal of a decision tree $\left[\ldots,r_{i},\ldots,r_{j},\ldots,r_{k},\ldots\right]$,
		if} \emph{$r_{j}$ precedes $r_{k}$ in the traversal, and $r_{i}$
		is the ancestor of $r_{j}$ and $r_{k}$, then either:}
	
	\emph{1. $r_{j}$ and $r_{k}$ are at the same level, or}
	
	\emph{2. $r_{k}$ is a descendant of $r_{j}$.}
	
	\emph{Only one of the two cases can occur. If the first case holds,
		they are the left or right children of another node, and their positions
		cannot be exchanged.\label{thm: main theorem}}
\end{theorem}
\begin{proof}
	We prove this by contradiction. Assume, by contradiction, $r_{j}$
	is in the same level as $r_{k}$ and $r_{k}$ can be a descendant
	of $r_{j}$. Suppose we have a pair of rules $r_{k}$ and $r_{j}$,
	where $r_{j}$ precedes $r_{k}$ in the level-order traversal.
	
	Case 1: Assume $r_{j}$ and $r_{k}$ are at the same level, we prove
	that $r_{k}$ cannot be the descendant of $r_{j}$. Because of Axiom
	3, if $r_{j}$ is in the same level of $r_{k}$, then $r_{j}$ and
	$r_{k}$ are generated from different branches of some ancestor $r_{i}$,
	which means that they lie in two disjoint regions defined by $r_{i}$.
	If $r_{k}$ is the descendant of $r_{j}$ then it is also a \emph{left-descendant}
	of $r_{i}$ due to associativity. According to Axiom 4, either $r_{j}$
	is a left child of $r_{i}$ or right children of $r_{i}$ it can not
	be both. This leads to a contradiction, as it would imply $r_{k}$
	belongs to both disjoint subregions defined by $r_{i}$.
	
	Case 2:\emph{ }Assume $r_{k}$ is a descendant of $r_{j}$, we prove
	that $r_{j}$ and $r_{k}$ cannot be at the same level. By the transitivity
	of the ancestry relation, both $r_{k}$ and $r_{j}$ are descendants
	of the parent node (\emph{immediate} \emph{ancestor}) of $r_{j}$,
	which we call $r_{i}$. Since, $r_{k}$ cannot be both the right-
	and left-child of $r_{i}$ at the same time, as it must either be
	the left-child or right-child according to Axiom 4. So $r_{k}$ and
	$r_{j}$ can not be in the same level if $r_{k}$ is a descendant
	of $r_{j}$.
	
	Thus, if $r_{j}$ precedes $r_{k}$ in the level-order traversal,
	this either places them at the same level \emph{or} establishes an
	ancestor-descendant relationship between them, but not both.
\end{proof}
An immediate consequence of the above theorem is that any $K$-permutation
of rules corresponds to the level-order traversal of \emph{at} \emph{most}
\emph{one} proper decision tree. The ordering between any two adjacent
rules corresponds to only one structure: either they are on the same
level, or one is the ancestor of the other. For instance, in Figure
\ref{fig:A decision tree with three splitting rules}, if $r_{2}$
and $r_{3}$ are in the same level and $r_{2}$ is the left-child
of $r_{1}$, then $r_{3}$ cannot be the child of $r_{2}$, because
it cannot be the left-child of $r_{1}$. Hence, only a proper decision
tree corresponds to the permutation $\left[r_{1},r_{2},r_{3}\right]$.
Therefore, once a proper decision tree is given, we can obtain its
$K$-permutation representation easily by using a level-order traversal.
\begin{corollary}
	\emph{A decision tree consisting of $K$ splitting rules corresponds
		to a unique $K$-permutation if and only if it is proper. In other
		words, there exist an injective mapping from proper decision trees
		to $K$-permutations.}
\end{corollary}
\begin{proof}
	\textbf{Sufficiency}: If a decision tree is proper, its level-order
	traversal yields a unique valid $K$-permutation by level-order traversal,
	which implies it is a proper decision tree.
	
	Part 1: Existence of mapping.
	
	Because the level-order traversal algorithm is deterministic and the
	tree's structure is fixed (each branch node has a fixed position,
	left or right child of its parent), any binary tree can be transformed
	into a $K$-permutations, by a level-order traversal.
	
	Part 2: The mapping is injective.
	
	Denote $\left(r_{1},r_{2},\ldots,r_{K-1},r_{K}\right)=\left[r_{2},\ldots,r_{K-1}\right]$
	as the partial permutation of rules obtained from $\left[r_{1},r_{2},\ldots,r_{K-1},r_{K}\right]$
	by excluding the first and last elements. For example $\left(r_{1},r_{2},r_{3},r_{4}\right)=\left[r_{2},r_{3}\right]$.
	
	Assume we have two proper decision trees $T_{1}$ and $T_{2}$, with
	corresponding permutations $\sigma_{1}$ and $\sigma_{2}$. To prove
	injectivity, we need to show that 
	\begin{equation}
		\forall T_{1},T_{2}:T_{1}\neq T_{2}\implies\sigma_{1}\neq\sigma_{2}
	\end{equation}
	We proceed by proving $\exists T_{1},T_{2}:T_{1}\neq T_{2}\implies\sigma_{1}=\sigma_{2}$
	will lead to contradiction. Suppose there exist two distinct trees
	$T_{1}\neq T_{2}$ corresponding to the same permutation $\left[\ldots,r_{1},\ldots,r_{2}\ldots\right]$.
	The only way this could happen is if $r_{1}$ and $r_{2}$ were originally
	at the same level $i$, in $T_{1}$ (so that all rules $\left(r_{1},\ldots,r_{2}\right)$
	lie at the same level), but in $T_{2}$, $r_{2}$ becomes a descendant
	of some rule $r$ in $\left(r_{1},\ldots,r_{2}\right)$ at a lower
	level $j>i$, while all subsequent rules in the permutation remain
	unchanged. We will show that this violates Axiom 4.
	
	In $T_{1}$ , $r_{2}$ is not the first element at that level; any
	rule $r\in\left(r_{1},\ldots,r_{2}\right)$ preceding $r_{2}$ is
	an older sibling. Consequently, there exists a rule $r_{3}$ such
	that $r$ lies in the left subtree and $r_{2}$ lies in the right
	subtree of $r_{3}$. The only scenario in which $T_{1}\neq T_{2}$
	but $\sigma_{1}=\sigma_{2}$ would occur is if $r_{2}$ becomes a
	descendant of some $r\in\left(r_{1},\ldots,r_{2}\right)$ at level
	$j>i$. However, by the level-order traversal definition, rules at
	the same level are scanned from left to right. If $r_{2}$ becomes
	a descendant of $r$ in $T_{2}$, then since $r$ is the left descendant
	of $r_{3}$, it follows by transitivity that $r_{2}$ would become
	a left descendant of $r_{3}$. This contradicts Axiom 4, because in
	$T_{1}$, $r_{2}$ is the right descendant of $r_{3}$.
	
	Hence, $T_{2}$ cannot coexist with $T_{1}$. This contradiction implies
	that $\forall T_{1},T_{2}:T_{1}\neq T_{2}\implies\sigma_{1}\neq\sigma_{2}$,
	and therefore the mapping is injective.
	
	\textbf{Necessity}: The mapping from decision trees to $K$-permutations
	is not injective if the trees are non-proper. Specifically, distinct
	non-proper decision trees can correspond to identical permutations,
	rendering the mapping non-unique and thus non-injective.
	
	Consider a permutation\emph{ $\left[\ldots,r_{i},\ldots,r_{j},\ldots,r_{k},\ldots\right]$}.
	In the absence of Axioms 4, the structure of the tree is not sufficiently
	constrained: $r_{j}$ and $r_{k}$ may reside at the same level, or
	$r_{j}$ is the ancestor of $r_{k}$, within the same permutation.
	Consequently, this permutation can be realized by at least two distinct
	non-proper trees $T_{1}$ and $T_{2}$. Since $T_{1}\neq T_{2}$,
	this gives us a non-injective map.
\end{proof}
The injectivity between $K$-permutations and proper decision trees
implies that the number of $K$-permutations strictly exceeds the
number of possible proper decision trees, leading to the following
lemma.
\begin{lemma}
	\emph{Given a size $K$ rule list $\mathit{rs}_{K}$, the complexity
		of the search space of proper decision trees $\left|\mathcal{S}\left(K,\mathit{rs}_{K}\right)\right|$
		is strictly smaller than $K!$.}
\end{lemma}
We refer to a $K$-permutation obtained from a proper decision tree
via level-order traversal as a \emph{valid} $K$-permutation. Note
that the precise size of \emph{$\left|\mathcal{S}\left(K,\mathit{rs}_{K}\right)\right|$}
is unpredictable, as it depends on the values in $\boldsymbol{K}$.
In Subsection \ref{subsec:Complexity-of-the}, we will show that the
maximal complexity is attained when $\boldsymbol{K}$ follows a special
format, which corresponds to a tree with ``chain-like'' structure,
in which case \emph{$\left|\mathcal{S}\left(K,\mathit{rs}_{K}\right)\right|=K!$}.

\subsection{Datatypes, homomorphisms and map functions}

Before explaining the main algorithmic results of this paper, this
section introduces terminology and basic functions that will be used
throughout the paper when reasoning about programs.

\subsubsection{Function composition and partial application}

The function composition is denoted by using infix symbol $\circ$:
\[
\left(f\circ g\right)\left(a\right)=f\left(g\left(a\right)\right).
\]

We will try to use infix binary operators wherever possible to simplify
the discussion. A binary operator can be transformed into a unary
operator through \emph{partial} \emph{application} (also known as
\emph{sectioning}). This technique allows us to fix one argument of
the binary operator, effectively creating a unary function that can
be applied to subsequent values
\[
\oplus_{a}\left(b\right)=a\oplus b=\oplus^{b}\left(a\right),
\]
where $\oplus$ is a binary operator, such as $+$ or $\times$ for
numerical operations, which can be partially applied. If $\oplus:\mathcal{A}\times\mathcal{B}\to\mathcal{C}$,
then $\oplus_{a}:\mathcal{B}\to\mathcal{C}$, where $a:\mathcal{A}$
is fixed. Moreover, the variables $a$ and $b$ can be functions,
we will see examples shortly when we explain the map function.

\subsubsection{Datatypes}

There are two binary tree data types used in this paper. The first
is the \emph{branch-labeled} tree, which is defined as

\[
\mathit{BTree}\left(\mathcal{A}\right)=L\mid N\left(\mathit{BTree}\left(\mathcal{A}\right),\mathcal{A},\mathit{BTree}\left(\mathcal{A}\right)\right).
\]
The symbol ``$\mid$'' is used to separate different \emph{constructors}
for that type. Each constructor represents a different way to build
a value of that type. The above definition consists of two constructors,
which states that a binary tree $\mathit{BTree}\left(\mathcal{A}\right)$
is either an \emph{unlabeled} leaf $L$ \emph{or} a binary tree $N\left(u,a,v\right):\mathit{BTree}\left(\mathcal{A}\right)$,
where $u,v:\mathit{BTree}\left(\mathcal{A}\right)$ are the left and
right subtrees, respectively, and $a:\mathcal{A}$ is the root node.

An alternative tree definition, called a \emph{moo-tree} \citep{gibbons1991algebras}, is named for its phonetic resemblance to the Chinese word for ``tree'',
is a binary tree in which leaf nodes and branch nodes have different
types. Formally, it is defined as 
\[
\mathit{MTree}\left(\mathcal{A},\mathcal{B}\right)=DL\left(\mathcal{B}\right)\mid DN\left(\mathit{MTree}\left(\mathcal{A},\mathcal{B}\right),\mathcal{A},\mathit{MTree}\left(\mathcal{A},\mathcal{B}\right)\right).
\]

A decision tree is a special case of the moo-tree datatype, which
we can define as $\mathit{DTree}\left(\mathcal{R},\mathcal{D}\right)=\mathit{MTree}\left(\mathcal{R},\mathcal{D}\right)$,
i.e. a decision tree is a moo-tree where branch nodes are splitting
rules and leaf nodes are subsets of the dataset $\mathcal{D}$.

\subsubsection{Homomorphisms and fusion}

Homomorphisms are functions that \emph{fuse} \emph{into} (or \emph{propagate}
\emph{through}) type constructors, preserving their structural composition.
In functional programming and algebraic approaches to program transformation,
homomorphisms enable efficient computation by fusing recursive structures
into more compact representations.

For example, given a binary operator $\oplus$ and an identity element
$e$, a homomorphism over a list is defined as follows:
\[
\begin{aligned}h & \left(\left[\,\right]\right)\:\;\quad=e\\
	h & \left(x\cup y\right)=h\left(x\right)\oplus h\left(y\right).
\end{aligned}
\]

Here, $h$ denotes the homomorphism function, where the identity element
maps the empty list, and the operation on the concatenation of two
lists ($x\cup y$) is the result of applying the binary operator $\oplus$
to the images of the two lists.

Similarly, given $f$ and $y$, there exists a unique homomorphism
$h$, such that, for all $u$, $r$ and $v$, the equations

\begin{align*}
	h & \left(\mathit{DL}\left(a\right)\right)\,\:\;\;\quad=f\left(a\right)\\
	h & \left(\mathit{DN}\left(u,r,v\right)\right)=g\left(h\left(u\right),r,h\left(v\right)\right),
\end{align*}
hold. The homomorphism $h$ replaces every occurrence of the constructor
$\mathit{DL}$ with the function $f$, and every occurrence of the
constructor $\mathit{DN}$ with the function $g$, which is essentially
a ``\emph{relabelling}'' process. Once a homomorphism is identified,
a number of \emph{fusion} or \emph{distributivity} laws \citep{bird1996algebra,little2021dynamic}
can be applied to reason about the properties of the program.

\subsubsection{The map functions}

One example of fusion is the map function. Given a list $x$ and a
unary function $f$, the map function over the list, denoted $\mathit{mapL}$,
can be defined as
\[
\mathit{mapL}\left(f,x\right)=\left[f\left(a\right)\mid a\leftarrow x\right].
\]

This definition corresponds to a standard list comprehension, meaning
that the function $f$ is applied to each element $a$ in $x$. By
using sectioning, $\mathit{mapL}$ can be partially applied by defining
the unary operator $\mathit{mapL}{}_{f}\left(x\right)=\mathit{mapL}\left(f,x\right)$.
Similarly, the map function can be defined over a decision tree as
follows
\[
\begin{aligned}\mathit{mapD} & \left(f,\mathit{DL}\left(a\right)\right)=\mathit{DL}\left(f\left(a\right)\right)\\
	\mathit{mapD} & \left(f,\mathit{DN}\left(u,r,v\right)\right)=DN\left(\mathit{mapD}\left(f,u\right),r,\mathit{mapD}\left(f,v\right)\right).
\end{aligned}
\]

This applies the function $f$ to every \emph{leaf} of the tree.

\section{A generic dynamic programming algorithm for solving the optimal,
	proper decision tree problem \label{sec: Decision-tree-problem specification}}

\subsection{Specifying the optimal proper decision tree problem through $K$-permutations\label{subsec:tree datatype decision tree generatr}}

The goal of the ODT problem is to construct a function $\mathit{odt}:\mathbb{N}\times\left[\mathcal{R}\right]\to\mathit{DTree}\left(\mathcal{R},\mathcal{D}\right)$
that outputs the optimal decision tree with $K$ splitting rules.
This can be specified as
\begin{equation}
	\begin{aligned}\mathit{odt}_{K} & :\left[\mathcal{R}\right]\to\mathit{DTree}\left(\mathcal{R},\mathcal{D}\right)\\
		\mathit{odt}_{K} & =\mathit{min}_{E}\circ\mathit{genDTKs}_{K}.
	\end{aligned}
	\label{odt_k --orginal}
\end{equation}

The function $\mathit{genDTKs}_{K}$, short for ``generate decision
trees with $K$ splitting rules'', generates all possible decision
trees of size $K\geq1$ from a given input of splitting rules $\mathit{rs}:\left[r_{1},r_{2},\ldots r_{M}\right]$,
where $M\geq K$, which is precisely the programmatic counterpart
of $\mathcal{S}\left(K,\mathit{rs}\right)$ given in (\ref{eq: ODT-specification-MIP}).

The above specification of $\mathit{odt}$ is essentially a brute-force
program, i.e. it exhaustively generates all possible decision trees
of size $K$ using $\mathit{genDTKs}_{K}$ and then selects the best
one using $\mathit{min}_{E}$. The function $\mathit{min}:\left(\mathit{DTree}\left(\mathcal{R},\mathcal{D}\right)\to\mathbb{R}\right)\times\left[\mathit{DTree}\left(\mathcal{R},\mathcal{D}\right)\right]\to\mathit{DTree}\left(\mathcal{R},\mathcal{D}\right)$
selects an optimal decision tree from a (assume, non-empty) list of
candidates based on the objective value calculated by $E:\mathit{DTree}\left(\mathcal{R},\mathcal{D}\right)\to\mathbb{R}$
(defined in Subsection \ref{subsec: DP algorithm}), which is given
by
\begin{equation}
	\begin{aligned}\mathit{min}_{E} & \left(\left[a\right]\right)=a\\
		\mathit{min}_{E} & \left(a:\mathit{as}\right)=\mathit{smaller}_{E}\left(a,min_{E}\left(\mathit{as}\right)\right),
	\end{aligned}
\end{equation}
Note that symbol ``$:$'' is overloaded: When appearing in code expressions,
such as above, $:$ denotes the \emph{list-construction} (cons) operator,
which prepends an element to an existing list. When appearing in type
signatures, the same symbol : specifies the type of an expression.

The function $\mathit{smaller}_{E}$ is defined as follows 
\begin{equation}
	\mathit{smaller}_{E}\left(a,b\right)=\begin{cases}
		a, & \text{if \ensuremath{E\left(a\right)\leq E\left(b\right)}}\\
		b, & \text{otherwise}.
	\end{cases}
\end{equation}

As discussed in the previous section on the combinatorial structure
of proper decision trees, each proper decision tree can be uniquely
represented by a $K$-permutation of splitting rules via a level-order
traversal. Accordingly, one possible definition of $\mathit{genDTKs}_{K}$
is provided below.
\begin{definition}
	\emph{Exhaustive proper decision tree generator based on $K$-permutations}.
	Given a list of splitting rules $\mathit{rs}:\left[r_{1},r_{2},\ldots r_{M}\right]$,
	where $M\geq K$, all possible size $K$ decision trees in search
	space $\mathcal{S}\left(K,\mathit{rs}\right)$ can be exhaustively
	enumerated by the program
\end{definition}
\begin{equation}
	\begin{aligned}\mathit{genDTKs}_{K} & :\left[\mathcal{R}\right]\to\left[\left[\mathcal{R}\right]\right]\\
		\mathit{genDTKs}_{K} & =\mathit{filter}_{p}\circ\mathit{kperms}_{K}.
	\end{aligned}
	\label{eq: DT generator through k-perms}
\end{equation}
where the predicate $p:\left[\mathcal{R}\right]\to\mathit{Bool}$
checks whether an input permutation is valid. The function $\mathit{filter}_{p}$
removes all non-valid permutations returned by $\mathit{kperms}_{K}$.

As noted, $K$-permutations are among the one of the best understood
combinatorial objects \citep{sedgewick1977permutation}. One possible
definition of a $K$-permutation generator is defined through a\emph{
	$K$-combination} generator---$K$-permutations are simply all possible
rearrangements (permutations) of each $K$-combination. In other words,
we can define $\mathit{kperms}_{K}$ by the following 
\begin{equation}
	\mathit{kperms}{}_{K}=\mathit{concatMapL}{}_{perms}\circ\mathit{kcombs}_{K},\label{eq: definition of K-permutation}
\end{equation}
where $\mathit{concatMapL}_{\mathit{perms}}=\mathit{concat}\circ\mathit{mapL}_{\mathit{perms}}$,
which first applies $\mathit{perms}:\left[\mathcal{A}\right]\to\left[\left[\mathcal{A}\right]\right]$
(generating all possible permutations of a given list) to each list
within a list of lists, and then uses the flatten operation $\mathit{concat}:\left[\left[\mathcal{A}\right]\right]\to\left[\mathcal{A}\right]$
which collapses the inner lists into a single list. Thus, by substituting
the definition of (\ref{eq: definition of K-permutation}) and (\ref{eq: DT generator through k-perms})
into (\ref{odt_k --orginal}), we obtain,

\begin{equation}
	\begin{aligned}\mathit{odt}{}_{K} & :\left[\mathcal{R}\right]\to\left[\mathcal{R}\right]\\
		\mathit{odt}{}_{K} & =\mathit{mi}n{}_{E}\circ\mathit{filter}{}_{p}\circ\mathit{concat}\circ\mathit{mapL}{}_{perms}\circ\mathit{kcombs}{}_{K},
	\end{aligned}
	\label{odt_k -- v2}
\end{equation}
and $\mathit{genDTKs}_{K}=\mathit{filter}{}_{p}\circ\mathit{concat}\circ\mathit{mapL}{}_{perms}\circ\mathit{kcombs}{}_{K}$.
Since $\mathit{kcombs}{}_{K}$ produces only $\left(\begin{array}{c}
	M\\
	K
\end{array}\right)=O\left(M^{K}\right)$ $K$-combinations, and $K$-permutations are all possible permutations
of each $K$-combinations, we have $\left|\mathit{kperms}_{K}\right|=K!\times\left(\begin{array}{c}
	M\\
	K
\end{array}\right)=O\left(M^{K}\right)$. This already provides a polynomial-time algorithm for solving the
ODT problem when $M$ and $K$ are fixed constants (assuming that
the predicate $p$ has polynomial complexity and $\mathit{min}_{E}$
is linear in the candidate list size). In Subsection \ref{subsec: ODT in ML},
we show that the number of possible rules $M$ grows polynomially
with the input data size for splitting rules considered in ML.

\subsection{A simplified decision tree problem: the decision tree problem with
	$K$ fixed splitting rules (branch nodes) \label{subsec: Deriving simplified odt problem using equational reasoning}}

\subsubsection{Simplified optimal decision tree problem}

In the expanded specification (\ref{odt_k -- v2}), we have obtained
\begin{equation}
	\mathit{genDTKs}_{K}=\mathit{filter}{}_{p}\circ\mathit{concat}\circ\mathit{mapL}{}_{perms}\circ\mathit{kcombs}{}_{K},\label{eq: expanded specification of genDTKs}
\end{equation}
which suggests an interesting fusion in the composed function $\mathit{filter}{}_{p}\circ\mathit{concat}\circ\mathit{mapL}{}_{perms}$
formalized by the following theorem.
\begin{theorem}
	\emph{The $\mathit{genDTKs}_{K}$ generator defined in (\ref{eq: expanded specification of genDTKs})
		is equivalent to 
		\begin{equation}
			\mathit{genDTKs}_{K}=\mathit{concatMapL}{}_{genDTs}\circ\mathit{kcombs}_{K}\label{eq: new-specification of genDTKs}
		\end{equation}
		where $\mathit{genDTs}=\mathit{filter}_{p}\circ\mathit{perms}$.}
\end{theorem}
\begin{proof}
	This can be proved simply by the following equational reasoning
	\begin{align*}
		& \mathit{filter}{}_{p}\circ\mathit{concat}\circ\mathit{mapL}{}_{perms}\circ\mathit{kcombs}{}_{K}\\
		\equiv & \text{ filter fusion laws \ensuremath{\mathit{filter}_{p}\circ\mathit{concat}=\mathit{concat}\circ\mathit{mapL}_{\mathit{filter}_{p}}}}\\
		& \mathit{concat}\circ\mathit{mapL}{}_{\mathit{filter}_{p}}\circ\mathit{mapL}{}_{perms}\circ\mathit{kcombs}{}_{K}\\
		\equiv & \text{ map composition \ensuremath{\mathit{mapL}_{f}\circ\mathit{mapL}_{g}=\mathit{mapL}_{f\circ g}}}\\
		& \mathit{concat}\circ\mathit{mapL}{}_{\mathit{filter}_{p}\circ\mathit{perms}}\circ\mathit{kcombs}{}_{K}\\
		\equiv & \text{ define \ensuremath{\mathit{genDTs}=\mathit{filter}_{p}\circ\mathit{perms}}}\\
		& \mathit{concatMapL}{}_{genDTs}\circ\mathit{kcombs}{}_{K},
	\end{align*}
	where the laws used in above derivation can all be found in \citet{bird1987introduction}
	(since these laws are easy to verify and intuitively obvious, we do
	not repeat their proofs here).
\end{proof}
Equation (\ref{eq: new-specification of genDTKs}) illustrates the
simple fact that generating size $K$ decision trees from an input
list $\mathit{rs}:\text{\ensuremath{\left[\mathcal{R}\right]} }$
can be decomposed into generating decision trees over smaller instances
$\mathit{rs}_{K}$--- the $K$-combinations of splitting rules---using
a (possibly) more efficient program produced by $\mathit{genDTs}$.

Consequently, substituting (\ref{eq: new-specification of genDTKs})
into the definition of $\mathit{odt}_{K}$, yields

\begin{equation}
	\begin{aligned}\mathit{odt}{}_{K} & :\left[\mathcal{R}\right]\to\mathit{DTree}\left(\mathcal{R},\mathcal{D}\right)\\
		\mathit{odt}{}_{K} & =\mathit{min}{}_{E}\circ\mathit{concatMapL}{}_{genDTs}\circ\mathit{kcombs}{}_{K}.
	\end{aligned}
	\label{odt_k -- v3}
\end{equation}

Interestingly, there exists a \emph{distributivity} between the function
$\mathit{smaller}_{E}$ function (which defines $\mathit{min}{}_{E}$)
and the list append operation $\cup$ (which defines $\mathit{concatMapL}$),
namely, $\mathit{smaller}_{E}\left(a,b\right)\cup c=\mathit{smaller}_{E}\left(a\cup c,b\cup c\right)$.
This leads to the following result, which complete the first part
of Theorem \ref{thm: sodt-introduction}.
\begin{theorem}
	Simplified decision tree theorem\emph{. Given the program $\mathit{odt}{}_{K}=\mathit{min}{}_{E}\circ\mathit{concatMapL}{}_{genDTs}\circ\mathit{kcombs}{}_{K}$,
		we have
		\begin{equation}
			\mathit{odt}{}_{K}=\mathit{min}{}_{E}\circ\mathit{concatMapL}{}_{genDTs}\circ\mathit{kcombs}{}_{K}=\mathit{min}{}_{E}\circ\mathit{mapL}_{\mathit{sodt}}\circ\mathit{kcombs}{}_{K}\label{eq: sodt theorem}
		\end{equation}
		where $\mathit{sodt}=\mathit{min}_{E}\circ\mathit{genDTs}$, and $\mathit{sodt}$
		is short for ``}simplified optimal decision tree problem\emph{.''\label{thm: Simplified-decision-tree problem}}
\end{theorem}
\begin{proof}
	Equation (\ref{eq: sodt theorem}) can be proved by following equational
	reasoning
	\[
	\begin{aligned} & \mathit{min}_{E}\circ\mathit{mapL}{}_{genDTs}\circ\mathit{kcombs}{}_{K}\\
		\equiv & \text{ distributivity \ensuremath{\mathit{min}_{E}\circ\mathit{concat}=\mathit{min}_{E}\circ\mathit{mapL}_{\mathit{min}_{E}}}}\\
		& \mathit{min}_{E}\circ\mathit{mapL}{}_{\mathit{min}_{E}\circ genDTs}\circ\mathit{kcombs}{}_{K}\\
		\equiv & \text{ define \ensuremath{\mathit{sodt}=\mathit{min}_{E}\circ\mathit{genDTs}}}\\
		& \mathit{min}_{E}\circ\mathit{mapL}_{\mathit{sodt}}\circ\mathit{kcombs}{}_{K}.
	\end{aligned}
	\]
	
	Using the property $\mathit{smaller}_{E}\left(a,b\right)\cup c=\mathit{smaller}_{E}\left(a\cup c,b\cup c\right)$,
	the distributivity law $\ensuremath{\mathit{min}_{E}\circ\mathit{concat}=\mathit{min}_{E}\circ\mathit{mapL}_{\mathit{min}_{E}}}$
	follows directly. This can be derived using the promotion lemmas in
	\citet{bird1987introduction} or, equivalently, Exercise 7.3 in \citet{bird2020algorithm}.
\end{proof}
The function $\mathit{sodt}\left(\mathit{rs}_{K}\right)$ computes
the optimal decision tree for a given $K$-combination of the rules
$\mathit{rs}_{K}$. By applying $\mathit{sodt}\left(\mathit{rs}_{K}\right)$
to each $K$-combination in ($\mathit{rs}_{K}\in\mathit{kcombs}{}_{K}\left(\mathit{rs}\right)$),
and selecting the best among them, we recover the same optimal solution
as the original ODT problem. In other words, $\mathit{sodt}\left(\mathit{rs}_{K}\right)$
solves the optimal decision tree problem over the search space $\mathcal{S}\left(K,\mathit{rs}_{K}\right)$
(all possible decision trees with respect to $K$ splitting rules).

Furthermore, the new definition of $\mathit{odt}_{K}$ on the right-hand
side of equation (\ref{eq: sodt theorem}) is potentially much more
efficient. For each $\mathit{rs}_{K}$, $\mathit{sodt}\left(\mathit{rs}_{K}\right)$
returns only \textbf{one} optimal tree, so $\mathit{min}{}_{E}$ only
needs to select from the set of optimal trees generated by $\mathit{concatMapL}{}_{genDTs}\circ\mathit{kcombs}{}_{K}$,
which is substantially smaller than the set of all possible size-$K$
decision trees produced by the original $\mathit{concatMapL}{}_{genDTs}\circ\mathit{kcombs}$.

\subsection{Potential speed-up for the simplified decision tree problem \label{subsec:Potential-speed-up-for}}

There is little room left for further speed-up in the overall structure
of the simplified decision tree problem (\ref{eq: sodt theorem}).
To improve efficiency, we must optimize the individual components
of (\ref{eq: sodt theorem}). However, $\mathit{min}_{E}$ offers
limited scope for improvement, as it merely selects the best element
from a list. Similarly, $\mathit{kcombs}{}_{K}$ is already well studied,
with several equivalently efficient definitions for generating combinations
(see Section II.1.2 in \citet{he2025ROF}).

Thus, the main opportunity for optimization lies in $\mathit{sodt}$.
Its definition is given as:
\begin{equation}
	\begin{aligned}\mathit{sodt} & :\left[\mathcal{R}\right]\to\left[\left[\mathcal{R}\right]\right]\\
		\mathit{sodt} & =\mathit{min}{}_{E}\circ\mathit{genDTs}\\
		& =\mathit{min}{}_{E}\circ\mathit{filter}_{p}\circ\mathit{perms}
	\end{aligned}
	\label{sodt --specification}
\end{equation}
As explained in the previous section, $\mathit{sodt}\left(\mathit{rs}_{K}\right)$
first generates all valid permutations of a $K$-combination of rules
$\mathit{rs}_{K}$, then selects the best one using $\mathit{min}{}_{E}$.
This is essentially a brute-force algorithm that addresses the problem
of finding the optimal decision tree with respect to $\mathit{rs}_{K}$.
It naturally invites the following question:
\begin{quote}
	\emph{What is the optimal proper decision tree with respect to search
		space $\mathcal{S}\left(K,\mathit{rs}_{K}\right)$? Does a greedy
		or dynamic programming (DP) solution exist for this problem?}
\end{quote}
It has long been recognized in the constructive programming community
\citep{de1994categories,bird1996algebra} that efficient combinatorial
algorithms---such as BnB, DP, and greedy methods---can be systematically
derived from brute-force specifications such as (\ref{sodt --specification}).
Two key factors are particularly important for constructing an efficient
combinatorial algorithm: (1) \textbf{Minimizer fusion}: fusing the
$\mathit{min}{}_{E}$ function into the definition of $\mathit{genDTs}$,
thereby pruning non-optimal solutions before they are extended. Once
such fusion is proven, an efficient DP algorithm follows \citep{de1994categories,bird1996algebra};
(1) \textbf{Efficient generator}: formulating an efficient definition
for the generator $\mathit{genDTs}$ forms the foundation for constructing
an efficient program. This not only facilitates the proof of fusion
but also leads to a more efficient fused implementation.

In the following discussion, we show that a more efficient definition
of $\mathit{genDTs}$ can be established and that $\mathit{min}{}_{E}$
can indeed be fused into this new definition, yielding an efficient
DP algorithm for $\mathit{sodt}$. Before presenting the derivation,
we first explain why the current definition of $\mathit{genDTs}$
based on the permutation function $\mathit{perms}$ is inefficient,
which in turn points toward a better formulation.

\paragraph{The redundancy of $K$-permutations\label{subsec:The-redundancy-of k permutation}}

\begin{figure}
	\centering{}\includegraphics[viewport=300bp 150bp 900bp 620bp,clip,scale=0.25]{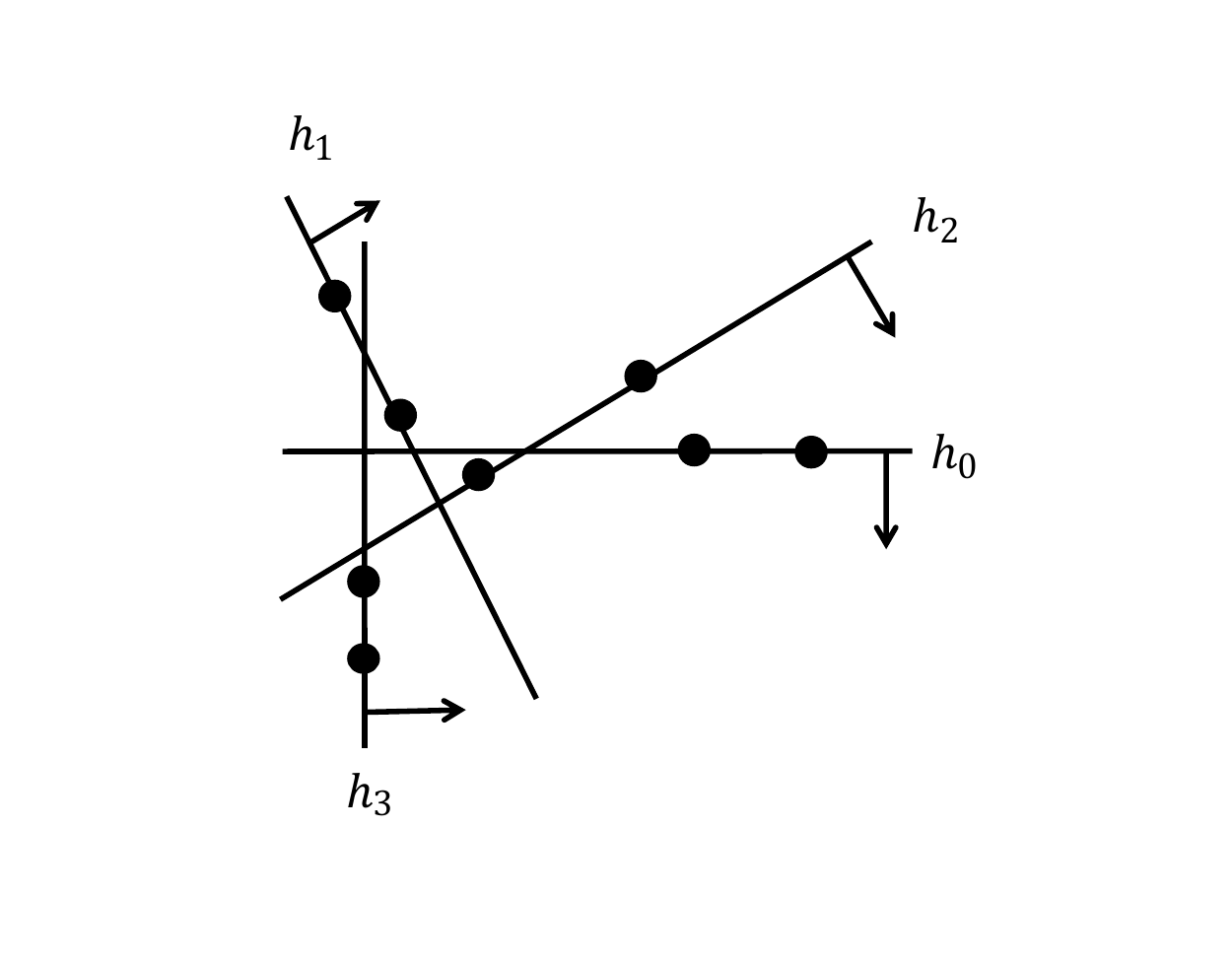}\caption{Four hyperplanes in $\mathbb{R}^{2}$. The black circles represent
		data points used to define these hyperplanes, and the black arrows
		indicate the direction of the hyperplanes. \label{fig:Four-hyperplanes, for generating decisiion trees}}
\end{figure}
\begin{figure}
	\begin{centering}
		\includegraphics[scale=0.25]{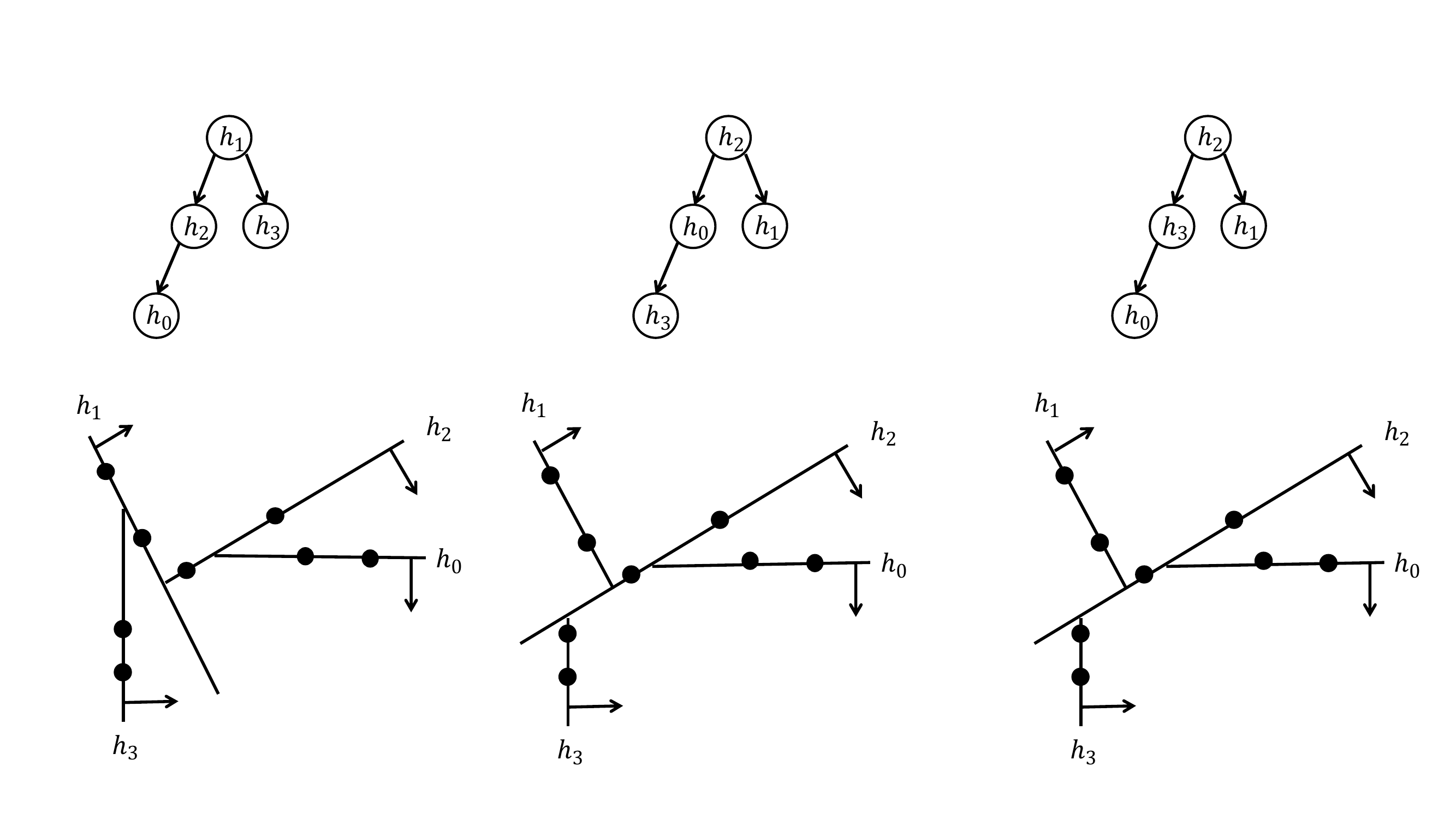}
		\par\end{centering}
	\caption{Three possible decision trees and their corresponding partitions of
		space in $\mathbb{R}^{2}$ for the four hyperplanes depicted in Figure
		\ref{fig:Four-hyperplanes, for generating decisiion trees}. The three
		decision trees above represent all possible decision trees for the
		given four-combination of hyperplanes in Figure \label{fig:possible_decision_trees},
		the three figures below describe the partition of space $\mathbb{R}^{2}$
		resulting from the corresponding decision tree.}
\end{figure}
The number of proper decision trees for a given set of rules is typically
much smaller than the total number of permutations, and determining
whether a permutation is feasible is often non-trivial. To quantify
the redundancy in generated permutations and the complexity of the
feasibility test, we examine a specific case---the hyperplane decision
tree problem---where splitting rules are defined by hyperplanes.

\citet{he2023efficient} showed that hyperplanes in $\mathbb{R}^{D}$
can be constructed using $D$-combinations of data points. Consequently,
within a decision region, hyperplanes can only be generated from the
data points contained in the decision region defined by their ancestors.
This implies that the feasibility test $p$ in the hyperplane decision
tree problem must ensure that all data points defining the hyperplanes
in a subtree remain within the decision region specified by its ancestors.
This requirement imposes a highly restrictive constraint.

To assess the impact of this constraint, we introduce simple probabilistic
assumptions. Suppose each hyperplane classifies a data point into
the positive or negative class with equal probability, i.e. 1/2 for
each class. If a hyperplane can serve as the root of a decision tree,
the probability of this occurring is $\left(\frac{1}{2}\right)^{D\times\left(K-1\right)}$,
since each hyperplane is defined by $D$-combinations of data points.
Furthermore, the probability of constructing a proper decision tree
with a \emph{chain-like }structure (each branch node has at most one
child) is given by 
\[
\left(\frac{1}{2}\right)^{D\times\left(K-1\right)}\times\left(\frac{1}{2}\right)^{D\times\left(K-2\right)}\times\ldots\times\left(\frac{1}{2}\right)^{D}=O\left(\left(\frac{1}{2}\right)^{D\times K^{2}}\right).
\]

In Subsection \ref{subsec:Complexity-of-the}, we will show that when
a decision tree has a chain-like structure, it exhibits the highest
combinatorial complexity. This naive probabilistic assumption provides
insight into the rarity of proper decision trees with $K$ splitting
rules compared to the total number of $K$-permutations of splitting
rules.

For example, the 4-combination of hyperplanes shown in Figure \ref{fig:Four-hyperplanes, for generating decisiion trees}
produces only \emph{three} decision trees, as illustrated in Figure
\ref{fig:possible_decision_trees}, while the total number of possible
permutations is $4!$=24. Interestingly, although there are three
possible trees, there are only two possible partitions. This suggests
a direction for further speeding up the algorithm by eliminating such
cases, as our current algorithm cannot remove these duplicate partitions.

Moreover, the feasibility test is a non-trivial operation. For each
hyperplane, it must be verified whether the data points used to define
it lie within the same region as determined by its ancestor hyperplanes.
Testing whether $D$ points are on the same side of a hyperplane requires
$O\left(D^{2}\right)$ computations. Given that a decision tree is
defined by $K$ hyperplanes, and in the worst case, a hyperplane may
have $K$ ancestors, the feasibility test incurs a worst-case time
complexity of $O\left(K^{2}D^{2}\right)$.

Therefore, to achieve optimal efficiency, it is essential to design
a new proper decision tree generator of type $\mathit{genDTs}:\left[\mathcal{R}\right]\to\left[\mathit{DTree}\left(\mathcal{R},\mathcal{D}\right)\right]$,
which directly generates all possible proper decision trees rather
than first producing valid permutations. This eliminates the need
for post-generation filtering, and ideally, the new generator should
also support fusion. In next few sections, we will explain the design
of such a proper decision tree generator and then demonstrate that
the recursive generator is fusable with the $\mathit{min}_{E}$ function.

\subsection{The partial proper decision tree generator based on the binary tree
	datatype\label{subsec: partial decision tree generator}}

As discussed in previous section, the key part of accelerating (\ref{eq: sodt theorem})
is to find an efficient program for solving $\mathit{sodt}$ , and
this requires us to find a new $\mathit{genDTs}:\left[\mathcal{R}\right]\to\left[\mathit{DTree}\left(\mathcal{R},\mathcal{D}\right)\right]$
function that can exhuastively generate \emph{$\mathcal{S}\left(K,\mathit{rs}_{K}\right)$
	directly, without relying on $K$-permutations, which will be expected
	to be more }efficient than the previous definition of $\mathit{genDTs}=\mathit{filter}_{p}\circ\mathit{perms}$
which generates $K$-permutation first and then eliminates non-viable
$K$-permutations, post-hoc.

Directly constructing a proper decision tree generator is challenging.
However, since the structure of a decision tree is fully determined
by its branch nodes---while the leaf nodes are defined by the intersections
of subspaces formed by the splitting rules along each root-to-leaf
path---we begin by introducing a \emph{partial decision tree generator},
$\mathit{genBTs}$, based on the $\mathit{BTree}\left(\mathcal{R}\right)$
datatype. In the Subsection \ref{subsec: complete decision tree generator},
we extend this construction into a complete decision tree generator
by generalizing \citet{gibbons1996computing}'s \emph{downward accumulation
	technique}, which will be discussed in the next section.

\paragraph{The splits function}

To construct a tree generator, we must address a key question: once
a root $r$ is fixed, which rules in $\mathit{rs}\backslash r$ (i.e.,
excluding $r$ from $\mathit{rs}$) should go to the left subtree
and which to the right? This question is central to nearly all tree-related
problems, as it is closely tied to the recursive structure underlying
tree algorithms.

Since the ancestry relation satisfies transitivity and descendant
rules can be generated only from the subregions of their ancestor
rules (Axiom 3), any splitting rules making up the root of the tree
must be the ancestor of \emph{all} its descendants. In other words,
given $K$ fixed rules $\mathit{rs}_{K}$, if $r_{i}$ is the root,
then $r_{i}\left(\swarrow\vee\searrow\right)r_{j}$, for all $r_{j}\in\mathit{rs}_{K}$
such that $j\neq i$. In the ancestry relation graph, if $r_{i}$
is the root, all edges connected to $r_{i}$ and $r_{j}\in rs_{K}$
for $j\neq i$ must have arrows, either incoming or outgoing. For
example, in Figure \ref{fig:Equivalent-representation-to}, $r_{0}$
cannot be the root of $\mathit{rs}_{4}=\left[r_{0},r_{1},r_{2},r_{3}\right]$,
because the head is closer to $r_{0}$ in the edge between $r_{0}$
to $r_{2}$, which does not contain an arrow. More simply, since $\boldsymbol{K}_{ij}\neq0$,
if $h_{i}$ can be the ancestor of $r_{j}$, then a rule $r_{i}$
can be the root if and only if $\boldsymbol{K}_{ij}\neq0$ for all
$r_{j}\in\mathit{rs}_{K}$. Therefore, assuming each element in $\boldsymbol{K}$
is unique, and $\boldsymbol{K}_{ij}=\pm1$ if $r_{j}$ can be generated
from $r_{i}^{\pm}$, a$\mathit{splits}$ function can be constructed
that identifies which splitting rules, within a given list of rules
$rs$, are viable candidates to become the root of a proper decision
tree:
\begin{align*}
	splits & :\left[\mathcal{R}\right]\to\left[\left(\left[\mathcal{R}\right],\mathcal{R},\left[\mathcal{R}\right]\right)\right]\\
	splits & \left(\mathit{rs}\right)=\left[\left(\mathit{rs}^{+},r_{i},\mathit{rs}^{-}\right)\mid r_{i}\leftarrow rs,\mathit{all}\left(r_{i},\mathit{rs}\right)=\mathit{True}\right],
\end{align*}
where $rs^{+}=\left[r_{j}\mid r_{j}\leftarrow rs,\boldsymbol{K}_{ij}=1\right]$,
$rs^{+}=\left[r_{j}\mid r_{j}\leftarrow rs,\boldsymbol{K}_{ij}=-1\right]$
and $\mathit{all}\left(r_{i},rs\right)$ returns true if all rules
$r_{j}$ in $rs$ satisfy $K_{ij}\neq0$ for $i\neq j$, and false
otherwise.

\paragraph{The partial proper decision tree generator}

With the help of the $\mathit{splits}$ function, which makes sure
only \emph{feasible} \emph{splitting} \emph{rules}---those that can
serve as the root of a tree or subtree---are selected as the root,
we can define an efficient decision tree generator as follows.
\begin{definition}
	Given a list of rule $\mathit{rs}:\left[\mathcal{R}\right]$, the
	search space of all possible \emph{branch-labeled} binary trees, that
	satisfies the proper decision tree axioms, with respect to $\mathit{rs}$
	is defined by
	\begin{equation}
		\begin{aligned}\mathit{genBTs} & :\left[\mathcal{R}\right]\to\left[\mathit{BTree}\left(\mathcal{R}\right)\right]\\
			\mathit{genBTs} & \left(\left[\;\right]\right)=\left[L\right]\\
			\mathit{genBTs} & \left(\left[r\right]\right)=\left[N\left(L,r,L\right)\right]\\
			\mathit{genBTs} & \left(\mathit{rs}\right)=\left[N\left(u,r_{i},v\right)\mid\left(rs^{+},r_{i},rs^{-}\right)\leftarrow\mathit{splits}\left(\mathit{rs}\right),u\leftarrow\mathit{genBTs}\left(\mathit{rs}^{+}\right),v\leftarrow\mathit{genBTs}\left(\mathit{rs}^{-}\right)\right].
		\end{aligned}
		\label{eq: genBTs}
	\end{equation}
\end{definition}
This $\mathit{genBTs}$ generator function recursively constructs
larger decision trees $N\left(u,r_{i},v\right)$ from smaller proper
decision trees $\mathit{genBTs}\left(\mathit{rs}^{+}\right)$ and
$\mathit{genBTs}\left(\mathit{rs}^{-}\right)$, the $\mathit{splits}$
function ensuring that only feasible splitting rules can become the
root of a subtree during recursion.

Replacing $\mathit{genDTs}$ with $\mathit{genBTs}$ in (\ref{eq: new-specification of genDTKs})
yields a program that generates binary trees with the same \emph{structure}
as the decision trees in the search space $\mathcal{S}\left(K,\mathit{rs}\right)$,
such that $\left|\mathit{rs}\right|>K$. The following lemma shows
that $\mathit{genBTs}\left(\mathit{rs}_{K}\right)$ exhaustively generates
all decision trees in $\mathcal{S}\left(K,\mathit{rs}_{K}\right)$
satisfying Axioms 1, 3, and 4 of the proper decision tree axioms,
but \textbf{not} Axiom 2, since the leaves $L$ contain no information
in the binary tree datatype.
\begin{lemma}
	\emph{The decision tree generator $\mathit{genBTs}\left(\mathit{rs}_{K}\right)$
		exhaustively generate decision trees in search space $\mathcal{S}\left(K,\mathit{rs}_{K}\right)$
		which satisifes the Axiom 1, 3, 4 of the proper decision tree axiom.\label{lem: exhaustiveness of genBTs}}
\end{lemma}
\begin{proof}
	The proof for singleton and empty case is trivial, the empty case
	consists of a single tree with a leaf contain no information. Similarly,
	the singleton case produces one tree with a rule at the root, which
	also trivially satisfies the axioms of proper decision trees.
	
	We now prove the recursive case by showing that $\mathit{splits}$
	preserves the proper decision tree axioms. The traversal $r_{i}\leftarrow\mathit{rs}$,
	which exhaustively selects $r_{i}$ from $\mathit{rs}$ as the root
	of the tree, together with the predicate $\mathit{all}\left(r_{i},\mathit{rs}\right)=\mathit{True}$
	ensures that only those splitting rules capable of serving as the
	ancestor of all their descendants can be chosen as roots. This satisfies
	the transitivity requirement (first part of Axiom 3). For each fixed
	root $r_{i}$, the remaining rules $\mathit{rs}\backslash r$ are
	partitioned into two disjoint subsets $\mathit{rs}^{+},\mathit{rs}^{-}$
	according to the ancestry matrix $\boldsymbol{K}$. The uniqueness
	assumption of $\boldsymbol{K}$ automatically guarantees the ancestral
	constraint (Axiom 4). Furthermore, since we have defined $\boldsymbol{K}_{ij}=\pm1$
	whenever $r_{j}$ can be generated from $r_{i}^{\pm}$, it follows
	that $\mathit{rs}^{\pm}$ precisely correspond to the sets of rules
	defined within the subspaces induced by $r_{i}^{\pm}$ (This satisfies
	the second part of Axiom 3, which is implicitly derived from the assumptions
	of Axiom 1 and the first part of Axiom 3).
	
	By the induction hypothesis, if $\mathit{genBTs}\left(\mathit{rs}^{\pm}\right)$
	generates all possible decision trees with respect to $\mathit{rs}^{\pm}$,
	then $\mathit{genBTs}\left(\mathit{rs}\right)$ also generates all
	possible decision trees satisfying Axiom 1, 3, 4 since $\mathit{splits}$
	preserves these axioms.
\end{proof}
In the proof of Lemma \ref{lem: exhaustiveness of genBTs}, all details
of the recursive case are encoded in the definition of $\mathit{splits}$,
which are implicitly determined by $\boldsymbol{K}$. The structure
of $\boldsymbol{K}$ is, in turn, governed by the ancestral constraint,
which enforces the uniqueness of its elements. In Subsection \ref{subsec:Non-exhua-Murtree},
we will examine two special cases that impose additional constraints
on $\boldsymbol{K}$, which will lead to a different definitions of
the decision tree search space.

In short, the $\mathit{splits}$ function \emph{governs the algorithmic
	structure of the ODT algorithm, which in turn is determined by the
	ancestral constraint of the decision tree}. Without first formalizing
the decision tree, the definition of the $\mathit{splits}$ function
remains unclear, and any claim regarding optimality would be unsound.

The remaining challenge lies in \emph{accumulating} information along
the paths from root to leaves. We next explain how a new generator,
$\mathit{genDTs}$ can be derived from $\mathit{genBTs}$ to also
satisfy axiom 2, by introducing \citet{gibbons1991algebras}'s downward
accumulation technique.

\subsection{Downwards accumulation for proper decision trees\label{subsec:Downwards-accumulation-for}}

We are now half way towards our goal. The $\mathit{genBTs}$ function
provides an efficient way of generating the \emph{structure} \emph{of
	the decision tree}, namely binary tree representations of a decision
tree. However, this is just a partial decision tree generator. Since
a decision tree is not a binary tree, we need to figure out how to
``pass information down the tree'' from the root towards the leaves
during the recursive construction of the tree. In other words, we
need to \emph{accumulate} all the information for each path of the
tree from the root to each leaf. In this section we introduce a technique
called \emph{downwards} \emph{accumulation} \citet{gibbons1991algebras},
which will helps us to construct the ``complete decision tree generator,''
such that it allows $\mathit{genBTs}$ to respect Axiom 2.

Accumulations are higher-order operations on structured objects that
leave the shape of an object unchanged, but replace every element
of that object with some accumulated information about other elements.
For instance, the \emph{prefix sums }(with binary operator $\oplus$)
over a list are an example of accumulation over list $\left[a_{1},a_{2},\ldots a_{n}\right]$,
\[
\left[a_{1},a_{1}\oplus a_{2},\ldots,a_{1}\oplus a_{2}\oplus\ldots\oplus a_{n}\right].
\]

Downward accumulation over binary trees is similar to list accumulation,
as it replaces every element of a tree with some function of that
element's ancestors. For example, given a tree

\begin{center}
	\begin{tikzpicture}[sibling distance=2.5cm,level distance=1.5cm, edge from parent/.style={draw,-latex}]
		\node[circle, draw] {$r_1$} 
		child {node[circle, draw] {$r_2$}}
		child {  node[circle, draw]  {$r_3$}
			child { node[circle, draw]  {$r_4$} }
			child { node[circle, draw]  {$r_5$} }};
	\end{tikzpicture}
\end{center}applying downwards accumulation with binary operator $\oplus$ to
the above tree results in

\begin{center}
	\begin{tikzpicture}[edge from parent/.style={draw,-latex}]
		\node[circle, draw] {$r_1$} 
		[sibling distance=2.5cm, level distance=1.5cm] 
		child {node(special) {$r_1 \oplus r_2$}}
		child {  node(special1)  {$r_1 \oplus r_3$} [sibling distance=2.5cm, level distance=1.5cm] 
			child { node(special2) {$r_1 \oplus r_3 \oplus r_4$} }
			child { node(special3)  {$r_1 \oplus r_3 \oplus r_5$} }};
		\begin{scope}[on background layer]
			\draw (special) ellipse [x radius=0.8cm, y radius=0.30cm];
			\draw (special1) ellipse [x radius=0.8cm, y radius=0.30cm];
			\draw (special2) ellipse [x radius=1.2cm, y radius=0.30cm];
			\draw (special3) ellipse [x radius=1.2cm, y radius=0.30cm];
		\end{scope}
	\end{tikzpicture}
\end{center}

The information in each leaf is determined by the path from the root
to that leaf. Therefore, \citet{gibbons1991algebras}'s downward accumulation
method can be adopted to construct the decision tree generator.

To formalize downward accumulation, we first need to define a \emph{path}
\emph{generator} and a \emph{path} \emph{datatype}. The ancestry relation
can be viewed as a path of length one, thus we can abstract $\swarrow$
and $\searrow$ as constructors of the datatype. We define the path
datatype recursively as

\[
\mathcal{P}\left(\mathcal{A}\right)=S\left(\mathcal{A}\right)\mid\mathcal{P}\left(\mathcal{A}\right)\swarrow\mathcal{P}\left(\mathcal{A}\right)\mid\mathcal{P}\left(\mathcal{A}\right)\searrow\mathcal{P}\left(\mathcal{A}\right).
\]

In words: a path is either a single node $S\left(\mathcal{A}\right)$,
or two paths connected by $\swarrow$ or $\searrow$.

The next step toward defining downward accumulation requires a function
that generates all possible paths of a tree. For this purpose, \citet{gibbons1996computing}
introduced a definition of \emph{paths }over a binary tree, which
replaces \emph{each} node with the path from the root to that node.
However, the accumulation required for the decision tree problem differs
from the classical formulation. In \citet{gibbons1996computing}'s
downward accumulation algorithm, information is propagated to every
node, treating both branch and leaf nodes uniformly. By contrast,
the decision tree problem requires passing information only to the
leaf nodes, leaving the branch nodes unchanged. We generalize the
$\mathit{paths}$ generator proposed by \citet{gibbons1996computing}
for the binary tree datatype to the decision tree datatype using the
following definition.
\begin{definition}
	The function $\mathit{paths}$ for decision tree datatype is defined
	by 
	\[
	\begin{aligned}\mathit{paths} & :\left(\mathcal{B}\to\mathcal{A}\right)\times\left(\mathcal{B}\to\mathcal{A}\right)\times\mathit{DTree}\left(\mathcal{B},\mathcal{A}\right)\to\mathit{DTree}\left(\mathcal{B},\mathcal{P}\left(\mathcal{A}\right)\right)\\
		\mathit{paths} & \left(f,g,\mathit{DL}\left(\mathit{xs}\right)\right)=\mathit{DL}\left(S\left(\mathit{xs}\right)\right)\\
		\mathit{paths} & \left(f,g,\mathit{DN}\left(u,r,v\right)\right)=\mathit{DN}\left(\mathit{mapD}{}_{\swarrow_{S\left(f\left(r\right)\right)}}\left(\mathit{paths}\left(f,g,u\right)\right),S\left(r\right),\mathit{mapD}{}_{\searrow_{S\left(g\left(r\right)\right)}}\left(\mathit{paths}\left(f,g,v\right)\right)\right),
	\end{aligned}
	\]
	where $\mathcal{P}$ receives only one type $\mathcal{A}$, two functions
	$f$ and $g$ are used to transform $r:\mathcal{B}$ into type $\mathcal{A}$
	, while also distinguishing between ``left turn'' ($\swarrow$) and
	``right turn'' ($\searrow$).
\end{definition}
To see how this works, consider the decision tree $T$ given below,
where for simplicity, the singleton path constructor $S\left(\;\right)$
is left implicit, and we denote the leaf value using symbol $\diamond:S\left(\mathcal{A}\right)$:

\begin{center}
	\begin{tikzpicture}[edge from parent/.style={draw,-latex}]
		\node[circle, draw] {$r_1$} 
		[sibling distance=3cm, level distance=1.2cm] 
		child {node[circle, draw] {$r2$}[sibling distance=1.2cm]
			child { node[circle, draw]  {$\diamond$} }
			child { node[circle, draw]  {$\diamond$} }}
		child {  node[circle, draw]  {$r_3$} [sibling distance=2cm] 
			child { node[circle, draw]  {$r_4$}[sibling distance=1.2cm] 
				child { node[circle, draw]  {$\diamond$} }
				child { node[circle, draw]  {$\diamond$} }}
			child { node[circle, draw]  {$r_5$}[sibling distance=1.2cm]
				child { node[circle, draw]  {$\diamond$} }
				child { node[circle, draw]  {$\diamond$} }}};
	\end{tikzpicture}
\end{center}

Running $paths$ on decision tree $T$, we obtain

\begin{center}
	\begin{tikzpicture}[edge from parent/.style={draw,-latex}]
		\node[circle, draw] {$r_1$} 
		[sibling distance=6cm, level distance=1cm] 
		child {node[circle, draw] {$r2$}[sibling distance=2.5cm]
			child { node (special2) {$r_1 \swarrow r_2 \swarrow \diamond$} }
			child { node (special3) {$r_1 \swarrow r_2 \searrow \diamond$} }}
		child {  node[circle, draw]  {$r_3$} [sibling distance=6.2cm] 
			child { node[circle, draw]  {$r_4$}[sibling distance=3.1cm] 
				child { node (special) {$r_1 \searrow r_3 \swarrow r_4 \swarrow \diamond$} }
				child { node (special4) {$r_1 \searrow r_3 \swarrow r_4 \searrow \diamond$} }}
			child { node[circle, draw]  {$r_5$}[sibling distance=3.1cm]
				child { node(special5)  {$r_1 \searrow r_3 \searrow r_5 \swarrow \diamond$} }
				child { node(special6)  {$r_1 \searrow r_3 \searrow r_5 \searrow \diamond$} }}};
		
		\begin{scope}[on background layer]
			\draw (special) ellipse [x radius=1.5cm, y radius=0.30cm];
			\draw (special2) ellipse [x radius=1.2cm, y radius=0.30cm];
			\draw (special3) ellipse [x radius=1.2cm, y radius=0.30cm];
			\draw (special4) ellipse [x radius=1.5cm, y radius=0.30cm];
			\draw (special5) ellipse [x radius=1.5cm, y radius=0.30cm];
			\draw (special6) ellipse [x radius=1.5cm, y radius=0.30cm];
		\end{scope}
	\end{tikzpicture}
\end{center}Compared with Gibbons' definition, here, only the leaf nodes are replaced
with the path from the root to the ancestors, while the structure
and branch nodes of the tree remain unchanged. This behavior is determined
by the decision tree definition: information is accumulated along
paths to the leaves (Axiom 2) without altering the splitting rules
at the branch nodes.

The $\mathit{paths}$ function constructs only the paths from the
root to each leaf. To perform operations during path construction,
we require a crucial function associated with the path datatype, called
\emph{path} \emph{reduction} (also known as \emph{path homomorphism}),
which reduces a path to a single value. \citet{gibbons1996computing}
defines the path reduction operator as follows.
\begin{definition}
	The function $\mathit{pathRed}$ is defined by 
	\begin{align*}
		\mathit{pathRed} & ::\left(\mathcal{A}\to\mathcal{B}\right)\times\left(\mathcal{A}\to\mathcal{B}\to\mathcal{B}\right)\times\left(\mathcal{A}\to\mathcal{B}\to\mathcal{B}\right)\times\mathcal{P}\left(\mathcal{A}\right)\to\mathcal{B}\\
		\mathit{pathRed} & \left(f,\oplus,\otimes,S\left(\mathit{xs}\right)\right)=f\left(\mathit{xs}\right)\\
		\mathit{pathRed} & \left(f,\oplus,\otimes,p\swarrow q\right)=\mathit{pathRed}\left(\oplus^{\mathit{pathRed}\left(f,\oplus,\otimes,q\right)},\oplus,\otimes,p\right)\\
		\mathit{pathRed} & \left(f,\oplus,\otimes,p\searrow q\right)=\mathit{pathRed}\left(\otimes^{\mathit{pathRed}\left(f,\oplus,\otimes,q\right)},\oplus,\otimes,p\right).
	\end{align*}
\end{definition}
From the type of $\mathit{pathRed}$, it is straightforward to consider
applying $\mathit{pathRed}$ to all paths generated by $\mathit{paths}$
via the following composition.
\begin{equation}
	\mathit{dacc}_{h,f,g,\oplus,\otimes}=\mathit{mapD}{}_{\mathit{pathRed}{}_{h,\oplus,\otimes}}\circ\mathit{paths}_{f,g},\label{eq: downward accumulation specification}
\end{equation}
where $\mathit{dacc}$ stands for ``downward accumulation.'' Indeed,
(\ref{eq: downward accumulation specification}) is the definition
of downward accumulation, which is can be seen as a \emph{programmatic
	definition for Axiom 2}. In other words, applying $\mathit{dacc}_{h,f,g,\oplus,\otimes}$
to the example tree above yields:

\begin{center}
	\begin{tikzpicture}[edge from parent/.style={draw,-latex}]
		\node[circle, draw] {$r_1$} 
		[sibling distance=6cm, level distance=1cm] 
		child {node[circle, draw] {$r2$}[sibling distance=2.5cm]
			child { node (special2) {$r_1 \oplus r_2 \oplus \diamond$} }
			child { node (special3) {$r_1 \oplus r_2 \otimes \diamond$} }}
		child {  node[circle, draw]  {$r_3$} [sibling distance=6.2cm] 
			child { node[circle, draw]  {$r_4$}[sibling distance=3.1cm] 
				child { node (special) {$r_1 \otimes r_3 \oplus r_4 \oplus \diamond$} }
				child { node (special4) {$r_1 \otimes r_3 \oplus r_4 \otimes \diamond$} }}
			child { node[circle, draw]  {$r_5$}[sibling distance=3.1cm]
				child { node(special5)  {$r_1 \otimes r_3 \otimes r_5 \oplus \diamond$} }
				child { node(special6)  {$r_1 \otimes r_3 \otimes r_5 \otimes \diamond$} }}};
		
		\begin{scope}[on background layer]
			\draw (special) ellipse [x radius=1.5cm, y radius=0.30cm];
			\draw (special2) ellipse [x radius=1.2cm, y radius=0.30cm];
			\draw (special3) ellipse [x radius=1.2cm, y radius=0.30cm];
			\draw (special4) ellipse [x radius=1.5cm, y radius=0.30cm];
			\draw (special5) ellipse [x radius=1.5cm, y radius=0.30cm];
			\draw (special6) ellipse [x radius=1.5cm, y radius=0.30cm];
		\end{scope}
	\end{tikzpicture}
\end{center}The function $\mathit{dacc}_{h,f,g,\oplus,\otimes}$ is applied to
a proper decision tree $\mathit{DTree}\left(\mathcal{B},\mathcal{A}\right)$,
when a path turns left, it applies the $\oplus$ operator and accumulates
the results, and when it turns right, it applies the $\otimes$ operator
and accumulates the results.

However, the $\mathit{dacc}$ function defined in (\ref{eq: downward accumulation specification})
is inefficient. Assuming $\oplus$, $\otimes$, $f$, and $g$ each
have $O\left(1\right)$ complexity, $\mathit{dacc}$ has $O\left(N^{2}\right)$
complexity for a decision tree of size $N$. Since the information
computed at a node $r_{i}$ can be stored and reused by both its left
and right descendants, it is possible to ``reuse'' the value from
a parent when computing values for its children. This observation
reduces the complexity from $O\left(N^{2}\right)$ to $O\left(N\right)$.

In particular, we have the following theorem.
\begin{theorem}
	Homomorphic downward accumulation. \emph{Every downward accumulation
		defined by (\ref{eq: downward accumulation specification}) has the
		following property}
	
	\begin{align}
		\mathit{mapD}{}_{\mathit{pathRed}{}_{h,\oplus,\otimes}} & \left(\mathit{paths}_{f,g}\left(\mathit{DN}\left(u,r,v\right)\right)\right)=\nonumber \\
		& \mathit{DN}\bigg(\mathit{mapD}{}_{\oplus_{f\left(r\right)}}\left(\mathit{mapD}{}_{\mathit{pathRed}{}_{h,\oplus,\otimes}}\left(\mathit{paths}_{f,g}\left(u\right)\right)\right),r\label{eq:  Homomorphic downward accumulation}\\
		& \quad\quad\quad\mathit{mapD}{}_{\otimes_{g\left(r\right)}}\left(\mathit{mapD}{}_{\mathit{pathRed}{}_{h,\oplus,\otimes}}\left(\mathit{paths}_{f,g}\left(v\right)\right)\right)\bigg),\nonumber 
	\end{align}
	\emph{where $\mathit{dacc}_{h,f,g,\oplus,\otimes}\left(\mathit{DN}\left(u,r,v\right)\right)=\mathit{mapD}{}_{\mathit{pathRed}{}_{h,\oplus,\otimes}}\left(\mathit{paths}_{f,g}\left(\mathit{DN}\left(u,r,v\right)\right)\right)$.\label{thm: downward accumulation}}
\end{theorem}
\begin{proof}
	We can prove (\ref{eq:  Homomorphic downward accumulation}) by following
	equational reasoning
	\[
	\begin{aligned} & \mathit{mapD}{}_{\mathit{pathRed}{}_{h,\oplus,\otimes}}\left(\mathit{paths}_{f,g}\left(\mathit{DN}\left(u,r,v\right)\right)\right)\\
		\equiv & \text{ definition of \ensuremath{\mathit{paths}}}\\
		& \mathit{mapD}{}_{\mathit{pathRed}{}_{h,\oplus,\otimes}}\left(\mathit{DN}\left(\mathit{mapD}_{\swarrow_{S\left(f\left(r\right)\right)}}\left(\mathit{paths}_{f,g}\left(u\right)\right),r,\mathit{mapD}_{\searrow_{S\left(g\left(r\right)\right)}}\left(\mathit{paths}_{f,g}\left(v\right)\right)\right)\right)\\
		\equiv & \text{ definition of \ensuremath{\mathit{mapD}}}\\
		& \mathit{DN}\left(\mathit{mapD}{}_{\mathit{pathRed}{}_{h,\oplus,\otimes}}\left(\mathit{mapD}_{\swarrow_{S\left(f\left(r\right)\right)}}\left(\mathit{paths}_{f,g}\left(u\right)\right)\right),r,\mathit{mapD}{}_{\mathit{pathRed}{}_{h,\oplus,\otimes}}\left(\mathit{mapD}_{\searrow_{S\left(g\left(r\right)\right)}}\left(\mathit{paths}_{f,g}\left(v\right)\right)\right)\right)\\
		\equiv & \text{ map composition}\\
		& \mathit{DN}\left(\mathit{mapD}_{\mathit{pathRed}{}_{h,\oplus,\otimes}\circ\swarrow_{S\left(f\left(r\right)\right)}}\left(\mathit{paths}_{f,g}\left(u\right)\right)\right),r,\mathit{mapD}{}_{\mathit{pathRed}{}_{h,\oplus,\otimes}}\left(\mathit{mapD}_{\mathit{pathRed}{}_{h,\oplus,\otimes}\circ\searrow_{S\left(g\left(r\right)\right)}}\left(\mathit{paths}_{f,g}\left(v\right)\right)\right)\\
		\equiv & \text{ definition of \ensuremath{\mathit{pathRed}}}\\
		& \mathit{DN}\left(\mathit{mapD}{}_{\oplus_{f\left(r\right)}\circ\mathit{pathRed}{}_{h,\oplus,\otimes}}\left(\mathit{paths}_{f,g}\left(u\right)\right),r,\mathit{mapD}{}_{\mathit{\otimes_{g\left(r\right)}\circ pathRed}{}_{h,\oplus,\otimes}}\left(\mathit{paths}_{f,g}\left(v\right)\right)\right)\\
		\equiv & \text{ map composition}\\
		& \mathit{DN}\left(\mathit{mapD}_{\oplus_{f\left(r\right)}}\left(\mathit{mapD}{}_{\mathit{pathRed}{}_{h,\oplus,\otimes}}\left(\mathit{paths}_{f,g}\left(u\right)\right)\right),r,\mathit{mapD}_{\otimes_{g\left(r\right)}}\left(\mathit{mapD}{}_{\mathit{pathRed}{}_{h,\oplus,\otimes}}\left(\mathit{paths}_{f,g}\left(v\right)\right)\right)\right),
	\end{aligned}
	\]
\end{proof}
This homomorphic downward accumulation can be computed in parallel
functional time proportional to the product of the depth of the tree
and the time taken by the individual operations \citep{gibbons1996computing},
and thus is amenable to fusion.

\subsection{The complete proper decision tree generator\label{subsec: complete decision tree generator}}

In Subsection \ref{subsec: partial decision tree generator}, we described
the construction of a decision tree generator based on the binary
tree data type. However, the $\mathit{\mathit{genBTs}}$ function
generates only the \emph{structure} of the decision tree, which contains
information solely about the branch nodes. While this structure is
sufficient for evaluating the tree, constructing a \emph{complete}
\emph{decision} \emph{tree}---one that incorporates both branch nodes
and leaf nodes---is essential for improving the algorithm's efficiency.

Before we moving towards deriving a complete decision tree generator,
we need to generalize $\mathit{genBTs}$ to define it over the $\mathit{DTree}$
datatype

\begin{align*}
	\mathit{genBTs}^{\prime} & :\left[\mathcal{R}\right]\times\mathcal{D}\to\left[\mathit{DTree}\left(\mathcal{R},\mathcal{D}\right)\right]\\
	\mathit{genBTs}^{\prime} & \left(\left[\;\right],\mathit{xs}\right)=\left[\mathit{DL}\left(\mathit{xs}\right)\right]\\
	\mathit{genBTs}^{\prime} & \left(\left[r\right],\mathit{xs}\right)=\left[\mathit{DN}\left(\mathit{DL}\left(\mathit{xs}\right),r,\mathit{DL}\left(\mathit{xs}\right)\right)\right]\\
	\mathit{genBTs}^{\prime} & \left(\mathit{rs},\mathit{xs}\right)=\left[\mathit{DN}\left(u,r_{i},v\right)\mid\left(\mathit{rs}{}^{+},r_{i},\mathit{rs}{}^{-}\right)\leftarrow\mathit{splits}\left(\mathit{rs}\right),u\leftarrow\mathit{genBTs}^{\prime}\left(\mathit{rs}{}^{+},\mathit{xs}\right),v\leftarrow\mathit{genBTs}^{\prime}\left(\mathit{rs}{}^{-},\mathit{xs}\right)\right].
\end{align*}

The difference between complete and partial decision tree generators
lies in the fact that the complete one contains accumulated information
in leaves and the partial one has no information, just the leaf label
$L$. This reasoning allows us to define the complete decision tree
generator $\mathit{genDTs}$ based on the partial $\mathit{genBTs}^{\prime}$
as follows.
\begin{definition}
	The complete proper decision tree generator is defined as
	
	\begin{equation}
		\mathit{genDTs}\left(\mathit{rs},\mathit{xs}\right)=\mathit{mapL}_{\mathit{dacc}_{\mathit{id},\mathit{sl},\mathit{sr},\cap,\cap}}\left(\mathit{genBTs}^{\prime}\left(\mathit{rs},\mathit{xs}\right)\right),\label{eq: specification of the complete decision tree generator}
	\end{equation}
	where $\mathit{sl}$ and $\mathit{sr}$ are short for ``split left'',
	and ``split right'', respectively, defined as $\mathit{sl}\left(r\right)=r^{+}$
	and $\mathit{sr}\left(r\right)=r^{-}$.
\end{definition}
	By composing $\mathit{mapL}_{\mathit{dacc}_{\mathit{id},\mathit{sl},\mathit{sr},\cap,\cap}}$
	with $\mathit{genBTs}^{\prime}\left(\mathit{rs},\mathit{xs}\right)$,
	we effectively apply the downward accumulation $\mathit{dacc}_{\mathit{id},\mathit{sl},\mathit{sr},\cap,\cap}$
	to every binary tree generated by $\mathit{genBTs}^{\prime}\left(\mathit{rs},\mathit{xs}\right)$,
	accumulating information from root to leaves for each tree and incorporating
	Axiom 2 into the $\mathit{genDTs}$ function. As a result, $\mathit{genDTs}\left(\mathit{rs},\mathit{xs}\right)$
	now satisfies all desired properties of proper decision tree axioms.
	However, applying $\mathit{mapL}_{\mathit{dacc}_{\mathit{id},\mathit{sl},\mathit{sr},\cap,\cap}}$
	post-hoc to each tree is in efficient.
	
	Fortunately, we can show that $\mathit{mapL}_{\mathit{dacc}_{\mathit{id},\mathit{sl},\mathit{sr},\cap,\cap}}$
	can be fused into $\mathit{genBTs}^{\prime}$, thereby yielding an
	more efficient definition of $\mathit{genDTs}$. The following theorem
	explains this fusion.

\begin{theorem}
	\emph{Complete proper decision tree generator. The $\mathit{genDTs}\left(\mathit{rs},\mathit{xs}\right)$
		defined by
		\begin{equation}
			\begin{aligned}\mathit{genDTs} & :\left[\mathcal{R}\right]\times\mathcal{D}\to\left[\mathit{DTree}\left(\mathcal{R},\mathcal{D}\right)\right]\\
				\mathit{genDTs} & \left(\left[\;\right],\mathit{xs}\right)=\left[\mathit{DL}\left(\mathit{xs}\right)\right]\\
				\mathit{genDTs} & \left(\left[r\right],\mathit{xs}\right)=\left[DN\left(\mathit{DL}\left(r^{+}\cap\mathit{xs}\right),r,\mathit{DL}\left(r^{-}\cap\mathit{xs}\right)\right)\right]\\
				\mathit{genDTs} & \left(\mathit{rs},\mathit{xs}\right)=\bigg[\mathit{DN}\left(\mathit{mapD}{}_{\cap_{r_{i}^{+}}}\left(u\right),r_{i},\mathit{mapD}{}_{\cap_{r_{i}^{-}}}\left(v\right)\right)\mid\\
				& \begin{aligned} & \quad\quad\quad\left(\mathit{rs}^{+},r_{i},\mathit{rs}^{-}\right)\leftarrow\mathit{splits}\left(\mathit{rs}\right),u\leftarrow\mathit{genDTs}\left(\mathit{rs}^{+},\mathit{xs}\right),v\leftarrow\mathit{genDTs}\left(\mathit{rs}^{-},\mathit{xs}\right)\bigg].\end{aligned}
			\end{aligned}
		\end{equation}
		is equivalent to (\ref{eq: specification of the complete decision tree generator}).
		\label{thm: comlete DT generator}}
\end{theorem}
\begin{proof}
	The singleton and empty cases are easy to prove, so we only prove
	the singleton case here, the empty case is omitted for reasons of
	space:
	\[
	\begin{aligned} & \mathit{mapL}_{\mathit{dacc}_{\mathit{id},\mathit{sl},\mathit{sr},\cap,\cap}}\left(\mathit{genBTs}^{\prime}\left(\mathit{rs},\mathit{xs}\right)\right)\\
		\equiv & \text{ definition of \ensuremath{\mathit{genBTs}^{\prime}}}\\
		& \mathit{mapL}_{\mathit{dacc}_{\mathit{id},\mathit{sl},\mathit{sr},\cap,\cap}}\left[\mathit{DN}\left(\mathit{DL}\left(\mathit{xs}\right),r,\mathit{DL}\left(\mathit{xs}\right)\right)\right]\\
		\equiv & \text{ definition of }\mathit{mapL}\\
		& \left[\mathit{dacc}_{\mathit{id},\mathit{sl},\mathit{sr},\cap,\cap}\left(\mathit{DN}\left(\mathit{DL}\left(\mathit{xs}\right),r,\mathit{DL}\left(\mathit{xs}\right)\right)\right)\right]\\
		\equiv & \text{ definition of }\text{\ensuremath{\mathit{dacc}_{\mathit{id},\mathit{sl},\mathit{sr},\cap,\cap}}}\\
		& \left[\mathit{DN}\left(\mathit{DL}\left(r^{+}\cap\mathit{xs}\right),r,\mathit{DL}\left(r^{-}\cap\mathit{xs}\right)\right)\right].
	\end{aligned}
	\]
	
	The recursive case is proved by the following equational reasoning.
	For simplicity, let $f$ denote $\mathit{pathRed}_{id,\cap,\cap}$
	and $g$ denote $\mathit{paths}_{\mathit{sl},\mathit{sr}}$. Thus,
	we have $\mathit{mapD}{}_{\mathit{pathRed}_{id,\cap,\cap}}\circ\mathit{paths}_{\mathit{sl},\mathit{sr}}=\left(\mathit{mapD}{}_{f}\right)\circ g=\mathit{mapD}{}_{f}\circ g$
	
	\[
	\begin{aligned} & \mathit{mapL}_{\mathit{dacc}_{\mathit{id},\mathit{sl},\mathit{sr},\cap,\cap}}\left(\mathit{genBTs}^{\prime}\left(\mathit{rs},\mathit{xs}\right)\right)\\
		\equiv & \text{ definition of \ensuremath{dacc}}\\
		& \mathit{mapL}{}_{\mathit{mapD}_{f}\circ g}\left(\mathit{genBTs}^{\prime}\left(\mathit{rs},\mathit{xs}\right)\right)\\
		\equiv & \text{ definition of \ensuremath{\mathit{genBTs}^{\prime}}}\\
		& \begin{aligned} & \mathit{mapL}{}_{\mathit{mapD}{}_{f}\circ g}\bigg[\mathit{DN}\left(u,r_{i},v\right)\mid\\
			& \quad\quad\left(\mathit{rs}^{+},r_{i},\mathit{rs}^{-}\right)\leftarrow\mathit{splits}\left(rs\right),u\leftarrow\mathit{genBTs}^{\prime}\left(\mathit{rs}^{+},\mathit{xs}\right),v\leftarrow\mathit{genBTs}^{\prime}\left(\mathit{rs}^{-},xs\right)\bigg],
		\end{aligned}
		\\
		\equiv & \text{ definition of }\mathit{mapL}\\
		& \begin{aligned} & \bigg[\mathit{mapD}_{f}\left(g\left(\mathit{DN}\left(u,r_{i},v\right)\right)\right)\mid\\
			& \quad\quad\left(\mathit{rs}^{+},r_{i},\mathit{rs}^{-}\right)\leftarrow\mathit{splits}\left(rs\right),u\leftarrow\mathit{genBTs}^{\prime}\left(\mathit{rs}^{+},\mathit{xs}\right),v\leftarrow\mathit{genBTs}^{\prime}\left(\mathit{rs}^{-},xs\right)\bigg],
		\end{aligned}
		\\
		\equiv & \text{ Theorem (\ref{thm: downward accumulation}), definition of \ensuremath{\mathit{sl}\left(r\right)} and \ensuremath{\mathit{sr}\left(r\right)}}\\
		& \begin{aligned} & \bigg[\mathit{DN}\left(\mathit{mapD}{}_{\cap_{r^{+}}\circ f}\left(g\left(u\right)\right),r_{i},\mathit{mapD}{}_{\cap_{r^{-}}\circ f}\left(g\left(v\right)\right)\right)\mid\\
			& \quad\quad\left(\mathit{rs}^{+},r_{i},\mathit{rs}^{-}\right)\leftarrow\mathit{splits}\left(rs\right),u\leftarrow\mathit{genBTs}^{\prime}\left(\mathit{rs}^{+},\mathit{xs}\right),v\leftarrow\mathit{genBTs}^{\prime}\left(\mathit{rs}^{-},xs\right)\bigg],
		\end{aligned}
		\\
		\equiv & \text{ map composition}\\
		& \begin{aligned} & \bigg[\mathit{DN}\left(\mathit{mapD}{}_{\cap_{r^{+}}}\left(\mathit{mapD}_{f}\left(g\left(u\right)\right)\right),r_{i},\mathit{mapD}{}_{\cap_{r^{-}}}\left(\mathit{mapD}_{f}\left(g\left(v\right)\right)\right)\right)\mid\\
			& \quad\quad\left(\mathit{rs}^{+},r_{i},\mathit{rs}^{-}\right)\leftarrow\mathit{splits}\left(rs\right),u\leftarrow\mathit{genBTs}^{\prime}\left(\mathit{rs}^{+},\mathit{xs}\right),v\leftarrow\mathit{genBTs}^{\prime}\left(\mathit{rs}^{-},xs\right)\bigg],
		\end{aligned}
		\\
		\equiv & \text{ definition of list comprehension}\\
		& \begin{aligned} & \bigg[\mathit{DN}\left(\mathit{mapD}{}_{\cap_{r^{+}}}\left(u\right),r_{i},\mathit{mapD}{}_{\cap_{r^{-}}}\left(v\right)\right)\mid\\
			& \quad\quad\left(\mathit{rs}^{+},r_{i},\mathit{rs}^{-}\right)\leftarrow\mathit{splits}\left(rs\right),u\leftarrow\mathit{mapL}_{\mathit{mapD}{}_{f}\circ g}\left(\mathit{genBTs}^{\prime}\left(\mathit{rs}^{+},\mathit{xs}\right)\right)\\
			& \quad\quad\quad\quad\quad\quad\quad\quad\quad\quad\quad\quad\quad\quad v\leftarrow\mathit{mapL}_{\mathit{mapD}{}_{f}\circ g}\left(\mathit{genBTs}^{\prime}\left(\mathit{rs}^{-},\mathit{xs}\right)\right)\bigg],
		\end{aligned}
		\\
		\equiv & \text{ definition of \ensuremath{\mathit{genDTs}}}\\
		& \begin{aligned} & \bigg[\mathit{DN}\left(\mathit{mapD}{}_{\cap_{r^{+}}}\left(u\right),r_{i},\mathit{mapD}{}_{\cap_{r^{-}}}\left(v\right)\right)\mid\\
			& \quad\quad\left(\mathit{rs}^{+},r_{i},\mathit{rs}^{-}\right)\leftarrow\mathit{splits}\left(rs\right),u\leftarrow\mathit{genDTs}\left(\mathit{rs}^{+},\mathit{xs}\right),v\leftarrow\mathit{genDTs}\left(\mathit{rs}^{-},\mathit{xs}\right)\bigg].
		\end{aligned}
	\end{aligned}
	\]
	Although the notation above may appear dense, the only crucial principle
	is the property of homomorphic downward accumulation (Theorem \ref{thm: downward accumulation}).
	The rest follows directly from the definitions, such as the definition
	of the map function $\left[f\left(x\right)\mid x\leftarrow\mathit{xs}\right]=\left[x\mid x\leftarrow\mathit{mapL}\left(f,\mathit{xs}\right)\right]$,
	map composition $\mathit{mapD}_{f\circ g}=\mathit{mapD}_{f}\circ\mathit{mapD}_{g}$
	(which can be proved by simply induction), and the definition of $\mathit{genDTs}\left(\mathit{rs},\mathit{xs}\right)=\mathit{mapL}_{\mathit{dacc}_{\mathit{id},\mathit{sl},\mathit{sr},\cap,\cap}}\left(\mathit{genBTs}^{\prime}\left(\mathit{rs},\mathit{xs}\right)\right)$.
	To follow the proof, we recommend first understanding how the symbols
	are manipulated using these definitions and properties, and then focusing
	on the meaning of each step.
\end{proof}
Running algorithm $\mathit{genDTs}\left(\mathit{rs},\mathit{xs}\right)$
will generate all proper decision trees satisfying the proper decision
tree axioms with respect to a list of rules $\mathit{rs}$. As a results,
substituting $\mathit{genDTs}$ in (\ref{eq: new-specification of genDTKs})
yields a programmatic definition for the search space of the decision
tree $\mathcal{S}$.

\paragraph{Implications of the complete decision tree generator}

The proof of Theorem \ref{thm: comlete DT generator} relies only
on the \emph{structural information} already established in $\ensuremath{\mathit{genBTs}}$,
which uses Axiom 1 (defining $\ensuremath{\mathit{sl}\left(r\right)}$
and $\mathit{sr}\left(r\right)$), and the transitivity property in
Axiom 3. Together, these imply the use of downward accumulation. Since
transitivity and binary splits ensure that information at a leaf is
the accumulated result of the branch nodes along its path, this accumulation
can be computed by path reduction. Specifically, given a path from
the root to a leaf $L$, $\mathcal{P}_{L}=\left\{ r_{i},r_{j},\ldots,r_{p}\right\} $,
the path reduction at this leaf is defined by the map composition
$\mathit{mapD}_{\cap_{r_{i}^{\pm}}\circ\cap_{r_{j}^{\pm}}\ldots\cap_{r_{p}^{\pm}}}=\mathit{mapD}_{\cap_{r_{i}^{\pm}}}\circ\mathit{mapD}_{\cap_{r_{i}^{\pm}}}\ldots\circ\mathit{mapD}_{\cap_{r_{p}^{\pm}}}$.
The generator $\mathit{genDTs}$ shows how this can be calculated
systematically and efficiently for all leaves in a tree.

By contrast, the definition of the $\mathit{splits}$ function (determined
by ancestral constraints, i.e., Axiom 4) is unrelated to the proof
of Theorem \ref{thm: comlete DT generator}. This has an important
consequence: for any decision tree satisfying structural constraints
(Axioms 1--3) , an alternative ancestral constraint (denoted Axiom
4') may replace our Axiom 4 in defining the split function. In such
cases, the search space can be explored by first defining a partial
decision tree generator satisfying Axioms 1, 3, and 4', and then deriving
a complete generator.

This approach allows us to reason about \emph{any} decision tree problem
even in the absence of Axiom 2, since it can be derived indirectly
from other structural information. For instance, in Section \ref{sec:Extension-to-non-proper},
we demonstrate how to analyze non-proper decision tree problems in
which ancestry constraints is not given---corresponds to ordinary
binary labeled trees.

\subsection{A generic dynamic programming algorithm for the proper optimal decision
	tree problem\label{subsec: DP algorithm}}

\subsubsection{A dynamic programming recursion}

The key fusion step in designing a DP algorithm is to combine the
$\mathit{min}_{E}$ function with the generator, thereby preventing
the generation of partial configurations that cannot be extended to
optimal solutions, i.e. optimal solutions to problems can be expressed
purely in terms of optimal solutions to subproblems. This is also
known as the \emph{principle of optimality}, originally investigated
by \citep{bellman1954theory}. Since 1967 \citep{karp1967finite},
extensive study \citep{karp1967finite,ibaraki1977power,bird1996algebra,bird2020algorithm}
shows that the essence of the principle of optimality is \emph{monotonicity}.
In this section, we will explain the role of monotonicity in the decision
tree problem and demonstrate how it leads to the derivation of the
dynamic programming algorithm.

We have previously specified the simplified decision tree problem
in (\ref{sodt --specification}). However this specification concerns
decision tree problems in general, which may not involve any input
data. Since here we aim to derive a dynamic programming algorithm
for solving an optimization problem, we now redefine the $\mathit{sodt}:\left[\mathcal{R}\right]\to\mathcal{D}\to\mathit{DTree}\left(\mathcal{R},\mathcal{D}\right)$
problem by parameterizing it with an input sequence $\mathit{xs}$
(we can use the same reasoning to derive a parameterized $\mathit{sodt}$
from a parameterized $\mathit{odt}_{K}$ specification).
\begin{definition}
	Simplified optimal decision tree problem over decision tree datatype.
\end{definition}
\begin{equation}
	\mathit{sodt}_{rs}=\mathit{\mathit{min}{}_{E}}\circ\mathit{genDTs}_{rs},\label{eq: specification of sodt}
\end{equation}
where $\mathit{min}_{E}:\left[\mathit{DTree}\left(\mathcal{R},\mathcal{D}\right)\right]\to\mathit{DTree}\left(\mathcal{R},\mathcal{D}\right)$
selects the optimal tree returned by $\mathit{genDTs}_{rs}:\mathcal{D}\to\left[\mathit{DTree}\left(\mathcal{R},\mathcal{D}\right)\right]$,
with respect to the objective value calculated by $E$.

The objective function for any decision tree problem conforms to the
following general scheme.
\begin{definition}
	The objective function $E$ is defined as
	\begin{equation}
		\begin{aligned}E & :\mathit{DTree}\left(\mathcal{R},\mathcal{D}\right)\to\mathbb{R}\\
			E & \left(f,g,\mathit{DL}\left(\mathit{xs}\right)\right)=f\left(\mathit{xs}\right)\\
			E & \left(f,g,\mathit{DN}\left(u,r,v\right)\right)=g\left(E\left(\mathit{mapD}{}_{\cap_{r_{i}^{+}}}\left(u\right)\right),E\left(\mathit{mapD}{}_{\cap_{r_{i}^{+}}}\left(v\right)\right)\right).
		\end{aligned}
		\label{eq: objective function}
	\end{equation}
	such that $g\left(a,b\right)\geq\max\left(a,b\right)$, this requirement
	ensures that the objective value of a larger tree is at least as great
	as that of a smaller tree.
\end{definition}
For example, consider the decision tree model for the \emph{classification
	problem}\footnote{Classification is the activity of assigning objects to some pre-existing
	classes or categories. Algorithms for classification problems have
	output restricted to a finite set of values (usually a finite set
	of integers, called labels).}. Like all classification problems, the ultimate goal is to find an
appropriate decision tree that minimizes the number of misclassifications
\citep{breiman1984classification}. Assume each data point is assigned
a label $y_{i}\in\left\{ 1,2,\ldots,K\right\} =\mathcal{K}$, i.e.
$xs=\left[\left(x_{1},y_{1}\right),\left(x_{2},y_{2}\right),\ldots,\left(x_{n},y_{n}\right)\right]$.
Given a tree $\mathit{DN}\left(u,r,v\right)$, we can define this
objective as
\begin{equation}
	\begin{aligned}E^{\prime} & :\mathit{DTree}\left(\mathcal{R},\mathcal{D}\right)\to\mathbb{R}\\
		E^{\prime} & \left(\mathit{DL}\left(\mathit{xs}\right)\right)=\sum_{\left(x_{i},y_{i}\right)\in\mathit{xs}}\mathbf{1}\left[\hat{y}\neq y_{i}\right]\\
		E^{\prime} & \left(\mathit{DN}\left(u,r,v\right)\right)=E^{\prime}\left(\mathit{mapD}{}_{\cap_{r_{i}^{+}}}\left(u\right)\right)+E^{\prime}\left(\mathit{mapD}{}_{\cap_{r_{i}^{+}}}\left(v\right)\right),
	\end{aligned}
\end{equation}
where $\hat{y}=\underset{k\in\mathcal{K}}{\text{argmax}}\sum_{\left(x_{i},y_{i}\right)\in\mathit{xs}}\mathbf{1}\left[y_{i}=k\right]$,
which is the majority class in a leaf. This is the most common decision
tree objective function used in machine learning; alternative objective
functions can also be used.

Based on this definition of the objective function, we show that an
efficient dynamic programming algorithm for the \emph{simplified optimal
	decision tree problem} (\ref{eq: specification of sodt}) can be derived
using the following recursive formulation.
\begin{theorem}
	Dynamic programming algorithm for the simplified optimal decision
	tree problem\emph{. Given a list of rules $\mathit{rs}$, and the
		objective function defined in the form of (\ref{eq: objective function}),
		then the solution of}
	
	\begin{equation}
		\begin{aligned}\mathit{sodt} & :\left[\mathcal{R}\right]\times\mathcal{D}\to\mathit{DTree}\left(\mathcal{R},\mathcal{D}\right)\\
			\mathit{sodt} & \left(\left[\;\right],\mathit{xs}\right)=\left[\mathit{DL}\left(\mathit{xs}\right)\right]\\
			\mathit{sodt} & \left(\left[r\right],\mathit{xs}\right)=\left[\mathit{DN}\left(\mathit{DL}\left(r^{+}\cap\mathit{xs}\right),r,\mathit{DL}\left(r^{-}\cap\mathit{xs}\right)\right)\right]\\
			\mathit{sodt} & \left(\mathit{rs},\mathit{xs}\right)=\mathit{min}{}_{E}\left[\mathit{DN}\left(\mathit{sodt}\left(\mathit{rs}^{+},r_{i}^{+}\cap\mathit{xs}\right),r_{i},\mathit{sodt}\left(\mathit{rs}^{-},r_{i}^{-}\cap\mathit{xs}\right)\right)\mid\left(\mathit{rs}^{+},r_{i},\mathit{rs}^{-}\right)\leftarrow\mathit{splits}\left(\mathit{rs}\right)\right].
		\end{aligned}
		\label{eq: DP sodt}
	\end{equation}
	\emph{is also a solution of} \emph{(\ref{eq: specification of sodt}).
		\label{thm: DP-sodt }}
\end{theorem}
\begin{proof}
	Note that, to apply equational reasoning to the optimization problem,
	we need to modify the $\mathit{min}_{E}$ function to make it into
	a \emph{non-deterministic }(\emph{relational}) function $\mathit{minR}_{E}:\left[\mathit{DTree}\left(\mathcal{R},\mathcal{D}\right)\right]\to\mathit{\mathit{DTree}\left(\mathcal{R},\mathcal{D}\right)}$,
	which selects \emph{one} of the optimal solutions out of a list of
	candidates. Redefining this from scratch would be cumbersome; $\mathit{minR}_{E}$
	is simply introduced to extend our powers of specification and will
	not appear in any final algorithm. It is safe to use as long as we
	remember that $\mathit{minR}_{E}$ returns \emph{one possible }optimal
	solution \citep{bird2020algorithm,bird1996algebra}.
	
	The algorithm is proved by the following equational reasoning
	
	\begin{align*}
		& \begin{aligned} & \mathit{minR}{}_{E}\bigg[\mathit{DN}\left(\mathit{mapD}{}_{\cap_{r_{i}^{+}}}\left(u\right),r_{i},\mathit{mapD}{}_{\cap_{r_{i}^{-}}}\left(v\right)\right)\mid\\
			& \quad\quad\left(\mathit{rs}^{+},r_{i},\mathit{rs}^{-}\right)\leftarrow\mathit{splits}\left(\mathit{rs}\right),u\leftarrow\mathit{genDTs}\left(\mathit{rs}^{+},\mathit{xs}\right),v\leftarrow\mathit{genDTs}\left(\mathit{rs}^{-},\mathit{xs}\right)\bigg]
		\end{aligned}
		\\
		\subseteq & \text{ monotonicity: }E\left(\mathit{mapD}{}_{\cap_{r^{+}}}\left(u\right)\right)\leq E\left(\mathit{mapD}{}_{\cap_{r^{+}}}\left(u^{\prime}\right)\right)\wedge E\left(\mathit{mapD}{}_{\cap_{r^{-}}}\left(v\right)\right)\le E\left(\mathit{mapD}{}_{\cap_{r^{-}}}\left(v^{\prime}\right)\right)\\
		& \qquad\qquad\qquad\qquad\qquad\qquad\qquad\qquad\qquad\Longrightarrow E\left(DN\left(u,r,v\right)\right)\leq E\left(DN\left(u^{\prime},r,v^{\prime}\right)\right)\\
		& \mathit{minR}_{E}\bigg[\mathit{DN}\bigg(\mathit{minR}_{E}\left[\mathit{mapD}{}_{\cap_{r_{i}^{+}}}\left(u\right)\mid u\leftarrow\mathit{genDTs}\left(\mathit{rs}^{+},\mathit{xs}\right)\right],r_{i},\\
		& \quad\quad\mathit{minR}_{E}\left[\mathit{mapD}{}_{\cap_{r_{i}^{-}}}\left(u\right)\mid u\leftarrow\mathit{genDTs}\left(\mathit{rs}^{-},\mathit{xs}\right)\right]\bigg)\mid\left(\mathit{rs}^{+},r_{i},\mathit{rs}^{-}\right)\leftarrow\mathit{splits}\left(rs\right)\bigg]\\
		\equiv & \text{ definition of \ensuremath{\mathit{mapL}}}\\
		& \begin{aligned} & \mathit{minR}_{E}\bigg[\mathit{DN}\bigg(\mathit{minR}_{E}\left(mapL_{\mathit{mapD}{}_{\cap_{r_{i}^{+}}}}\left(\mathit{genDTs}\left(\mathit{rs}^{+},\mathit{xs}\right)\right)\right),r_{i},\\
			& \quad\quad\mathit{minR}_{E}\left(mapL_{\mathit{mapD}{}_{\cap_{r_{i}^{-}}}}\left(\mathit{genDTs}\left(\mathit{rs}^{-},\mathit{xs}\right)\right)\right)\bigg)\mid\left(\mathit{rs}^{+},r_{i},\mathit{rs}^{-}\right)\leftarrow\mathit{splits}\left(\mathit{rs}\right)\bigg]
		\end{aligned}
		\\
		\equiv & \text{ definition of \ensuremath{\mathit{mapD}}}\\
		& \begin{aligned} & \mathit{minR}_{E}\bigg[\mathit{DN}\bigg(\mathit{minR}_{E}\left(\mathit{genDTs}\left(\mathit{rs}^{+},r_{i}^{+}\cap\mathit{xs}\right)\right),r_{i},\\
			& \quad\quad\mathit{minR}_{E}\left(\mathit{genDTs}\left(\mathit{rs}^{-},r_{i}^{-}\cap\mathit{xs}\right)\right)\bigg)\mid\left(\mathit{rs}^{+},r_{i},\mathit{rs}^{-}\right)\leftarrow\mathit{splits}\left(\mathit{rs}\right)\bigg]
		\end{aligned}
		\\
		\equiv & \text{ definition of \ensuremath{\mathit{sodt}}}\\
		& \mathit{minR}_{E}\bigg[\mathit{DN}\bigg(sodt\left(\mathit{rs}^{+},r_{i}^{+}\cap xs\right),r_{i},\left(sodt\left(\mathit{rs}^{-},r_{i}^{-}\cap\mathit{xs}\right)\right)\bigg)\mid\left(\mathit{rs}^{+},r_{i},\mathit{rs}^{-}\right)\leftarrow\mathit{splits}\left(\mathit{rs}\right)\bigg].
	\end{align*}
\end{proof}
This DP algorithm (\ref{eq: DP sodt}) recursively constructs the
ODT from optimal subtrees $\mathit{sodt}\left(rs^{+},r_{i}^{+}\cap\mathit{xs}\right)$
and $\mathit{sodt}\left(\mathit{rs}^{-},r_{i}^{-}\cap\mathit{xs}\right)$
with respect to a smaller data set $r_{i}^{+}\cap\mathit{xs}$ and
$r_{i}^{-}\cap\mathit{xs}$ respectively. Algorithm (\ref{eq: DP sodt})
is more efficient than (\ref{eq: specification of sodt}) because
it eliminates all non-optimal partial configurations by applying $\mathit{min}_{E}$
within the definition of $\mathit{sodt}$, rather than after exhaustively
generating all possible decision trees using $\mathit{genDTs}_{rs}$
as in (\ref{eq: specification of sodt}).

Combining Theorems \ref{thm: DP-sodt } and \ref{thm: Simplified-decision-tree problem}
completes the proof for Theorem \ref{thm: sodt-introduction}.

\subsubsection{Applicability of the memoization technique\label{subsec:Applicability-of-the}}

In the computer science community, DP is widely recognized as recursion
with overlapping subproblems, combined with memoization to avoid re-computations
of subproblems. If both conditions are satisfied, we say, a DP solution
exists.

At first glance, the ODT problem involves shared subproblems, suggesting
that a DP solution is possible. However, we will explain in this section
that, despite the existence of these shared subproblems, \emph{memoization}
is impractical for most of optimal decision tree problems.

Below, we analyze why this is the case, using a counterexample where
the memoization technique \textbf{is} \textbf{applicable}---the matrix
chain multiplication problem (MCMP)---and discuss the key differences.

In DP algorithms, a key requirement often overlooked in literature
is that the optimal solution to one subproblem must be \emph{equivalent
}to the optimal solution to another. For example, in the matrix chain
multiplication problem, the goal is to determine the most efficient
way to multiply a sequence of matrices. Consider multiplying four
matrices $A$, $B$, $C$, and $D$. The equality $\left(\left(AB\right)C\right)D=\left(AB\right)\left(CD\right)$
states that two ways of multiplying the matrices will yield the same
result.

Because the computations involved may differ due to varying matrix
sizes, the computation on one side may be more efficient than the
other. Nonetheless, our discussion here is not focused on the computational
complexity of this problem. One of the key components of the DP algorithm
for MCMP is that the computational result $\left(AB\right)$ can be
reused. This is evident as $\left(AB\right)$ appears in both $\left(\left(AB\right)C\right)D$
and $\left(AB\right)\left(CD\right)$. It is therefore possible to
compute the result for the subproblem $\left(AB\right)$ first, and
then directly use it in the subsequent computations of $\left(\left(AB\right)C\right)D$
and $\left(AB\right)\left(CD\right)$, thereby avoiding the recomputation
of $\left(AB\right)$.

However, in the decision tree problem, the monotonicity used in the
derivation of (\ref{eq: DP sodt}),
\begin{align*}
	& \ensuremath{E\left(\mathit{mapD}{}_{\cap_{r^{+}}}\left(u\right)\right)\leq E\left(\mathit{mapD}{}_{\cap_{r^{\prime+}}}\left(u^{\prime}\right)\right)\wedge E\left(\mathit{mapD}{}_{\cap_{r^{-}}}\left(v\right)\right)\leq E\left(\mathit{mapD}{}_{\cap_{r^{\prime-}}}\left(v^{\prime}\right)\right)}\\
	& \Longrightarrow E\left(DN\left(u,r,v\right)\right)\leq E\left(DN\left(u^{\prime},r^{\prime},v^{\prime}\right)\right)
\end{align*}
holds only if the root $r$ in tree $DN\left(u,r,v\right)$ is the
same as the root $DN\left(u^{\prime},r^{\prime},v^{\prime}\right)$.
When applying memoization, many cases arise where the roots of the
subtrees differ, because most ODT instances have a very large number
of candidate roots, and each root produces a different partition of
the dataset. To apply memoization correctly, both the root and the
subtrees must be included as indices to ensure correct reuse.

Therefore, to use the memoization technique, we need to store not
only the optimal solution of a subtree generated by a given set of
rules $rs$, but also the root of each subtree. This requires at least
$O\left(\sum_{k\in\mathcal{K}}\left|\mathcal{S}_{\text{Dtree}}\left(k\right)\right|\times\left|\mathcal{S}_{\mathcal{H}}\right|\right)$
space, where $\left|\mathcal{S}_{\text{Dtree}}\left(k\right)\right|$
and $\left|\mathcal{S}_{\mathcal{H}}\right|$ are the number of possible
decision trees with respect to $k$ splitting rules and the number
of possible roots, respectively, with $\mathcal{K}=\left\{ 1,\ldots,K\right\} $.
Thus, storing all this information during the algorithm's runtime
is impractical in terms of \emph{space} \emph{complexity} for most
decision tree problems considered in ML. For example, a hyperplane
decision tree problem can involve $O\left(N^{D}\right)$ possible
splitting rules and $O\left(N^{DK}\right)$possible subtrees in the
worst case. Storing such a large number of trees is not only impractical,
the likelihood of reusing information by caching a small proportion
of subtrees is extremely low, since both the root and the subtrees
must match exactly.

Consequently, while \citet{demirovic2022murtree} claim a trade-off
between \emph{branch} and \emph{dataset} caching, as introduced by
\citet{nijssen2007mining}, we argue this claim is incorrect because
both the root (which determines the dataset) and the branch (which
determines the subtrees) must be stored properly. Indeed, in their
subsequent study, \citet{brita2025optimal} implement caching by
reusing solutions indexed by both sub-dataset and subtree depth. Since
the problem in \citet{brita2025optimal} has far more candidate splitting
rules than that studied by \citet{demirovic2022murtree}, the caching
approach used in \citet{brita2025optimal}---which stores only a
few thousand optimal subtrees---has a very low probability of cache
hits. Our tests on multiple datasets confirm this: in extreme cases,
their algorithm examined millions of decision trees without a single
cache hit. This suggests that the efficiency attributed to caching
in the ConTree algorithm is overstated; most performance gains likely
arise from factors other than caching.

\subsubsection{Complexity of the optimal decision tree algorithm\label{subsec:Complexity-of-the}}

According to Theorem (\ref{thm: sodt-introduction}), analyzing the
complexity of the $\mathit{odt}_{K}$ algorithm for solving the optimal
size-$K$ decision tree problem requires a detailed examination of
the complexity of its components: $\mathit{min}_{E}$, $\mathit{sodt}$,
and $\mathit{kcombs}_{K}$. The complexities of $\mathit{min}_{E}$
and $\mathit{kcombs}_{K}$ are fixed once the input rule list $\mathit{rs}$
(of size $M\geq K$), representing the set of available rules for
constructing the tree, is given. We now analyze the complexity of
$\mathit{sodt}\left(\mathit{rs}_{K},\mathit{xs}\right)$, where $K$
is a size $\mathit{rs}_{K}$ rule list generated by $\mathit{kcombs}_{K}\left(\mathit{rs}\right)$
and $\mathit{xs}$ is a length $N$ list.

\paragraph{Complexity of $\mathit{sodt}$}

Precisely characterizing the average- or best-case combinatorial complexity
of the decision tree problem is difficult because it depends heavily
on the specific ancestry relation matrix $\boldsymbol{K}$. Unless
assumptions are made about $\boldsymbol{K}$, we can only analyze
the worst-case complexity. This is captured by the following lemma.
\begin{lemma}
	\emph{Given $K$ possible splitting rules $\mathit{rs}_{K}$ that
		satisfy the proper decision tree axioms, the search space $\mathcal{S}\left(K,\mathit{rs}_{K}\right)$
		achieves maximum combinatorial complexity when any rule can serve
		as the root and each branch node has exactly one child. Formally,
		for any $r_{i}\in\mathit{rs}_{K}$ , we have $\boldsymbol{K}_{ij}=1$
		or} $\boldsymbol{K}_{ij}=-1$ for all\emph{ }$r_{j}\in\mathit{rs}_{K}/r_{i}$\emph{,
		and $\sum_{j\in\mathit{rs}}\left|\boldsymbol{K}_{ij}\right|=\left|\mathit{rs}_{K}/r_{i}\right|$,
		which implies $\nexists r_{j}\in\mathit{rs}_{K}/r_{i}:\boldsymbol{K}_{ij}=0$.
		Here $\left|\boldsymbol{K}_{ij}\right|$ denotes the absolute value
		and $\left|\mathit{rs}_{K}/r_{i}\right|$ denotes the size of the
		list. \label{lem: complexity}}
\end{lemma}
\begin{proof}
	Consider the case where for any $r_{i}\in\mathit{rs}$ , we have $\boldsymbol{K}_{ij}=1$
	or $\boldsymbol{K}_{ij}=-1$ for all $h_{j}\in hs$, $i\neq j$, and
	\emph{$\sum_{j\in\mathit{rs}}\left|\boldsymbol{K}_{ij}\right|=\left|\mathit{rs}_{K}/r_{i}\right|$}.
	Under these conditions, each subtree has exactly one child, resulting
	in a tree with a single path (excluding leaf nodes). Since the structure
	is fully determined by the branch nodes, we can disregard the leaf
	nodes. This configuration permits any permutation of branch nodes,
	yielding maximum combinatorial complexity. We demonstrate this by
	proving that separating two rules in different branches will reduce
	the problem's complexity.
	
	For a $K$-permutation $p$, consider first the case of a chain of
	decision rules where each node has exactly one child. Given our assumption
	that any splitting rule can serve as the root, all permutations of
	the decision tree are valid, resulting in $K!$ possible chains. For
	the alternative case, consider a permutation $p=\left[\ldots,r_{j},r_{k},\ldots\right]$
	where rules $r_{j}$ and $r_{k}$ occupy the same level with immediate
	ancestor $r_{i}$. By Theorem \ref{thm: main theorem}, these rules
	must be the left and right children of $r_{i}$, respectively, and
	their positions are immutable. When $r_{i}$ precedes both $r_{k}$
	and $r_{j}$ in the permutation, $r_{j}$ and $r_{k}$ will always
	be separated into different branches. In the worst case, $r_{k}$
	and $r_{j}$ are at the tail of the permutation list, i.e. $p^{\prime}=\left[\ldots,r_{j},r_{k}\right]$.
	Thus, when the permutation $\left[\ldots,r_{k},r_{j}\right]$ is not
	allowed, all permutations where $\left(K-2\right)$ rules precede
	both $r_{k}$ and $r_{j}$ become invalid, eliminating $\left(K-2\right)!$
	possible permutations. As additional pairs of rules become constrained
	to the same level, the number of invalid permutations increases monotonically.
	Therefore, the decision tree attains maximal combinatorial complexity
	when it assumes a ``chain'' structure, where each non-leaf node has
	exactly one child node.
\end{proof}
Lemma \ref{lem: complexity} implies that the search space $\mathcal{S}\left(K,\mathit{rs}_{K}\right)$
has a worse-complexity of $O\left(K!\right)$, and all rules $\mathit{rs}_{K}/r_{i}$
are classified into one branches once $r_{i}$ is fixed as root. Therefore,
assume $r_{i}^{\pm},\forall i\in\left\{ 1,\ldots,K\right\} $are pre-computed
and can be indexed in $O\left(1\right)$ time, and denote by $T\left(K\right)$
the worst-case complexity of $\mathit{sodt}$ with respect to $\mathit{rs}_{K}$,
and size $N$ input data $\mathit{xs}$, the following recurrence
applies,
\[
\begin{aligned}T & \left(1\right)=O\left(1\right)\\
	T & \left(K\right)=K\times\left(T\left(K-1\right)+T\left(1\right)\right)+O\left(N\right),
\end{aligned}
\]
with solution $T\left(K\right)=O\left(K!\times N\right)$, where $O\left(N\right)$
is the complexity for calculating $\mathit{r}^{+}$ and $r^{-}$.
While this complexity is factorial in $K$, it is important to note
that the worst-case scenario occurs only when the tree consists of
a \textbf{\emph{single}} path of length $K$. However, such a tree
is generally considered the least useful solution in practical decision
tree problems, as it represents an extremely deep and narrow structure.

In most cases, decision trees that are as shallow as possible are
preferred, as shallow trees are typically more interpretable. Deeper
trees tend to become less interpretable, particularly when the number
of nodes increases. Therefore, while the worst-case complexity is
factorial, it does not necessarily represent the typical behavior
of decision tree generation in practical scenarios, where the goal
is often to minimize tree depth for improved clarity and efficiency.

\paragraph{Complexity of $\mathit{odt}$}

We can now analyze the worst-case complexity of the main program $\mathit{odt}\left(K,\mathit{rs}\right)$.
Consider a rule list of size $M$, $\mathit{rs}=\left[r_{1},\ldots,r_{M}\right]$,
where $\mathit{kcombs}\left(K,\mathit{rs}\right)$ generates $\left(\begin{array}{c}
	M\\
	K
\end{array}\right)$ possible size-$K$ splitting rule subsets $\mathit{rs}_{K}$ in $O\left(\left(\begin{array}{c}
	M+1\\
	K+1
\end{array}\right)\right)$ time \citep{ruskey2003combinatorial}. Since $\mathit{concatMapL}\left(\mathit{sodt},\mathit{kcombs}\left(K,\mathit{rs}\right)\right)$
applies $\mathit{sodt}$ to each $\mathit{rs}_{K}\in\mathit{kcombs}\left(K,\mathit{rs}\right)$the
resulting worst-case complexity is $O\left(\left(\begin{array}{c}
	M+1\\
	K+1
\end{array}\right)+\left(\begin{array}{c}
	M\\
	K
\end{array}\right)\times K!\times N\right)$. Finally, $\mathit{min}_{E}$ selects the optimal decision tree from
the $O\left(M^{K}\right)$ solutions returned by $\mathit{concatMapL}\left(\mathit{sodt},\mathit{kcombs}\left(K,\mathit{rs}\right)\right)$.
Assuming $E$ has $O\left(1\right)$ complexity---which can be achieved
by pairing each tree with its corresponding objective during construction
in $\mathit{sodt}$---the total complexity of $\mathit{odt}$ becomes
\[
O\left(\left(\begin{array}{c}
	M+1\\
	K+1
\end{array}\right)+\left(\begin{array}{c}
	M\\
	K
\end{array}\right)\times K!\times N+\left(\begin{array}{c}
	M\\
	K
\end{array}\right)\right)=O\left(K!\times M^{K}\times N\right)
\]
If $K$ is a fixed constant, the overall complexity of $\mathit{odt}\left(K,\mathit{rs}\right)$
is polynomial in $M$, but factorial in $K$. If $\mathit{r}^{+}$
and $r^{-}$ can be calculated in $O\left(1\right)$ time, then the
algorithm has a complexity of $O\left(K!\times M^{K}\right)$.

\subsection{Further speed-up---prefix-closed filtering and the thinning method}

\subsubsection{Prefix-closed filtering}

In ML research, to prevent \emph{overfitting}, a common approach is
to impose a constraint that the number of data points in each leaf
node must exceed a fixed size, $N_{\min}$, to avoid situations where
a leaf contains only a small number of data points. One straightforward
method to apply this constraint is to incorporate a filtering process
by defining

\begin{equation}
	\mathit{genDTFs}{}_{N_{\text{min}},rs}=\mathit{filter}{}_{q_{N_{\text{min}}}}\circ\mathit{genDTs}_{rs}.
\end{equation}

However, this direct specification is not ideal, as $\mathit{genDTs}_{\mathit{rs}}$
can potentially generate an extremely large number of trees, making
post-generation filtering computationally inefficient. To make this
program efficient, it is necessary to fuse the post-filtering process
inside the generating function. It is well-known in various fields
\citep{bird1996algebra,ruskey2003combinatorial,bird2020algorithm}
that if the one-step update function in a recursion ``\emph{reflects}''
a predicate $p$, then the filtering process can be incorporated directly
into the recursion. This approach allows for the elimination of infeasible
configurations before they are fully generated.

In this context, we say that ``$f$ reflects $p$'' if $p\left(f\left(\mathit{DN}\left(u,r,v\right)\right)\right)\implies p\left(u\right)\wedge p\left(v\right)$,
where $f$ is defined as $f\left(\mathit{DN}\left(u,r,v\right)\right)=\mathit{DN}\left(\mathit{mapD}{}_{\cap_{r_{i}^{+}}}\left(u\right),r_{i},\mathit{mapD}{}_{\cap_{r_{i}^{-}}}\left(v\right)\right)$
in the $\mathit{genDTs}$ function. Since the number of data points
in each leaf decreases as more splitting rules are introduced, it
is trivial to verify that the implication holds.

As a result, the filtering process can be integrated into the generator,
and the new generator, after fusion, is defined as
\[
\begin{aligned}\mathit{genDTFs}{}_{N_{\text{min}}} & :\left[\mathcal{R}\right]\times\mathcal{D}\to\left[\mathit{DTree}\left(\mathcal{R},\mathcal{D}\right)\right]\\
	\mathit{genDTFs}{}_{N_{\text{min}}} & \left(\left[\;\right],\mathit{xs}\right)=\left[\mathit{\mathit{DL}}\left(xs\right)\right]\\
	\mathit{genDTFs}{}_{N_{\text{min}}} & \left(\left[r\right],\mathit{xs}\right)=\left[\mathit{DN}\left(\mathit{DL}\left(r^{+}\cap\mathit{xs}\right),r,\mathit{DL}\left(r^{-}\cap\mathit{xs}\right)\right)\right]\\
	\mathit{genDTFs}{}_{N_{\text{min}}} & \left(\mathit{rs},\mathit{xs}\right)=\mathit{filter}{}_{q_{N_{\text{min}}}}\Big[\mathit{DN}\left(\mathit{mapD}{}_{\cap_{r_{i}^{+}}}\left(u\right),r_{i},\mathit{mapD}{}_{\cap_{r_{i}^{-}}}\left(v\right)\right)\mid\\
	& \begin{aligned} & \quad\quad\quad\quad\quad\quad\quad\quad\quad\quad\quad\left(\mathit{rs}^{+},r_{i},\mathit{rs}^{-}\right)\leftarrow\mathit{splits}\left(\mathit{rs}\right),\\
		& \quad\quad\quad\quad\quad\quad\quad\quad\quad\quad\quad u\leftarrow\mathit{genDTFs}{}_{N_{\text{min}}}\left(\mathit{rs}^{+},\mathit{xs}\right),\;v\leftarrow\mathit{genDTFs}{}_{N_{\text{min}}}\left(\mathit{rs}^{-},\mathit{xs}\right)\Big]\Bigg).
	\end{aligned}
\end{aligned}
\]
where $q_{N_{\text{min}}}:\mathit{DTree}\left(\mathcal{R},\mathcal{D}\right)\to\mathit{Bool}$
return True if all leafs in a tree contains data points greater than
$N_{\min}$.

Substituting definition $\mathit{genDTFs}{}_{N_{\text{min}}}$ into
the derivation of $\mathit{sodt}$ could potentially give a more efficient
definition for $\mathit{sodt}$, as $\mathit{genDTFs}{}_{N_{\text{min}}}$
generates provably less configurations than $\mathit{genDTs}$.

Alternatively, one can also incorporate a tree-depth constraint. It
is trivial to verify that the predicate defining the tree-depth constraint
will also be reflected by $f$, as adding more branch nodes will inevitably
increase the tree depth. In other words, we have $q^{\prime}\left(f\left(\mathit{DN}\left(u,r,v\right)\right)\right)\implies q^{\prime}\left(u\right)\wedge q^{\prime}\left(v\right)$,
where $q^{\prime}$ calculate the tree depth.

\subsubsection{The thinning method}

The \emph{thinning} \emph{algorithm\index{thinning algorithm@thinning\emph{ }algorithm}}
is equivalent to the exploitation of \emph{dominance relations }in
the algorithm design literature \citep{ibaraki1977power,bird1996algebra}.
The use of thinning or dominance relations is concerned with improving
the time complexity of naive DP algorithms \citep{galil1989speeding,DEMOOR19993}.

The thinning technique exploits the fundamental fact that certain
partial configurations are superior to others, and it is a waste of
computational resources to extend these non-optimal partial configurations.
The thinning relation can be introduced into an optimization problem
by the following

\begin{equation}
	\mathit{genDTTs}_{r}=\mathit{thin}_{r}\circ\mathit{genDTs},
\end{equation}
where $\mathit{thin}_{r}:\left[\mathcal{A}\right]\to\left[\mathcal{A}\right]$
and $r:\mathcal{A}\to\mathcal{A}\to\mathit{Bool}$ is a Boolean-valued
binary function. Following \citet{bird1996algebra}'s result, if we
can find a \emph{relation} $r$ which is a \emph{preorder} and satisfies
\emph{monotonicity} \citet{bird1996algebra}, then
\begin{equation}
	\begin{aligned}\mathit{genDTTs}_{r} & :\left[\mathcal{R}\right]\times\mathcal{D}\to\left[\mathit{DTree}\left(\mathcal{R},\mathcal{D}\right)\right]\\
		\mathit{genDTTs}_{r} & \left(\left[\;\right],\mathit{xs}\right)=\left[\mathit{DL}\left(\mathit{xs}\right)\right]\\
		\mathit{\mathit{genDTTs}_{r}} & \left(\left[r\right],\mathit{xs}\right)=\left[\mathit{DN}\left(\mathit{DL}\left(r^{+}\cap\mathit{xs}\right),r,\mathit{DL}\left(r^{-}\cap\mathit{xs}\right)\right)\right]\\
		\mathit{genDTTs}_{r} & \left(\mathit{rs},\mathit{xs}\right)=\mathit{thin}_{r}\bigg[\mathit{\mathit{DN}}\left(\mathit{mapD}{}_{\cap_{r_{i}^{+}}}\left(u\right),r_{i},\mathit{mapD}{}_{\cap_{r_{i}^{-}}}\left(v\right)\right)\mid\\
		& \begin{aligned} & \quad\quad\quad\left(\mathit{rs}^{+},r_{i},\mathit{rs}^{-}\right)\leftarrow\mathit{splits}\left(\mathit{\mathit{rs}}\right),u\leftarrow\mathit{genDTTs}\left(\mathit{rs}^{+},\mathit{xs}\right),v\leftarrow\mathit{genDTTs}\left(\mathit{rs}^{-},\mathit{xs}\right)\bigg].\end{aligned}
	\end{aligned}
\end{equation}

Again, substituting the definition of $\mathit{genDTTs}_{r}$ in the
derivation of $\mathit{sodt}$ could potentially yield a more efficient
definition for $\mathit{sodt}$, as $\mathit{genDTTs}_{r}$ provably
generates fewer configurations than $\mathit{genDTs}$. However, whether
the program actually runs faster depends on the implementation of
$\mathit{thin}_{r}$ and the specific application, as the complexity
of $\mathit{thin}_{r}$ is nontrivial, since removing more configurations
requires additional computations. One example definition of the thinning
algorithm can be found in \citet{bird2020algorithm}.

Thinning is different from $\mathit{min}_{E}$. Indeed, the $\mathit{min}_{E}$
function can be understood as a special thinning function with respect
to a \emph{total} \emph{order} defined by the objective function $E$,
whereas the thinning is based on a\emph{ preorder} $r$. In a preorder
relation, some configurations are not comparable. Thus, $\mathit{thin}_{r}:\left[\mathcal{A}\right]\to\left[\mathcal{A}\right]$
receives a list and returns a list, whereas $\mathit{min}_{E}:\left[\mathcal{A}\right]\to\mathcal{A}$
always returns a single element.

\section{Applications of proper decision trees\label{sec: Applications}}

\subsection{The binary space partition problem\label{subsec:Binary-space-partition}}

\begin{figure}[h]
	\centering{}\includegraphics[scale=0.25]{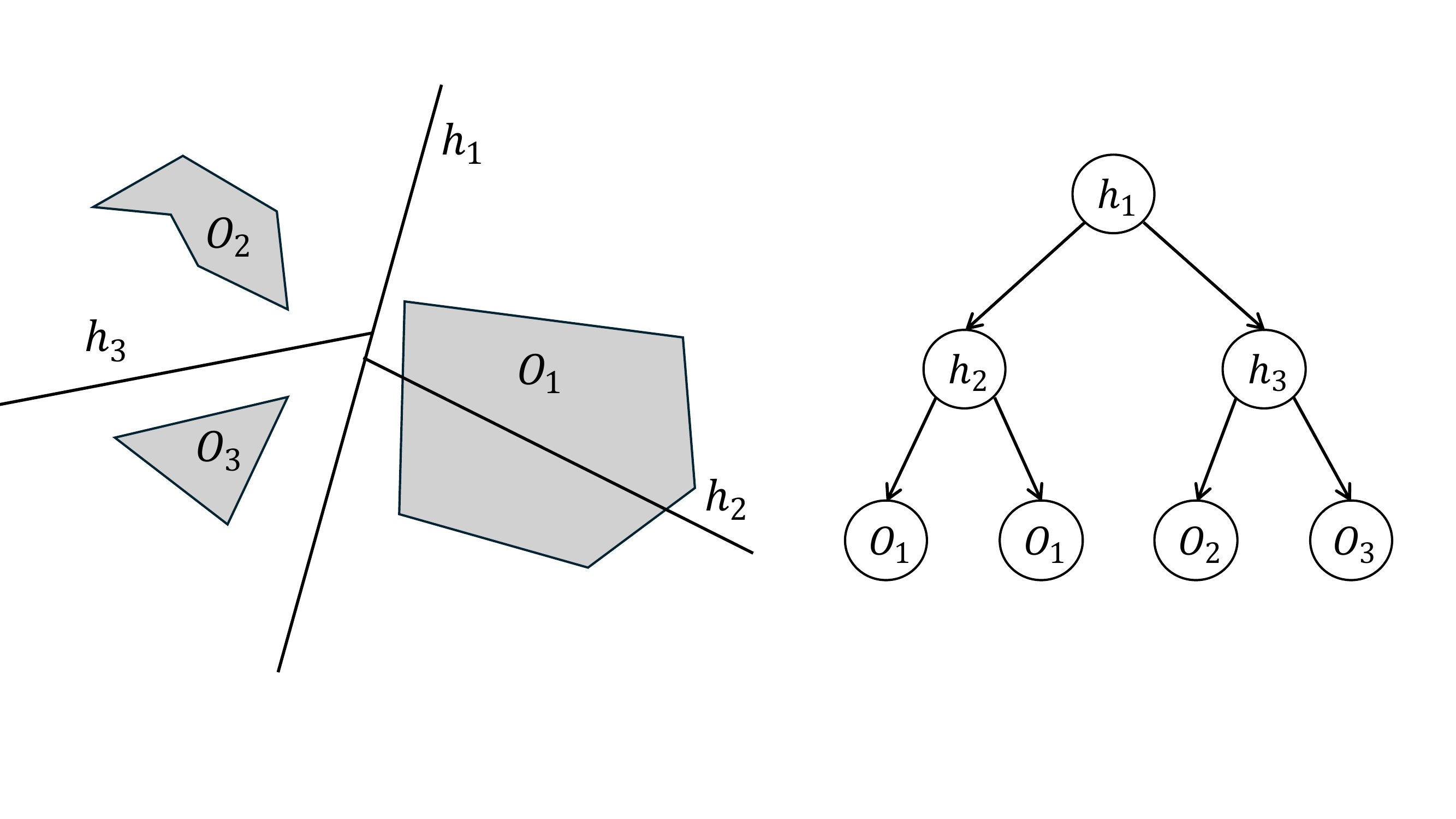}\caption{A binary space partition (left) of objects $O_{1}$, $O_{2}$ and
		$O_{3}$ using hyperplanes $h_{1}$, $h_{2}$, and $h_{3}$, and the
		corresponding binary space partition tree (right).\label{fig: BSP-polygon}}
\end{figure}

\emph{Binary} \emph{space} \emph{partitioning} (BSP) arose from the
need in computer graphics to rapidly draw three-dimensional scenes
composed of some physical objects. A simple way to draw such scenes
is \emph{painter's algorithm}: draw polygons in order of distance
from the viewer, back to front, painting over the background and previous
polygons with each closer object. The objects are then scanned in
this so-called \emph{depth} \emph{order}, starting with the one farthest
from the viewpoint.

However, successfully applying painter's algorithm depends on the
ability to quickly sort objects by depth, which is not always trivial.
In some cases, a strict depth ordering may not exist. In such cases,
objects must be subdivided into pieces before sorting. To implement
the painter's algorithm in a real-time environment, such as flight
simulation, preprocessing the scene is essential to ensure that a
valid rendering order can be determined efficiently for any viewpoint.

A BSP tree provides an elegant solution to this problem, which is
essentially a decision tree in which each leaf node contains at most
one polygon (or it can be empty), with splitting rules defined by
hyperplanes. For instance, consider the 2D case; the left panel of
Figure \ref{fig: BSP-polygon} illustrates the situation where the
splitting rules are defined by hyperplanes, and the objects are polygons.

\begin{figure}[h]
	\centering{}\includegraphics[scale=0.25]{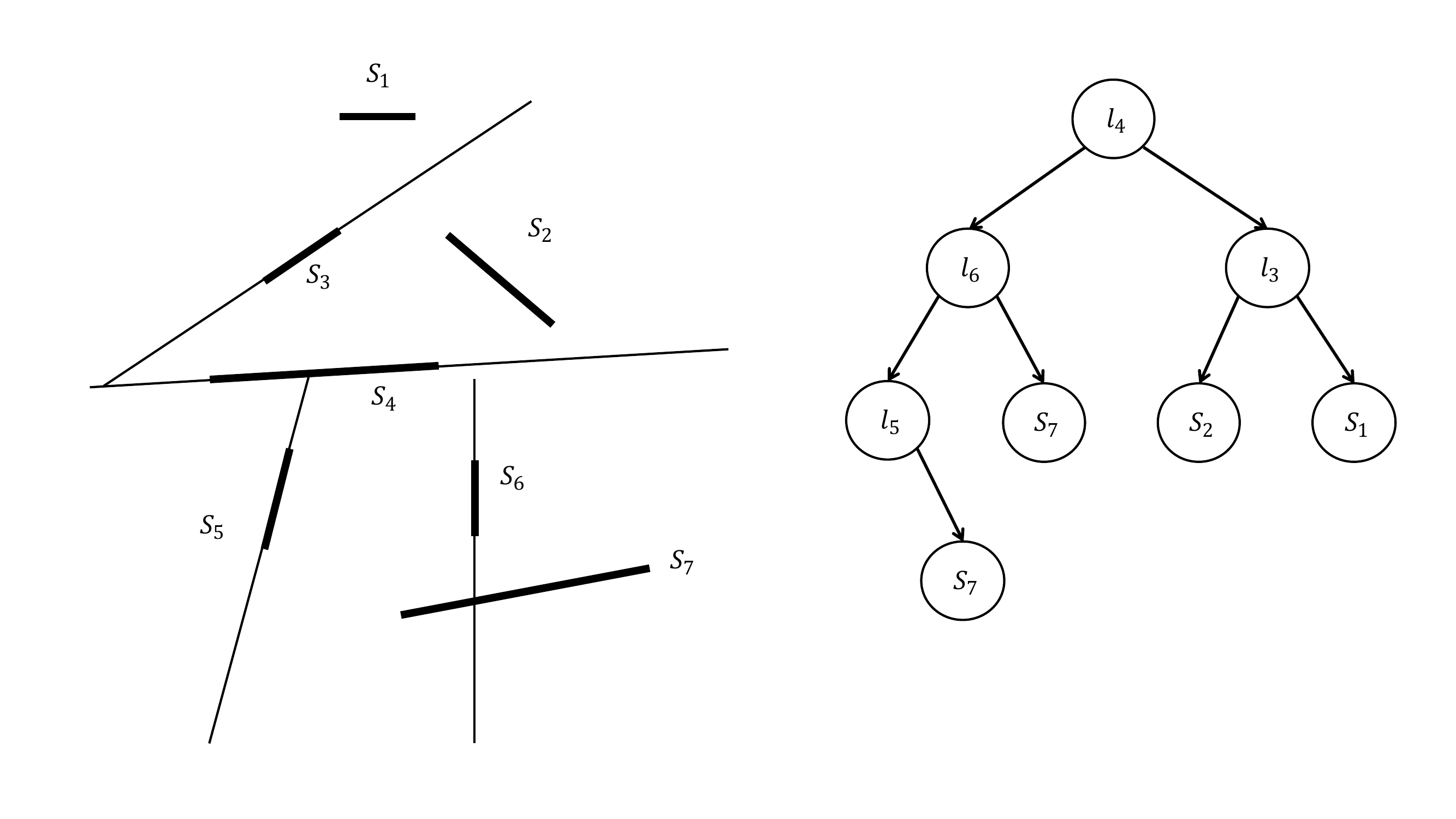}\caption{The auto-partition (left) for segments $S_{1}$, $S_{2}$, $S_{3}$,
		$S_{4}$, $S_{5}$, $S_{6}$, $S_{7}$, and the corresponding binary
		space partition tree (right). We denote the extending lines for segment
		$S_{i}$ as $l_{i}$. \label{fig: BSP-segments}\protect \\
	}
\end{figure}

To make painter's algorithm efficient, the resulting BSP tree should
be as small as possible in the sense that it has a minimal number
of leaf nodes and splitting rules. In theory, the splitting rules
used to define the BSP tree can be arbitrary. However, since BSP is
primarily applied to problems that require highly efficient solutions---such
as dynamically rendering a scene in real time---the splitting rules
are typically chosen based on segments (or, in the three-dimensional
case, affine flats created by 2D polygons) present in the diagram.
A BSP tree that uses only these segments to define splitting rules
is called an \emph{auto-partition}, and we will refer to these rules
as \emph{auto-rules}. For example, as shown in Figure \ref{fig: BSP-segments},
when the objects being partitioned are segments, the auto-rules generated
by these segments are their extending lines.

The most common algorithm for creating a BSP tree involves randomly
choosing permutations of auto-rules and then selecting the best permutation
\citep{motwani1996randomized,de2000computational}, although the exhaustiveness
of permutations has not been properly analyzed in any previous research.
While auto-partitions cannot always produce a minimum-size BSP tree,
previous probability analyses have shown that the BSP tree created
by randomly selecting auto-rules can still produce reasonably small
trees, with an expected size of $O\left(N\log N\right)$ for 2D objects
and $O\left(N^{2}\right)$ for 3D objects, where $N$ is the number
of auto-rules \citep{motwani1996randomized}.

For a BSP tree, any splitting rule can become the root, but some segments
may split others into two, thereby creating new splitting rules, as
seen in Figure \ref{fig: BSP-segments}, where segment $S_{7}$ is
split into two. Therefore, we need to modify the $\mathit{splits}$
function by defining it as
\begin{equation}
	\mathit{splits}_{\text{BSP}}\left(\mathit{rs}\right)=\left[\left(\mathit{sp}_{\text{BSP}}\left(r_{i},\mathit{rs}\right),r_{i},\mathit{sn}_{\text{BSP}}\left(r_{i},\mathit{rs}\right)\right)\mid r_{i}\leftarrow rs\right],
\end{equation}
where $\mathit{sp}_{\text{BSP}}:\mathcal{R}\to\left[\mathcal{R}\right]\to\left[\mathcal{R}\right]$
and $\mathit{sn}_{\text{BSP}}:\mathcal{R}\to\left[\mathcal{R}\right]\to\left[\mathcal{R}\right]$
are short for ``split positive'' and ``split negative'', respectively.
These functions take a splitting rule $r$ and a list of rules $\mathit{rs}$
and return all segments lying on the positive and negative sides of
$r$, respectively, including the newly generated rules. At the same
time, we need to modify the objective function by simply counting
the number of leaf nodes and branch nodes

\begin{equation}
	\begin{aligned}E_{\text{BSP}} & :\mathit{DTree}\left(\mathcal{R},\mathcal{D}\right)\to\mathbb{N}\\
		E_{\text{BSP}} & \left(\mathit{DL}\left(\mathit{xs}\right)\right)=1\\
		E_{\text{BSP}} & \left(\mathit{DN}\left(u,r,v\right)\right)=E\left(u\right)+E\left(v\right)+1.
	\end{aligned}
\end{equation}

The BSP tree produced by the $\mathit{sodt}$ algorithm can, by definition,
achieve the minimal size tree with respect to a given set of auto-rules,
with a worst-case complexity of $O\left(K!\times N\right)$. By contrast,
the classical randomized algorithm always checks all possible permutations
in all scenarios to obtain the minimal BSP tree, requiring provably
more computations compared to the worst-case scenario of the $\mathit{sodt}$
algorithm. This is because calculating permutations involves additional
steps to transform them into trees, and several permutations may correspond
to the same tree.

Moreover, the BSP tree is a very general data structure that encompasses
several well-known structures, including the $K$-D tree, the max-margin
tree ($\mathit{MM}$-tree), and random-projection tree (RP-tree) \citep{fan2018binary,fan2019binary}.
We will explore how to construct optimal $K$-D trees later.

\subsection{Optimal decision tree problems with axis-parallel, hyperplane or
	hypersurface splitting rules in machine learning\label{subsec: ODT in ML}}

As discussed in the introduction, due to the intractable combinatorics
of the decision tree problem, studies on the ODT problem with even
axis-parallel splitting rules are scarce, let alone research on the
ODT problem for hyperplanes or more complex hypersurface splitting
rules.

However, the more complex the splitting rules, the simpler and more
accurate the resulting tree tends to be. To illustrate, Figure \ref{fig:DT in ML}
three different decision tree models---the axis-parallel, the hyperplane,
and the hypersurface decision tree (defined by a degree-two polynomial)---used
to classify the same dataset. As the complexity of the splitting rule
increases, the resulting decision tree becomes simpler and more accurate.

Unlike the BSP problem, where splitting rules are predefined, for
the ODT problem in ML, the splitting rules are unknown. The algorithm
must learn the best set of rules that will yield the best partition.
Therefore, we need a separate process to generate all possible splitting
rules in $\mathbb{R}^{D}$, which is precisely the $\mathit{genRules}$
function mentioned in the introduction. At first glance, the number
of possible splitting rules for any given type appears infinite, as
the space is continuous. Despite the apparent infinitude of possible
splitting rules, the finiteness of the dataset constrains the number
of distinct partitions that these rules can generate. This implies
the existence of equivalence classes among different rules.

In the axis-parallel decision tree problem, $\mathit{genRules}_{\text{AODT}}$
can be defined as the set of all axis-parallel hyperplanes passing
through each data point $x\in\mathbb{R}^{D}$ of a size $N$ data
list $\mathit{xs}$, giving a total of $N\times D$ hyperplanes. Concretely,
for a data list $\mathit{xs}$, with points $x=\left(x_{1},\ldots,x_{D}\right)$,
we have
\[
\mathit{genRules}_{\text{AODT}}\left(\mathit{xs}\right)=\left[x_{i}\mid x_{i}\leftarrow x,x\leftarrow\mathit{xs}\right],
\]
Equivalently, \citet{brita2025optimal,mazumder2022quant} use the
medians of adjacent pairs of data points along each dimension $D$.

For more general settings, \citet{he2025ROF,he2023efficient} showed
that the number of equivalence classes for hyperplane and hypersurface
partitions is $O\left(N^{D}\right)$ and $O\left(N^{G}\right)$, respectively,
where $G=\left(\begin{array}{c}
	D+M\\
	D
\end{array}\right)-1$ and $M$ is the polynomial degree for defining hypersurfaces. This
is achieved by exhaustively enumerating all hyperplanes passing through
exactly $D$ points (for hyperplanes) or $G$ points (for hypersurfaces),
which has combinatorial complexity $\left(\begin{array}{c}
	N\\
	D
\end{array}\right)=O\left(N^{D}\right)$ and $\left(\begin{array}{c}
	N\\
	G
\end{array}\right)=O\left(N^{G}\right)$, respectively. In other words, given a list of data points $\mathit{xs}:\left[\mathbb{R}^{D}\right]$
and polynomial degree $M$, the splitting rule generator for general
hyperplanes and hypersurfaces can be formulated as follows:

\[
\mathit{genRules}_{\text{HODT}}\left(M,\mathit{xs}\right)=\left[r_{i}\mid r_{i}\leftarrow\mathit{kcombs}\left(G,\mathit{xs}\right)\right],
\]
where $G=\left(\begin{array}{c}
	D+M\\
	D
\end{array}\right)-1$.

Therefore, to solve the ODT problem in ML, we simply compose $\mathit{genRules}:\mathcal{D}\to\left[\mathcal{R}\right]$
with $\mathit{odt}_{K}:\left[\mathcal{R}\right]\to\mathit{DTree}\left(\mathcal{R},\mathcal{D}\right)$
as
\begin{align*}
	\mathit{odtML}_{K} & :\mathcal{D}\to\mathit{DTree}\left(\mathcal{R},\mathcal{D}\right)\\
	\mathit{odtML}_{K} & =\mathit{odt}_{K}\circ\mathit{genRules}
\end{align*}

By switching the definition of $\mathit{genRules}$ we can solve either
the AODT problem or the more general hypersurface decision tree problem,
while keeping the main program $\mathit{odt}_{K}$ unchanged. This
\textbf{modularity} provides flexible solutions for ODT problems in
machine learning. The resulting complexities are $O\left(N\times D+K!\times N\times\left(N\times D\right)^{K}\right)=O\left(N^{K}\right)$
for the AODT problem, $O\left(N^{D}+K!\times N\times\left(N^{D}\right)^{K}\right)=O\left(N^{DK}\right)$
for optimal hyperplane decision trees, and $O\left(N^{G}+K!\times N\times\left(N^{G}\right)^{K}\right)=O\left(N^{GK}\right)$
for hypersurface decision trees of degree $M$.

Even better, with the help of an ingenious combination generator,
such as the one developed \citep{he2024ekm}, we can create the ODT
with mixed splitting rules---where axis-parallel, hyperplane, and
hypersurface splitting rules are used simultaneously within the same
tree. By contrast, classical approaches can only assume one type of
splitting rule. To the best of our knowledge, such a decision tree
with mixed splitting rules has not been described previously.

\subsection{$K$-D tree}

\begin{figure}[h]
	\centering{}\includegraphics[scale=0.25]{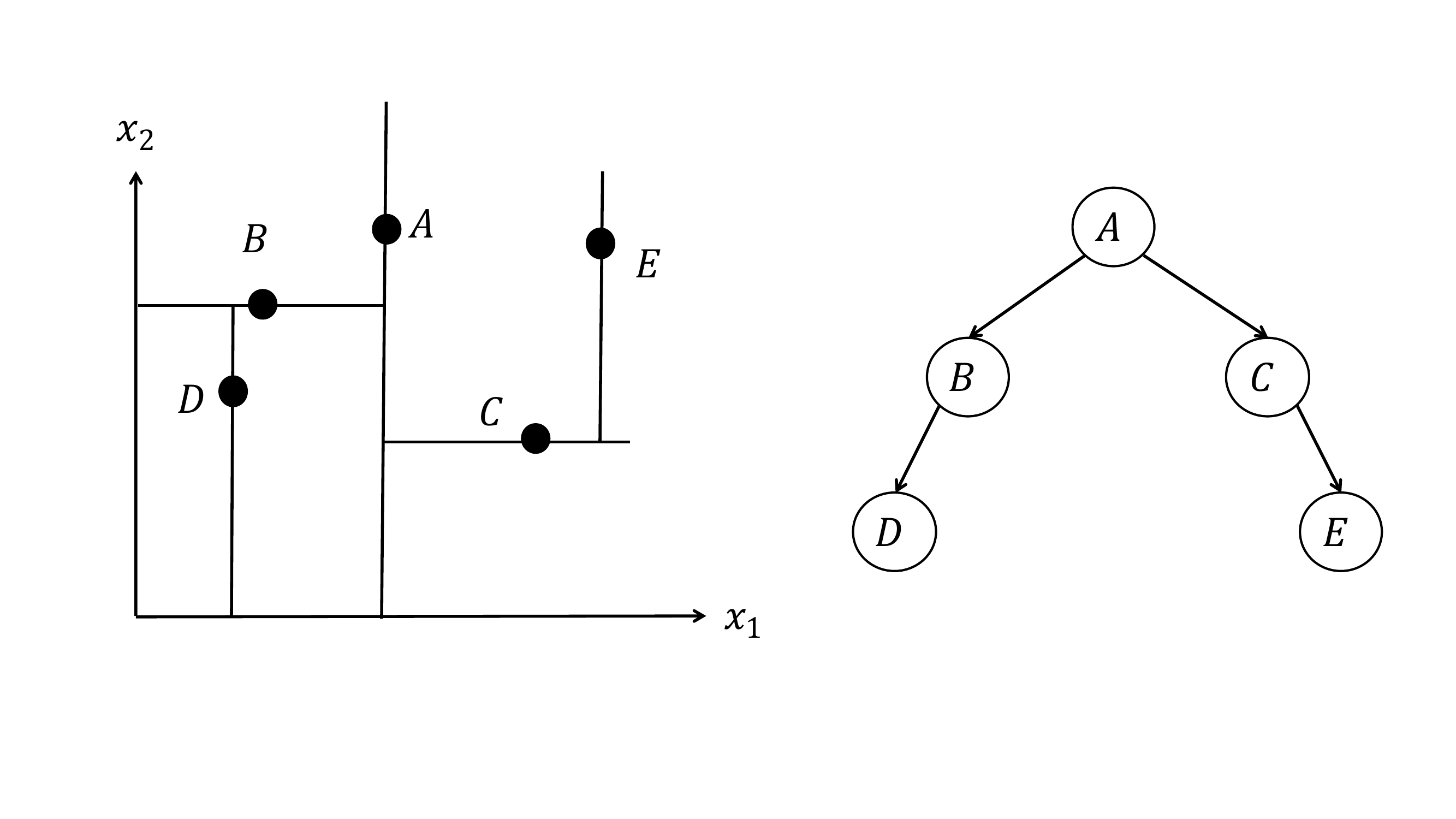}\caption{A $K$-D contains seven data points. The split axis for branch nodes
		at the same level is consistent. Node $A$ splits along the first
		coordinate,$x_{1}$, while nodes $B$ and $C$ split along the second
		coordinate, $x_{2}$. Nodes $D$ and $E$ then split along the first
		coordinate, $x_{1}$, once again.\label{fig: K-D tree illustrate}}
\end{figure}

The $K$-D tree is a fundamental data structure designed for efficiently
processing multi-dimensional search queries. Introduced by \citet{bentley1975multidimensional}
, it shares similarities with the axis-parallel decision tree model
in ML. The key distinction lies in the branching rules: while axis-parallel
decision trees allow branch nodes to be defined by arbitrary axis-aligned
splitting rules, $K$-D trees impose a constraint in which all branch
nodes at the same level must follow a predefined splitting rule based
on a specific dimension.

For instance, in the $K$-D tree illustrated in Figure \ref{fig: K-D tree illustrate},
the root node applies a splitting rule based on the horizontal axis.
Then the splitting axis alternates between the vertical and horizontal
axes at each subsequent level.

Similar to an axis-parallel hyperplane decision tree, where possible
splitting rules are derived from the data---each dimension having
$O\left(N\right)$ choices---resulting in a total of $O\left(N\times D\right)$
possible splits. However, the K-D tree imposes an additional constraint:
all splits at the same level must occur along a fixed dimension. This
restriction reduces the combinatorial complexity of the problem, as
it limits the consideration to $O\left(N\right)$ possible splits,
each corresponding to one of the $O\left(N\right)$ data points along
a predetermined dimension. Consequently, the tree data type must be
redefined to incorporate dimension information at the root of each
subtree. This can be achieved by pairing each branch node of the tree
with a natural number:
\[
\mathit{Tree}_{\text{KD}}\left(\mathcal{A},\mathcal{B}\right)=L\left(\mathcal{B}\right)\mid N\left(\mathit{Tree}_{\text{KD}}\left(\mathcal{A},\mathcal{B}\right),\left(\mathcal{A},\mathbb{N}\right),\mathit{Tree}_{\text{KD}}\left(\mathcal{A},\mathcal{B}\right)\right).
\]

Also, the $\mathit{splits}$ function is redefined to incorporate
the tree's depth information, formally expressed as:

\begin{align*}
	\mathit{splits}_{\text{KD}} & \left(d,\mathit{xs}\right)=\left[\left(\mathit{sp}_{\text{KD}}\left(x,d,\mathit{xs}\right),x,\mathit{sn}_{\text{KD}}\left(x,d,\mathit{xs}\right)\right)\mid x\leftarrow xs\right],
\end{align*}
where $\mathit{xs}$ represents a list of data points in $\mathbb{R}^{D}$.
The function $\mathit{sp}_{\text{KD}}$ takes a root node $x_{i}$
and outputs all data points in $\mathit{xs}$ whose $d$th coordinate
is smaller than that of $x$. Similarly, the function $\mathit{sn}_{\text{KD}}$
selects all data points in $\mathit{xs}$ with greater $d$th coordinates
than $x$.

Since $K$-D trees have numerous applications, including nearest neighbor
search (finding the closest point(s) to a given query point in a dataset),
range search (retrieving all points within a specified range or bounding
box), and image processing (feature matching or clustering in multi-dimensional
feature spaces), the definition of the objective function depends
on the specific application requirements. By combining the $\mathit{splits}_{\text{KD}}$
functions with the$\mathit{odt}$ program, we can construct an optimal
$K$-D tree tailored to a given objective function.

\section{Extension to non-proper decision trees\label{sec:Extension-to-non-proper}}

We refer to decision trees that violate \emph{any} proper decision
tree axioms as \emph{non-proper decision trees}. We examine four examples
here. The first is the optimal decision tree problem over binary feature
data (ODT-BF), which discards the ancestral constraint directly (Axiom
4), as a descendant rules can be both left- and right-children. The
second and third examples arise from modeling the Murtree algorithm
proposed by \citet{demirovic2022murtree}, a state-of-the-art method
for solving the ODT-BF problem. Both replace Axiom 4 with a different
ancestral constraint.

The last example is the well-studied matrix chain multiplication problem
(MCMP), which violates a structural constraint (Axiom 1), as its branch
nodes contain no information. Indeed, the MCMP is not a canonical
``decision tree problem'' as it represents a tree whose branch nodes
are unlabeled. Nevertheless, we include its discussion because a closely
related variant (the first model for the Murtree algorithm) can be
viewed as a decision tree problem, where the splitting rules are fixed
in a predetermined order. This allows for a meaningful cross-comparison
with different formulations.

\subsection{The optimal decision tree problem over binary feature data}

\paragraph{Interpretation of splitting rules for the decision tree over binary
	feature data}

The ODT-BF problem aims to find a decision tree over a binary feature
dataset $\left\{ 0,1\right\} ^{D}\to L$, where $L$ is the type of
the labels. Some researchers classify ODT-BF as an axis-parallel decision
tree problem; however, this characterization obscures the geometric
meaning of axis-parallel hyperplanes. Binary feature data correspond
to points at the vertices of a hypercube, and splits cannot be generated
from the midpoints between vertices, as the data are binary rather
than continuous.

We argue that a better way to understand the splitting rule in this
question is to consider following: Given a new data item $x\in\mathbb{R}^{D}$,
each splitting rule in ODT-BF problem corresponds to determining whether
the $i$-th feature of a data point $x\in\mathbb{R}^{D}$ is 0 or
1---equivalently, whether $x$ contains feature $i$? A consequence
of this characterization is that the number of splitting rules becomes
independent to the number of data, because a hypercube in $\mathbb{R}^{D}$
has at most $2^{D}$ vertexes. Indeed, in the ODT-BF problem, many
data points are duplicates and occupy the same vertices of the cube.

Assume a rule $r_{i}$ is the ancestor of another rule $r_{j}$. The
question of whether $x$ contains feature $i$ is independent of whether
it contains feature $j$. This means that fixing an ancestor rule
$r_{i}$ does not influence whether $r_{j}$ should go to the left
or right subtree of $r_{i}$. In other words, if $r_{i}$ is the ancestor
of $r_{j}$, both $r_{i}\swarrow r_{j}$ and $r_{i}\searrow r_{j}$
are valid. Consequently, Axiom 4 of the proper decision tree no longer
holds, and the ODT-BF problem cannot be solved using the $\mathit{odt}_{K}$
program.

As a result, ODT-BF satisfies only the structural constrains of decision
trees, there is \textbf{no} ancestral constraint to the ODT-BF problem.
Consequently, the ODT-BF problem is an ordinary labeled binary tree
problem with no explicit constraints, resulting in a complexity of
$K!\times\mathit{Catalan}\left(K!\right)$ for an input of size $K$,
since there are $\mathit{Catalan}\left(K!\right)$ possible tree shapes
and $K!$ possible labelings.

\paragraph{Exhaustive generator for the unconstrained labeled binary tree problem}

Similar to $\mathit{genDTs}$, where we showed that structural constraint
Axiom 2 can be derived once a generator satisfies Axioms 1 and 3,
here we only discuss a partial generator for the ODT-BF problem. The
complete generator can be obtained by following a process analogous
to Theorem \ref{thm: comlete DT generator}.

Since the ODT-BF problem has no ancestral constraints, how should
we define its split function? In other words, once a root $r$ is
fixed, how should the two sublists $\mathit{rs}^{+}$ and $\mathit{rs}^{-}$
for its subtrees be determined from $\mathit{rs}/r$? The answer is
straightforward: \textbf{any subset of} $\mathit{rs}/r$ \textbf{is
	valid}! The following function produces all possible sublists $\mathit{xs}^{\prime}$
along with their complementary subsets $\mathit{xs}/\mathit{xs}^{\prime}$
as a pair.
\[
\begin{aligned}\mathit{subsPair} & \left(\left[\:\right],\mathit{ys}\right)=\left[\left(\left[\:\right],\mathit{ys}\right)\right]\\
	\mathit{subsPair} & \left(x:\mathit{xs},\mathit{ys}\right)=\mathit{subsPair}\left(\mathit{xs},\mathit{ys}\right)\cup\mathit{mapL}\left(f,\mathit{subsPair}\left(\mathit{xs},\mathit{ys}\right)\right)
\end{aligned}
\]
where $f\left(a,\mathit{as},\mathit{bs}\right)=\left(a:\mathit{as},\mathit{bs}/a\right)$.

For instance, running $\mathit{subsPair}\left(\left[1,2,3\right],\left[1,2,3\right]\right)$
returns\\
 $\left[\left(\left[\:\right],\left[1,2,3\right]\right),\left(\left[3\right],\left[1,2\right]\right),\left(\left[2\right],\left[1,3\right]\right),\left(\left[2,3\right],\left[1\right]\right),\left(\left[1\right],\left[2,3\right]\right),\left(\left[1,3\right],\left[2\right]\right),\left(\left[1,2\right],\left[3\right]\right),\left(\left[1,2,3\right],\left[\:\right]\right)\right]$, 
the $2^{3}$ possible sublists of $\left[1,2,3\right]$. Then, the
split function for the ODT-BF problem is defined as

$\mathit{splits}_{\text{BF}}\left(\mathit{rs}\right)=\mathit{subsPair}\left(\mathit{rs},\mathit{rs}\right)$

Following a similar derivation process as $\mathit{genDTBFs}$, we
can now formulate the search space of the ODT-BF problem using the
following generator.
\begin{definition}
	The generator for the search space of the ODT-BF problem over the
	binary tree datatype is defined as
	\begin{equation}
		\begin{aligned}\mathit{genDTBFs} & :\left[\mathcal{R}\right]\to\left[\mathcal{R}\right]\\
			\mathit{genDTBFs} & \left(\left[\:\right]\right)=\left[\:\right]\\
			\mathit{genDTBFs} & \left(\left[r\right]\right)=\left[\mathit{N}\left(L,r,L\right)\right]\\
			\mathit{genDTBFs} & \left(\mathit{rs}\right)=\left[\mathit{N}\left(u,r_{i},v\right)\mid r_{i}\leftarrow\mathit{rs},\left(\mathit{rs}^{+},\mathit{rs}^{-}\right)\leftarrow\mathit{splits}_{\text{BF}}\left(\mathit{rs}/r\right),u\leftarrow\mathit{\mathit{genDTBFs}}\left(\mathit{rs}^{+}\right),v\leftarrow\mathit{\mathit{genDTBFs}}\left(\mathit{rs}^{-}\right)\right].
		\end{aligned}
		\label{eq: generator of ODT-BF}
	\end{equation}
\end{definition}
The following lemma proves that $\mathit{genDTBFs}$ generates $K!\times\mathit{Catalan}\left(K!\right)$
trees for an input of size $K$.
\begin{lemma}
	\emph{The combinatorial complexity of $\mathit{genDTBFs}$ is given
		by}
	
	\emph{
		\[
		\begin{aligned}f_{\text{BF}} & \left(0\right)=1\\
			f_{\text{BF}} & \left(1\right)=1\\
			f_{\text{BF}} & \left(K\right)=K\times\sum_{i=0}^{K-1}\left(\left(\begin{array}{c}
				K-1\\
				i
			\end{array}\right)\times\left(f_{\text{BF}}\left(i\right)\times f_{\text{BF}}\left(K-1-i\right)\right)\right)
		\end{aligned}
		\]
		with solution $f_{\text{BF}}\left(K\right)=K!\times\mathit{Catalan}\left(K\right)$.}
\end{lemma}
\begin{proof}
	The base cases $f_{\text{BF}}\left(0\right)$ and $f_{\text{BF}}\left(1\right)$
	hold trivially. We prove the recursive case by induction:
	\begin{align*}
		f_{\text{BF}}\left(K\right)= & K\times\sum_{i=0}^{K-1}\left(\left(\begin{array}{c}
			K-1\\
			i
		\end{array}\right)\times\left(f_{\text{BF}}\left(i\right)\times f_{\text{BF}}\left(K-1-i\right)\right)\right)\\
		= & \text{ induction hypothesis}\\
		& K\times\sum_{i=0}^{K-1}\left(\left(\begin{array}{c}
			K-1\\
			i
		\end{array}\right)\times\left(i!\times\mathit{Catalan}\left(i\right)\times\left(K-1-i\right)!\times\mathit{Catalan}\left(K-1-i\right)\right)\right)\\
		= & \text{ definition of permutation \ensuremath{\left(K-1-i\right)!\times i!=\frac{\left(K-1\right)!}{\left(\begin{array}{c}
						K-1\\
						i
					\end{array}\right)}}}\\
		& K\times\sum_{i=0}^{K-1}\left(\left(\begin{array}{c}
			K-1\\
			i
		\end{array}\right)\times\left(\frac{\left(K-1\right)!}{\left(\begin{array}{c}
				K-1\\
				i
			\end{array}\right)}\times\mathit{Catalan}\left(i\right)\times\mathit{Catalan}\left(K-1-i\right)\right)\right)\\
		= & K\times\sum_{i=0}^{K-1}\left(\left(K-1\right)!\times\mathit{Catalan}\left(i\right)\times\mathit{Catalan}\left(K-1-i\right)\right)\\
		= & K!\times\sum_{i=0}^{K-1}\left(\mathit{Catalan}\left(i\right)\times\mathit{Catalan}\left(K-1-i\right)\right)\\
		= & K!\times\mathit{Catalan}\left(K\right)
	\end{align*}
\end{proof}
As a result, $\mathit{genDTBFs}$ generates all possible unconstrained
labeled binary trees. By following a process analogous to that used
for $\mathit{sodt}$, an efficient DP algorithm can be derived from
(\ref{eq: generator of ODT-BF}).

\subsection{Formal models for the Murtree algorithm \label{subsec:Non-exhua-Murtree}}

The Murtree algorithm proposed by \citet{demirovic2022murtree} has
been recognized as state-of-art for the ODT-BF problem, and many of
the techniques introduced in their work have inspired subsequent studies
\citep{brita2025optimal,aglin2021pydl8}. However, several critical
aspects of their algorithm remain ambiguous. These ambiguities mean
the claimed optimality is insufficiently addressed, particularly regarding
the algorithm's exhaustiveness. The lack of clarity admits multiple,
potentially inconsistent interpretations.

For instance, from the description in \citet{demirovic2022murtree},
although they claim a key contribution is ``providing a clear high-level
view of the algorithm using conventional algorithmic principles, namely
dynamic programming and search,'', as summarized in their Equation
(1), this equation actually describes a process different from both
their pseudo-code and implementation in GitHub repository. As a result,
we can find at least two possible ways to define the splits function,
yielding two formal models for the Murtree algorithm, both of which
lead to suboptimal or incorrect solutions. We believe this investigation
is valuable, as it identifies possible sources of misunderstanding
and may help future researchers avoid similar pitfalls when interpreting
or reimplementing the method.

\paragraph{Specification for Equation (1)}

From the recursive pattern of Equation (1), their method proceeds
by sequentially scanning the features $i\in\left[0,n-1\right]$ (where
$n$ is the number of splitting rules in our terminology, denote this
list as $\mathit{rs}$) and, for a fixed feature $i$, we denote it
as $r_{i}$, partitions the remaining rules $\mathit{rs}/r_{i}$ into
two sublists: one containing $n-i-1$ rules and the other containing
$i$ rules. Since they do not specify which of the $n-i-1$ rules
in $\mathit{rs}/r_{i}$ go to the left subtree and which of the left-subtree,
a plausible interpretation is that all rules preceding $r_{i}$ in
$\mathit{rs}$ are assigned to the left subtree, and all rules following
$r_{i}$ are assigned to the right subtree. This amounts to replacing
Axiom 4 with the following ancestral constraint:
\begin{itemize}
	\item Axiom 4': Given a list of rule $\mathit{rs}_{K}=\left[r_{1},r_{2},\ldots,r_{K}\right]$,
	and fix a root $r_{i}$, $\boldsymbol{K}_{ij}=1$ if $r_{j}$ precedes
	$r_{i}$ in $\mathit{rs}_{K}$ and $\boldsymbol{K}_{ij}=-1$ if $r_{i}$
	precedes $r_{j}$ in $\mathit{rs}_{K}$.
\end{itemize}
This ancestral constraint can also be applied to the specification
of a \emph{binary search tree} (BST). For an arbitrary fixed integer
$i$, values greater than $i$ are assigned to the right subtree,
while values less than or equal to $i$ are assigned to the left subtree.
For example, if $\mathit{rs}=\left[r_{1},r_{2},r_{3},r_{4}\right]$,
and $r_{3}$ is chosen as the root, then $\left[r_{1},r_{2}\right]$
are assigned to the left subtree and $\left[r_{4}\right]$ is assigned
to the right subtree.

Under this ancestral constraint, denote $\left[r_{1},r_{2},\ldots,r_{i-1}\right]$
as $\mathit{rs}^{+}$ and $\left[r_{i+1},\ldots,r_{n}\right]$ as
$\mathit{rs}^{-}$ separately, the split function can be defined as
follows:
\[
\mathit{splits}_{\text{Catalan}}\left(r_{i},\left[r_{1},r_{2},\ldots,r_{i},\ldots,r_{n}\right]\right)=\left(\mathit{rs}^{+},\mathit{rs}^{-}\right)
\]
which we call the ``Catalan'' split because it resembles a Catalan-style
partitioning. Consequently, the search space of the Murtree algorithm
is defined as follows.
\begin{definition}
	The first possible generator for the search space that the Murtree
	algorithm attempts to explore is defined as
	
	\begin{equation}
		\begin{aligned}\mathit{genCatalans} & :\left[\mathcal{R}\right]\to\left[\mathit{BTree}\left(\mathcal{R}\right)\right]\\
			\mathit{genCatalans} & \left(\left[\:\right]\right)=\left[\:\right]\\
			\mathit{genCatalans} & \left(\left[r\right]\right)=\left[\mathit{N}\left(L,r,L\right)\right]\\
			\mathit{genCatalans} & \left(\mathit{rs}\right)=\left[\mathit{N}\left(u,r_{i},v\right)\mid r_{i}\leftarrow\mathit{rs},\left(\mathit{rs}^{+},\mathit{rs}^{-}\right)\leftarrow\mathit{splits}_{\text{Catalan}}\left(r_{i},\mathit{rs}\right),u\leftarrow\mathit{genMurtrees}\left(\mathit{rs}^{+}\right),v\leftarrow\mathit{genMurtrees}\left(\mathit{rs}^{-}\right)\right]
		\end{aligned}
	\end{equation}
\end{definition}
An audience familiar with combinatorics may already recognize that
$\mathit{genCatalans}$ is exactly the generator for unlabeled decision
trees, which has complexity $\mathit{Catalan}\left(K\right)$ for
an input $\mathit{rs}$ of size $K$. This is because $\mathit{splits}_{\text{Catalan}}$
implicitly assumes a fixed ordering for $\mathit{rs}$;\textbf{ a
	fixed-order set of labeled elements is equivalent to a set of unlabeled
	elements}. The generator $\mathit{genCatalans}$ also tightly related
to the search space of binary search tree, assume that $r_{1}\leq r_{2}\leq r_{3}$,
running $\mathit{genCatalans}\left(\left[r_{1},r_{2},r_{3}\right]\right)$
gives us five possible binary search tree with respect to $\left[r_{1},r_{2},r_{3}\right]$,
such that the splitting rules in the left subtree are smaller than
the root, and splitting rules in the right subtree are greater than
root, see top five trees in Figure \ref{fig: running examples of murtree and catatree}
for running results of $\mathit{genCatalans}\left(\left[r_{1},r_{2},r_{3}\right]\right)$.

\paragraph{Specification for pseudocode and implementation}

In \citet{demirovic2022murtree}'s Algorithm 2 (describing the recursive
case of their algorithm), lines 7--15 correspond to the behavior
of the split function ( lines 386--403 in the $\texttt{solver.cpp}$
file of their source code). Lines 7--8 impose a filtering condition
that constrains the tree depth.

The first for loop at line 9 (while loop 391 in$\texttt{solver.cpp}$
file) traverses the candidate rules list, analogous to $r_{i}\leftarrow\mathit{rs}$
in our list comprehension definition. At line 14, there is a second
for loop (not explicitly mentioned in their ``high-level abstraction''
Equation (1)) that traverses the remaining list from $\left[n_{\mathit{min}},n_{\mathit{max}}\right]$
to assign the remaining rules to the right subtree, where $n_{\mathit{min}},n_{\mathit{max}}$
is determined by the maximal depth of subtrees. Since our focus is
on analyzing the exhaustiveness of the search space, we ignore the
discussion regarding filtering conditions so instead of traversing
the filtered list $\left[n_{\mathit{min}},n_{\mathit{max}}\right]$
as computed in lines 7--8 of Algorithm 2, we traverse the complete
list $\left[1,n-1\right]$. Hence for each $i\in\left[1,n-1\right]$
we partition the remaining $\mathit{rs}_{K}\backslash r_{i}$ into
two disjoint lists.

This amounts to replacing Axiom 4 with the following ancestral constraint:
\begin{itemize}
	\item Axiom 4'': Given a list of rule $\mathit{rs}_{K}=\left[r_{1},r_{2},\ldots,r_{K}\right]$,
	and fix a root $r_{i}$, if $\boldsymbol{K}_{ij}=1$ then $\boldsymbol{K}_{ik}=1$
	for all $r_{k}$ in $\mathit{rs}_{K}$ precedes $r_{j}$, and if $\boldsymbol{K}_{ij}=-1$
	then $\boldsymbol{K}_{ik}=-1$ for all $r_{k}$ in $\mathit{rs}_{K}$
	come after $r_{j}$.
\end{itemize}
For example, given a candidate rule list $\mathit{rs}=\left[r_{1},r_{2},r_{3},r_{4}\right]$,
if we select $r_{3}$ as the root, the remaining rules are split in
four possible ways: 
\[
\left[\left(\left[\:\right],\left[1,2,4\right]\right),\left(\left[1\right],\left[2,4\right]\right),\left(\left[1,2\right],\left[4\right]\right),\left(\left[1,2,4\right],\left[\:\right]\right)\right].
\]
This is essentially all possible list partitions with respect to $\mathit{rs}\backslash r_{3}=\left[r_{1},r_{2},r_{4}\right]$.
In the Murtree algorithm, this process is repeatedly applied to each
$r_{i}$ in $\mathit{rs}$ (the first for loop at line 9 ``\textbf{for}
splitting feature $f\in\mathcal{F}$ \textbf{do}'' ), resulting a
total complexity of $K\times\left(\left(K-1\right)+1\right)=K^{2}$
for a rule list of size $K$ (For a list of size $K-1$, there are
$K-1+1$ possible list partitions).

Following \citet{he2025ROF}, list partitions can be defined recursively
as follows:
\[
\mathit{splits}_{\text{Murtree}}\left(r:\mathit{rs}\right)=\left(\left[\:\right],r:\mathit{rs}\right):\left[\left(x:\mathit{ys},\mathit{zs}\right)\mid\left(\mathit{ys},\mathit{zs}\right)\leftarrow\mathit{splits}_{\text{Murtree}}\left(\mathit{rs}\right)\right]
\]

We can now formalize the search space that the Murtree algorithm's
pseudocode attempts to explore as follows.
\begin{definition}
	The second possible generator for the search space that the Murtree
	algorithm attempts to explore is defined as
	
	\begin{equation}
		\begin{aligned}\mathit{genMurtrees} & :\left[\mathcal{R}\right]\to\left[\mathit{BTree}\left(\mathcal{R}\right)\right]\\
			\mathit{genMurtrees} & \left(\left[\:\right]\right)=\left[\:\right]\\
			\mathit{genMurtrees} & \left(\left[r\right]\right)=\left[\mathit{N}\left(L,r,L\right)\right]\\
			\mathit{genMurtrees} & \left(\mathit{rs}\right)=\left[\mathit{N}\left(u,r_{i},v\right)\mid r_{i}\leftarrow\mathit{rs},\left(\mathit{rs}^{+},\mathit{rs}^{-}\right)\leftarrow\mathit{splits}_{\text{Murtree}}\left(\mathit{rs}/r_{i}\right),u\leftarrow\mathit{genMurtrees}\left(\mathit{rs}^{+}\right),v\leftarrow\mathit{genMurtrees}\left(\mathit{rs}^{-}\right)\right]
		\end{aligned}
	\end{equation}
\end{definition}
By modeling the search space of the Murtree algorithm using $\mathit{genMurtrees}$
and $\mathit{genCatalans}$ it is straightforward to verify that the
Murtree algorithm is optimal with respect to $\mathit{genMurtrees}$
or $\mathit{genCatalans}$ by using a similar derivation that we show
in previous sections.

\begin{figure}
	\begin{centering}
		\includegraphics[viewport=50bp 250bp 1280bp 500bp,clip,scale=0.25]{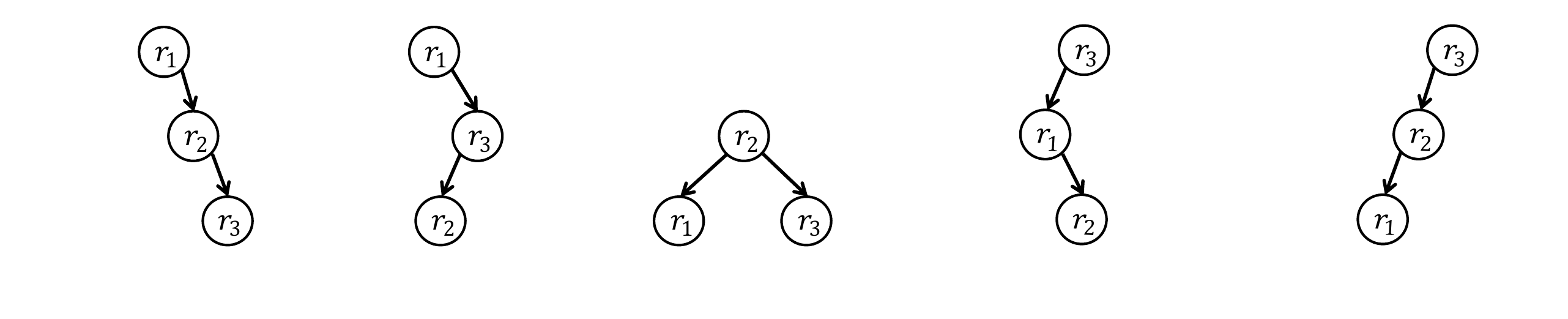}
		\par\end{centering}
	\begin{centering}
		\includegraphics[scale=0.25]{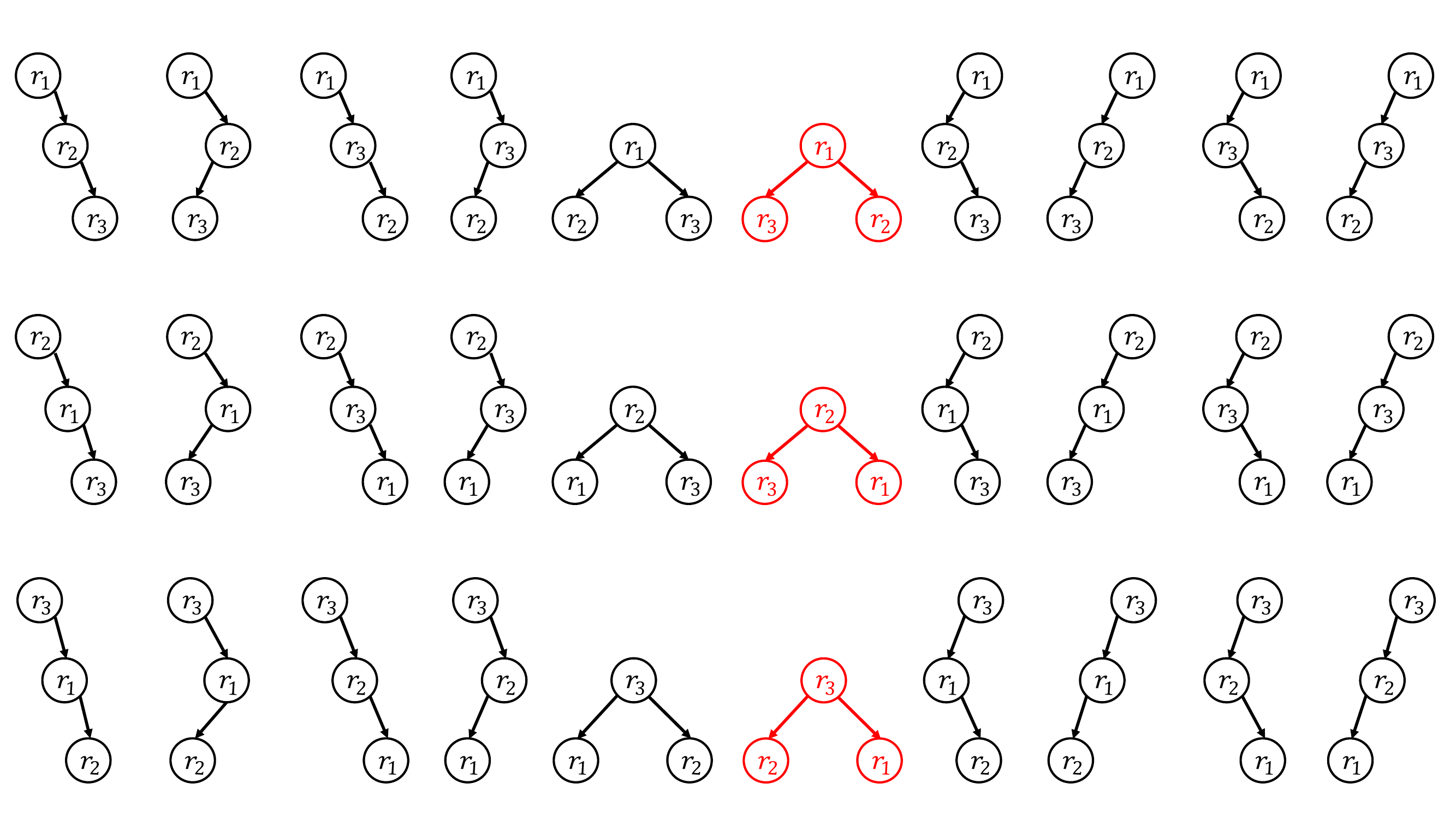}
		\par\end{centering}
	\caption{The $\mathit{genDTBFs}$ generator serves as the ideal exhaustive
		generator for the optimal decision tree problem over binary feature
		data (ODT-BF). Any generator that explores fewer trees than $\mathit{genDTBFs}$
		is non-exhaustive. Unless a specific objective function is chosen
		such that the omitted trees are non-optimal, non-exhaustiveness is
		therefore equivalent to non-optimality. The five trees on the top
		shows all possible decision trees generated by $\mathit{genCatalans}\left(\left[r_{1},r_{2},r_{3}\right]\right)$.
		The 30 trees on the bottom shows all possible decision trees generated
		by $\mathit{genDTBFs}\left(\left[r_{1},r_{2},r_{3}\right]\right)$,
		with a total size of $\left|\mathit{genDTBFs}\left(\left[r_{1},r_{2},r_{3}\right]\right)\right|=30$.
		The set of possible decision trees generated by $\mathit{genMurtrees}\left(\left[r_{1},r_{2},r_{3}\right]\right)$
		is obtained by excluding all red-colored trees from the right figure,
		resulting in a total size of $\left|\mathit{genMurtrees}\left(\left[r_{1},r_{2},r_{3}\right]\right)\right|=27$.\label{fig: running examples of murtree and catatree}}
\end{figure}

However, $\mathit{genMurtrees}$ does not generate the same number
of trees as $\mathit{genDTBFs}$. At each recursive step, $\mathit{splits}_{\text{Murtree}}$
partitions a size $N$ rule list into $N+1$ list partitions, whereas
$\mathit{splits}_{\text{BF}}$ produces $2^{N}$ possible splits,
which is substantially larger. For example, $\left|\mathit{genMurtrees}\left(\left[r_{1},r_{2},r_{3},r_{4},r_{5}\right]\right)\right|=2830$
and $\left|\mathit{genDTBFs}\left(\left[r_{1},r_{2},r_{3},r_{4},r_{5}\right]\right)\right|=5040$,
nearly twice as many even for just five candidate rules. See Figure
\ref{fig: running examples of murtree and catatree} for running examples
of $\mathit{genCatalans}\left(\left[r_{1},r_{2},r_{3}\right]\right)$,
$\mathit{genMurtrees}\left(\left[r_{1},r_{2},r_{3}\right]\right)$,
and $\mathit{genDTBFs}\left(\left[r_{1},r_{2},r_{3}\right]\right)$.

Let $f_{\text{Catalan}}$ and $f_{\text{Murtree}}$ denote the combinatorial
complexities of $\mathit{genCatalans}$ and $\mathit{genMurtrees}$,
respectively. Then we have the inclusion relation
\begin{equation}
	f_{\text{Catalan}}\left(N\right)<f_{\text{Murtree}}\left(N\right)<f_{\text{BF}}\left(N\right).
\end{equation}
Therefore, the two models ($\mathit{genCatalans}\left(\left[r_{1},r_{2},r_{3}\right]\right)$,
$\mathit{genMurtrees}\left(\left[r_{1},r_{2},r_{3}\right]\right)$)
for the search spaces of the Murtree algorithm are not exhaustive
with respect to the full search space of the ODT-BF problem, which
is generated by $\mathit{genDTBFs}$.

\subsection{Matrix chain multiplication problem\label{subsec:Matrix-chain-multiplication}}

\begin{figure}[h]
	\begin{centering}
		\includegraphics[scale=0.25]{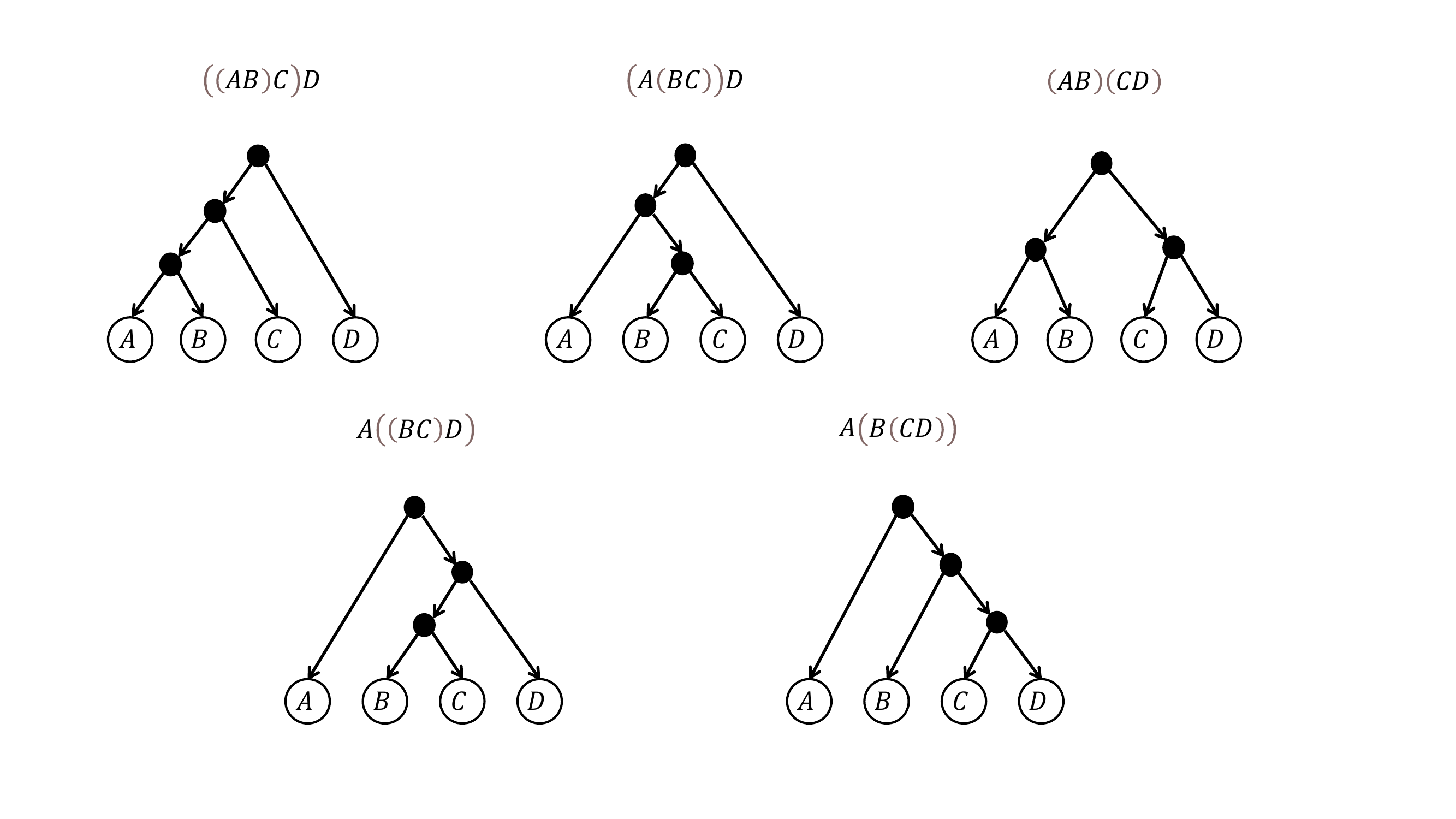}
		\par\end{centering}
	\caption{The possible parenthesizations for multiplication of four matrices
		$A$, $B$, $C$, $D$ correspond to the leaf-labeled trees below,
		where solid black nodes are the branch nodes that contain no information,
		and each leaf node stores a matrix.\label{fig: MCMP}}
\end{figure}

Another instance of the non-proper decision tree problem is the matrix
chain multiplication problem (MCMP), which was explored previously
by \citet{bird2020algorithm}. Indeed \citet{bird2020algorithm} showed
that this problem can be more generally formalized as the \emph{optimal
	bracketing problem}, which also provides a framework for solving optimization
problems related to \emph{binary search trees}.

The MCMP (as does the optimal bracketing problem) is closely related
to $\mathit{genCatalans}$ discussed above for the Murtree algorithm.
The key difference between MCMP and $\mathit{genCatalans}$ is that
its internal nodes are unlabeled. As \citet{bird2020algorithm} have
shown, we can simply delete the branch nodes and define the tree as
\[
\mathit{Tree}_{\text{MCMP}}\left(\mathcal{A}\right)=L\left(\mathcal{A}\right)\mid N\left(\mathit{Tree}_{\text{MCMP}}\left(\mathcal{A}\right),\mathit{Tree}_{\text{MCMP}}\left(\mathcal{A}\right)\right).
\]
In the case of the MCMP problem, for $N$ matrices stored in a list,
we can think of the MCMP as the number of ways to associate parenthesizations
of the $N$ matrices. For instance, when $N=4$, as demonstrated in
Figure \ref{fig: MCMP}, there are five possible parenthesizations:$\left(\left(AB\right)C\right)D$,
$\left(A\left(BC\right)\right)D$, $\left(AB\right)\left(CD\right)$,
$A\left(\left(BC\right)D\right)$, and $A\left(B\left(CD\right)\right)$.
The trees in Figure \ref{fig: MCMP} are isomorphic to those generated
by $\mathit{genCatalans}\left(\left[r_{1},r_{2},r_{3}\right]\right)$
(top five trees in Figure \ref{fig: running examples of murtree and catatree}),
except that all information is now encoded in the leaf nodes rather
than in the branch nodes.

As a result, matrices can be treated as the candidate rule list $\mathit{rs}$,
where the intrinsic ordering between them is fixed. For example, if
$\mathit{rs}=\left[A,B,C,D\right]$, then a split such as $\mathit{rs}^{+},\mathit{rs}^{-}=\left[A,C\right],\left[B,D\right]$
is not allowed, since $B$ precedes $C$ in $\mathit{rs}$. Another
required modification is to the split function, which is similar to
$\mathit{splits}_{\text{Catalan}}$, except that we do not require
any rule to serve as the root (i.e., no rule needs to remain in the
branch nodes as the key for indexing). Therefore, we can define the
$\mathit{splits}_{\text{Catalan}}$ function by simply splitting a
sequence of matrices into a pair of non-empty sub-sequences. In other
words, if all matrices are stored in a list $\mathit{rs}$ then we
want to know all possible partitions $\mathit{ys}$ and $\mathit{zs}$
such that $\mathit{rs}=\mathit{ys}\cup\mathit{zs}$. Hence, the $\mathit{splits}_{\text{MCMP}}$
function, as shown by \citet{bird2020algorithm}, can be defined as
\begin{align*}
	\mathit{splits}_{\text{MCMP}} & \left(\left[\right]\right)=\left[\right]\\
	\mathit{splits}_{\text{MCMP}} & \left(\left[x\right]\right)=\left[\right]\\
	\mathit{splits}_{\text{MCMP}} & \left(x:\mathit{xs}\right)=\left(\left[x\right],\mathit{xs}\right):\left[\left(x:\mathit{ys},\mathit{zs}\right)\mid\left(\mathit{ys},\mathit{zs}\right)\leftarrow\mathit{splits}_{\text{MCMP}}\left(\mathit{xs}\right)\right].
\end{align*}
A generator analogous to $\mathit{genCatalans}$, where the split
function is replaced by $\mathit{splits}_{\text{MCMP}}$ and the type
is updated to $\mathit{Tree}_{\text{MCMP}}$ yields an exhaustive
generator for the MCMP search space, and a similar derivation of the
optimal MCMP algorithm can be found in \citet{bird2020algorithm}.

\section{Discussion and Conclusions}

In this paper, our primary contribution is a novel axiomatic framework
for classifying and reasoning about decision trees through the identification
of structural and ancestral constraints. This framework provides a
formal foundation for future studies on decision tree problems.

The second contribution lies in the rigorous algorithmic development
for proper decision tree problems. We demonstrate the versatility
of this class in encompassing existing data structures and show how
proper decision trees simplify the combinatorics of binary-labeled
decision trees by establishing an injective mapping to $K$-permutations.
Moreover, the combinatorial insights from $K$-permutations allow
us to formally analyze the algorithmic properties of proper decision
trees. From this, we establish an unambiguous programmatic definition
of the search space for proper decision trees, which enables the derivation
of an efficient DP recursion for solving these problems.

The resulting DP algorithm also provides an interesting perspective
on memoization. While memoization is commonly used to accelerate DP,
we find that it is generally ineffective for most ODT problems with
large sets of candidate splitting rules. This finding also contradicts
claims in the existing literature regarding the purported trade-off
in caching techniques.

Another significant finding concerns the extension to non-proper decision
tree problems. We show that non-proper problems, which may involve
different ancestral constraints (as in ODT-BF and models for the Murtree
algorithm) or structural constraints (as in MCMP), can be addressed
using a similar but modified algorithmic approach.

Several interesting directions merit future exploration. First, although
we provide explicit derivations of optimal algorithms for the ODT-BF
and Murtree problems, future research can readily adapt the approach
established in this paper for similar derivations. Additionally, the
classification of non-proper decision tree problems resembles what
one might jokingly say amounts to ``classifying all animals that are
not elephants into the same class.'' This highlight the challenge
of further refining non-proper decision trees in future work.

\bibliographystyle{ACM-Reference-Format}
\bibliography{ODT}

\end{document}